\documentclass{article}

\usepackage[preprint]{neurips_2026}

\usepackage[utf8]{inputenc} % allow utf-8 input
\usepackage[T1]{fontenc}    % use 8-bit T1 fonts
\usepackage{hyperref}       % hyperlinks
\usepackage{url}            % simple URL typesetting
\usepackage{booktabs}       % professional-quality tables
\usepackage{amsfonts}       % blackboard math symbols
\usepackage{nicefrac}       % compact symbols for 1/2, etc.
\usepackage{microtype}      % microtypography
\usepackage{xcolor}         % colors
\usepackage{amsmath, amsthm}
\usepackage{enumitem}
\usepackage{cleveref}
\usepackage{tikz}
\usetikzlibrary{positioning,arrows.meta,calc}

\crefname{theorem}{Theorem}{Theorems}
\Crefname{theorem}{Theorem}{Theorems}
\crefname{appendix}{Appendix}{Appendices}
\Crefname{appendix}{Appendix}{Appendices}
\crefname{lemma}{Lemma}{Lemmas}
\Crefname{lemma}{Lemma}{Lemmas}
\crefname{corollary}{Corollary}{Corollaries}
\Crefname{corollary}{Corollary}{Corollaries}
\crefname{proposition}{Proposition}{Propositions}
\Crefname{proposition}{Proposition}{Propositions}
\crefname{definition}{Definition}{Definitions}
\Crefname{definition}{Definition}{Definitions}
\crefname{example}{Example}{Examples}
\Crefname{example}{Example}{Examples}
\crefname{remark}{Remark}{Remarks}
\Crefname{remark}{Remark}{Remarks}
\crefname{section}{Section}{Sections}
\Crefname{section}{Section}{Sections}
\crefname{subsection}{Section}{Sections}
\Crefname{subsection}{Section}{Sections}
\crefname{equation}{Eq.}{Eqs.}
\Crefname{equation}{Equation}{Equations}

\usepackage{amssymb, mathtools, bm}
\usepackage{mathrsfs}

\newtheorem{theorem}{Theorem}[section]
\newtheorem{lemma}[theorem]{Lemma}

\newtheorem{corollary}[theorem]{Corollary}
\newtheorem{definition}[theorem]{Definition}
\newtheorem{remark}[theorem]{Remark}
\newtheorem{example}[theorem]{Example}

\def\cB{\mathcal B}

\def\cD{\mathcal D}
\def\cF{\mathcal F}

\def\cK{\mathcal K}

\def\cM{\mathcal M}
\def\cN{\mathcal N}

\newcommand{\bA}{{\bf A}}

\newcommand{\bb}{{\bf b}}

\newcommand{\bD}{{\bf D}}
\newcommand{\bfe}{{\bf e}}

\newcommand{\bG}{{\bf G}}
\newcommand{\bh}{{\bf h}}

\newcommand{\bfs}{{\bf s}}

\newcommand{\bt}{{\bf t}}

\newcommand{\bv}{{\bf v}}

\newcommand{\bw}{{\bf w}}
\newcommand{\bW}{{\bf W}}
\newcommand{\bx}{{\bf x}}

\newcommand{\by}{{\bf y}}

\newcommand{\bz}{{\bf z}}

\newcommand{\tbfs}{\tilde\bfs}
\newcommand{\bone}{{\bf 1}}

\newcommand{\bgamma}{\boldsymbol{\gamma}}

\newcommand{\btheta}{\bm{\theta}}

\newcommand{\bxi}{\mbox{\boldmath{$\xi$}}}

\newcommand{\bbN}{{\mathbb N}}

\newcommand{\bbR}{{\mathbb R}}

\newcommand{\bbZ}{{\mathbb Z}}

\newcommand{\UB}{\mathcal{U}\mathcal{B}}

\DeclarePairedDelimiterX\Set[1]\{\}{%
  
  #1
}

\DeclarePairedDelimiterXPP\prob[1]{\mathbb{P}}(){}{%
  
  #1
}

\DeclarePairedDelimiterXPP\expt[1]{\mathbb{E}}[]{}{%
  
  #1
}

\DeclarePairedDelimiterXPP\exptd[2]{\mathbb{E}^{#2}}[]{}{%
  
  #1
}

\DeclarePairedDelimiterXPP\Log[1]{\operatorname{log}}(){}{#1}
\DeclarePairedDelimiterXPP\Exp[1]{\operatorname{exp}}(){}{#1}
\DeclarePairedDelimiterXPP\MEPR[1]{\operatorname{MEPR}}(){}{#1}
\DeclarePairedDelimiterXPP\Gaussian[2]{\mathcal{N}}(){}{#1,~#2}
\DeclarePairedDelimiterXPP\Ber[1]{\mathrm{Ber}}(){}{#1}
\DeclarePairedDelimiter\abs{\lvert}{\rvert}
\DeclarePairedDelimiter\norm{\lVert}{\rVert}
\DeclarePairedDelimiter\pbrac{(}{)}
\DeclarePairedDelimiter\sbrac{[}{]}
\DeclarePairedDelimiter\cbrac{\{}{\}}

\DeclarePairedDelimiterXPP\Covering[1]{\mathcal{N}}(){}{#1}

\makeatletter
\let\oldSet\Set
\def\Set{\@ifstar{\oldSet}{\oldSet*}}

\let\oldprob\prob
\def\prob{\@ifstar{\oldprob}{\oldprob*}}

\let\oldexpt\expt
\def\expt{\@ifstar{\oldexpt}{\oldexpt*}}

\let\oldexptd\exptd
\def\exptd{\@ifstar{\oldexptd}{\oldexptd*}}

\let\oldLog\Log
\def\Log{\@ifstar{\oldLog}{\oldLog*}}

\let\oldExp\Exp
\def\Exp{\@ifstar{\oldExp}{\oldExp*}}

\let\oldMEPR\MEPR
\def\MEPR{\@ifstar{\oldMEPR}{\oldMEPR*}}

\let\oldabs\abs
\def\abs{\@ifstar{\oldabs}{\oldabs*}}

\let\oldnorm\norm
\def\norm{\@ifstar{\oldnorm}{\oldnorm*}}

\let\oldBer\Ber
\def\Ber{\@ifstar{\oldBer}{\oldBer*}}

\let\oldGaussian\Gaussian
\def\Gaussian{\@ifstar{\oldGaussian}{\oldGaussian*}}

\let\oldpbrac\pbrac
\def\pbrac{\@ifstar{\oldpbrac}{\oldpbrac*}}

\let\oldsbrac\sbrac
\def\sbrac{\@ifstar{\oldsbrac}{\oldsbrac*}}

\let\oldcbrac\cbrac
\def\cbrac{\@ifstar{\oldcbrac}{\oldcbrac*}}
\makeatother

\DeclareMathOperator*{\argmin}{arg\,min}

\title{Posterior Contraction Rates for Sparse Kolmogorov-Arnold Networks in Anisotropic Besov Spaces}

\author{%
  Jeunghun Oh$^{1}$ \quad
  Jaeyong Lee$^{1}$ \quad
  Lizhen Lin$^{2}$ \quad
  Kyeongwon Lee$^{3}$\thanks{Corresponding author.} \\
  $^{1}$Department of Statistics, Seoul National University \\
  $^{2}$Department of Mathematics, University of Maryland \\
  $^{3}$Department of Statistics, Inha University \\
  \texttt{nomad1994@snu.ac.kr} \quad \texttt{leejyc@gmail.com} \\ \texttt{lizhen01@umd.edu} \quad \texttt{kwlee@inha.ac.kr}
}

\begin{document}

\maketitle

\begin{abstract}
We study posterior contraction rates for sparse Bayesian Kolmogorov-Arnold networks (KANs) over anisotropic Besov spaces, providing a statistical foundation for KANs from a Bayesian point of view.
We show that sparse Bayesian KANs equipped with spike-and-slab-type sparsity priors attain a near-minimax posterior contraction rate. In particular, the contraction rate depends on the intrinsic anisotropic smoothness of the underlying function.
Moreover, by placing a hyperprior on a single model-size parameter, the resulting posterior adapts to unknown anisotropic smoothness and still achieves the corresponding near-minimax rate. 
A distinctive feature of our results, compared with those for standard multilayer perceptrons, is that the KAN depth can be kept fixed: owing to the expressive power of learnable spline edge functions, the required approximation complexity is controlled through the network width, the range and size of the spline grid, and parameter sparsity.
Our analysis develops theoretical tools tailored to sparse spline-edge architectures, including approximation and complexity bounds for Bayesian KANs.  We then extend to compositional Besov spaces and show that the contraction rates depend on \emph{layerwise smoothness and effective dimension} of the underlying compositional structure, thereby effectively avoiding the curse of dimensionality. 
Together, these results advance the theoretical understanding of Bayesian neural networks and establish rigorous Bayesian guarantees for KANs.
\end{abstract}

\section{Introduction}

Kolmogorov-Arnold Networks are a recently introduced alternative to multilayer perceptrons (MLPs) \citep{liu2024kan, liu2024kan2}.
This design is motivated by the Kolmogorov-Arnold representation theorem, which states that any continuous multivariate function can be expressed through sums and compositions of suitable univariate functions. By representing the edges of the network as linear combinations of univariate B-splines, KANs offer greater modeling flexibility than conventional fully connected neural networks that use fixed activation functions. KANs are also more interpretable than MLPs, since each edge function can be visualized and  interpreted directly.
Motivated by these properties, KANs have recently been applied in a wide range of domains, including time series analysis \citep{vaca2024kolmogorov}, graph learning \citep{kiamari2024gkan, bresson2025kagnn}, and differential equations \citep{koenig2024kan}.

In high-dimensional nonparametric regression, statistical performance is strongly affected by both the ambient dimension and the regularity structure of the underlying regression function. The deterioration of estimation rates with dimension, the so-called curse of dimensionality, is a familiar manifestation of this dependence. In addition, when the smoothness of the target function varies across coordinates or the function exhibits local irregularities, isotropic smoothness classes such as H\"older or Sobolev spaces may fail to capture the effective complexity of the problem. Anisotropic Besov spaces offer a flexible alternative by encoding direction-dependent regularity through a smoothness vector \citep{kerkyacharian2001nonlinear,suzuki2021deep}; the corresponding contraction rates in this paper are governed by the intrinsic smoothness which will be presented later.

Recent works have established near-minimax guarantees for Bayesian deep learning procedures over anisotropic function spaces. 
For fractional posteriors, \citet{egels2025posterior} proved adaptive near-minimax contraction rates for deep ReLU networks with heavy-tailed priors over several function classes, including anisotropic Besov spaces. 
For standard posteriors, \citet{lee2025posterior} showed that sparse Bayesian neural networks equipped with sparsity-inducing priors attain near-optimal contraction rates over anisotropic Besov and composite anisotropic Besov spaces, together with adaptive results through hierarchical priors.
For related results on (near-)minimax guarantees for multilayer perceptrons (MLPs), see \citet{suzuki2018adaptivity, schmidt2020nonparametric, suzuki2021deep} for frequentist approaches and \citet{polson2018posterior, lee2022asymptotic, kong2025posterior} for Bayesian approaches.

As mentioned, the accumulated body of works has been developed for ReLU-based MLPs, and it remains unclear whether analogous results can be established for spline-edge architectures such as KANs.
Recently, \citet{kratsios2026approximation} established optimal approximation rates for residual KANs in Besov norms, together with learning guarantees based on complexity bounds, providing theoretical support for the expressive power of KANs. However, to the best of our knowledge, the rigorous theoretical support for Bayesian KANs with respect to the posterior contraction rates has not previously been established. We aim to close such gaps in recent literature by proposing sparse Bayesian KANs and providing the rigorous theory.

One of the key features of our work is that fixed-depth KANs attain the approximation and near-optimal posterior contraction rates. However, in the aforementioned studies on ReLU-based MLPs, the depth of the networks in order to obtain theoretical guarantees should grow with the sample size, typically, at an order of $\log n$. Therefore, our conclusion suggests one of promising features of KANs for expressive efficiency \citep{wang2025expressiveness}.

To summarize, our main contributions are threefold.
First, we develop fundamental results for sparse Bayesian KANs. 
We prove that sparse KANs approximate functions over anisotropic Besov spaces.
Then, we propose sparse Bayesian KANs with sparsity-inducing priors including spike-and-slab and Bernoulli spike-and-slab and prove that KANs attain near-minimax posterior contraction rates, with the rate depending on the intrinsic smoothness; our construction keeps the KAN depth fixed and controls approximation complexity through sparse spline-edge representations rather than depth growth or increasingly fine grid spacing.
Second, we establish adaptation to unknown anisotropic smoothness by placing a hyperprior on a single model-size parameter that jointly controls the width, grid size, and grid range of the KAN architecture.
Third, we extend the approximation theorem to compositional anisotropic Besov spaces, obtaining rates governed by layerwise smoothness and effective dimension.

\section{Model assumptions and setup}

\paragraph{Notation.} Let $\bbN$ and $\bbR$ denote the sets of natural numbers and real numbers, respectively.
Moreover, let $\bbR_+$ denote the set of nonnegative real numbers and $\bbR_{++}$ the set of positive real numbers.
For $a \in \bbR$, let $\lfloor a \rfloor$ and $\lceil a \rceil$ be the floor and ceiling functions. For $n \in \bbN$, let $[n]$ denote the set $\{1, 2, \ldots, n \}$.
For each $a, b \in \bbR$, write $a \vee b := \max \{ a, b \}$ and $a \wedge b := \min \{a, b \}$.
For a real-valued vector $\bv$, denote its $\ell_p$-norm by $\| \bv \|_p$, $1 \leq p \leq \infty$, and let $\| \bv \|_0$  be  the number of nonzero coordinates. For a measurable function $f:[0,1]^d \rightarrow \bbR$ and a measure $\mu$ on $[0,1]^d$ with $d \in \bbN$, define
$\| f \|_{L^p(\mu)} := ( \int_{[0,1]^d} | f(x)|^p d \mu(x) )^{1/p}$ for $0 < p < \infty$.
When $p = \infty$, define $\| f \|_\infty := \sup_{x \in [0, 1]^d} | f(x) |$.
Let $\delta_0$ denote the point mass at $0$. For two sequences $a_n$, $b_n \geq 0$, if $a_n \leq C b_n$ holds for some universal constant $C > 0$, then we write $a_n \lesssim b_n$; equivalently, we write $b_n \gtrsim a_n$. If $a_n \lesssim b_n \lesssim a_n$, then we write $a_n \asymp b_n$.
For convenience, define the uniformly bounded function class $\UB := \{ f:[0,1]^d \rightarrow \bbR : \| f \|_\infty \leq 1 \}$. For a normed space $\cF$, let $U(\cF)$ denote the unit ball of $\cF$. Let $\cN(\epsilon, \cF, \norm{\cdot})$ denote the $\epsilon$-covering number of the set $\cF$ with respect to the norm $\norm{\cdot}$.

\paragraph{Regression model.}
We observe data $\mathcal{D}_n=\{(X_i,Y_i)\}_{i=1}^n$ with $(X_i,Y_i) \in [0,1]^{d} \times \bbR$, from  the following random design nonparametric regression model: 
\begin{equation*}
Y_i = f_0(X_i) + \varepsilon_i, \qquad Y_i \, | \, X_i \stackrel{ind}{\sim} N(f_0(X_i),\sigma_0^2), \qquad X_i \stackrel{iid}{\sim} P_X
\end{equation*}
where $f_0:[0,1]^d\to\mathbb{R}$ is the true regression function, and $\sigma_0^2>0$ is the variance of the observation noise.
The distribution $P_X$ denotes the distribution of the design points $X_i$.

\subsection{Kolmogorov--Arnold Networks}

\subsubsection{B-splines}
Following the original formulation of \citet{liu2024kan}, we assume that each edge in the network is given by a sum of a B-spline expansion and a SiLU activation.
B-spline bases admit several constructions depending on the knot, or grid, specification such as allowing repeated or nonuniform knots \citep{schumaker2007spline}.
In this work, we restrict attention to \emph{fixed-knot KANs}, whose edges are built from B-splines on a fixed, equally spaced knot grid with no repeated knots.

Specifically, let $m \in \bbN$ be the spline degree and let $G \in \bbN$ be the number of grid intervals.
For an extended interval $[a, b]$, define a knot vector $\boldsymbol{\xi} = (\xi_{-m}, \cdots, \xi_{G + m} ) \in [a, b]^{G + 2m + 1}$  whose entries are strictly increasing and equally spaced, i.e. $\xi_k < \xi_{k + 1}$ and $\Delta = \xi_{k + 1} - \xi_k$, $\forall k$. We denote by $\{ B_{k,m}(\cdot; \bxi) \}_{k=1}^{G + m}$ the collection of degree $m$ B-spline basis functions generated by $\bxi$ \citep{de2001practical}.
We take $[\xi_0, \xi_G]$ as the estimation domain, while the extended interval $[a, b]$ is used to mitigate boundary effects.

\subsubsection{KAN specification}
We proceed with introducing the KAN architecture by defining each edge of KANs. Let $L$
 denote the depth of the network, $\bD=(D_0,\ldots,D_L)$  its layer widths, and $\bG=(G_0,\ldots,G_{L - 1})$ the per-layer numbers of grid intervals, with $D_0=d$ and $D_L=1$. 
For each $l = 0, \ldots, L - 1$, the edge function $\phi_{i,j}^{(l)}:\bbR \rightarrow \bbR$ connecting node $j \in [D_{l}]$ in layer $l$ to node $i \in [D_{l + 1}]$ in layer $l + 1$ is defined by
\begin{equation}    \label{eq:KANs_edge}
\phi_{i,j}^{(l)}(x)
=
\sum_{k=1}^{G_l + m}\theta_{i,j,k}^{(l)} B_{k,m}(x; \bxi^{(l)})
+
\theta_{i,j,0}^{(l)}\,\operatorname{SiLU}(x),  \quad x \in \bbR
\end{equation}
where $B_{k,m}(x; \bxi^{(l)})$ denotes the B-spline basis generated by the  layer $l$ knots $\bxi^{(l)} \in [a_l, b_l]^{G_l + 2m + 1}$, and $\operatorname{SiLU}(x)=x/(1+e^{-x})$. We set $[a_l, b_l] = [-H, H]$, $l = 1, \ldots, L - 1$ for $H > 0$ and if $l=0$, any fixed points $a_0  < \xi_0^{(0)} < \xi_{G_0}^{(0)} < b_0$ such that $[\xi_0^{(0)}, \xi_{G_0}^{(0)}]^d \supseteq [0, 1]^d$.  Each $\theta_{i,j,k}^{(l)} \in \bbR$ is the corresponding edge weight, and we collect the weights of edge $(i,j)$ at layer $l$ into the vector
$\btheta^{(l)}_{i,j} = (\theta_{i,j,0}^{(l)}, \ldots, \theta_{i,j,G_l + m}^{(l)} ) \in \bbR^{G_l + m + 1}$. The transformation $\Phi_l : \bbR^{D_l} \rightarrow \bbR^{D_{l + 1}}$ at layer $l$ is then given by
\begin{equation}    \label{eq:KANs_layer}
\Phi_l (\bx)
=
\Bigl(
\sum_{j=1}^{D_l}\phi_{1,j}^{(l)}(x_j),
\dots,
\sum_{j=1}^{D_l}\phi_{D_{l + 1},j}^{(l)}(x_j)
\Bigr)^\top
\end{equation}
with weight tensor $\bW^{(l)} = (\btheta_{i,j}^{(l)})_{i,j} \in \bbR^{D_{l + 1} \times D_l \times (G_l + m + 1)}$. The Kolmogorov-Arnold network  (KAN) $f_{\btheta}:\bbR^d \rightarrow \bbR$ is then defined as
$
f_{\boldsymbol{\theta}}
=
\Phi_{L-1}\circ\cdots\circ\Phi_0
$
where $\btheta = (\bW^{(l)})_l = (\theta_{i,j,k}^{(l)})_{i,j,k,l}$ collects all the  parameters; we also use the 
vectorized representation $\btheta = (\theta_1, \ldots, \theta_T)$ when convenient. The  total number of parameters is $T(L, \bD, \bG) = \sum_{l=0}^{L - 1} D_{l + 1} D_l (G_l + m + 1)$. For  a fixed structure $(L, \bD, \bG, H)$, where $L \in \bbN$, $\bD \in \bbN^{L + 1}$, $\bG \in \bbN^{L}$, $H > 0$, and parameters $\boldsymbol{\theta} = (\theta_{i,j,k}^{(l)})_{i,j,k,l}$, we define the class of KANs of order $m$ :
\begin{equation}    \label{eq:KANs_architecture}
    KAN(L, \bD, \bG, H ; m)
    := \left\{ f_{\boldsymbol{\theta}} = \Phi_{L - 1} \circ \cdots \circ \Phi_0 \middle| 
        \begin{array}{l}
        \btheta = (\bW^{(l)})_l \; \text{such that}    \\
        \Phi_l \text{ is defined by \cref{eq:KANs_edge}-\cref{eq:KANs_layer}},   \\
        \bW^{(l)} \in \bbR^{D_{l + 1} \times D_l \times (G_l + m + 1)}, \; \forall l
    \end{array}
    \right\}.
\end{equation}

For the posterior contraction analysis, we further specialize to KANs with equal hidden width and a uniform fixed-knot grid: for $D, G \in \bbN$ and $H > 0$,
\begin{equation*}
    \mathbf{D}=(d,D,\dots,D,1), \quad \mathbf{G}=(G_0, G,\dots,G), \quad [a_l, b_l] = [-H, H] \; \text{for} \; l = 1, \ldots, L - 1
\end{equation*}
where $G_0$ is fixed. For notational simplicity, we write KANs with equal-width and uniform fixed-knot structure with parameter $\btheta$ as $f_{\btheta} \in KAN(L,D,G,H)$, dropping the boldface to indicate $D$ and $G$ are now scalars. Under this setting, the total number of parameters is given by $T = T(L,D,G) = d D (G_0 + m  + 1) + ((L - 2)D^2 + D) (G + m + 1 )$.

% Sparse fixed-knot B-spline KAN schematic.
% Requires: tikz with positioning, arrows.meta, calc.
\begin{figure}[t]
\centering
\begin{tikzpicture}[
    x=1cm,
    y=1cm,
    >=Stealth,
    font=\footnotesize,
    panel/.style={draw=black!25, fill=black!2, rounded corners=3pt, line width=0.45pt},
    title/.style={font=\bfseries\footnotesize, align=center},
    tinylabel/.style={font=\scriptsize, align=center},
    formula/.style={font=\scriptsize, align=center},
    nodept/.style={circle, fill=black, inner sep=1.45pt},
    activeedge/.style={draw=black!88, line width=0.62pt},
    inactiveedge/.style={draw=black!88, line width=0.62pt},
    % inactiveedge/.style={draw=black!20, dashed, line width=0.42pt},
    omitted/.style={draw=black!45, line width=0.58pt, line cap=round, dash pattern=on 0.55pt off 2.0pt},
    flow/.style={-Stealth, line width=0.75pt, draw=black!65},
    zoomflow/.style={-Stealth, line width=0.70pt, draw=black!55},
    coefzero/.style={draw=black!20, fill=black!6, line width=0.35pt},
    basis/.style={draw=black!32, line width=0.48pt},
    basisselect/.style={draw=blue!65!black, line width=0.72pt},
    resultcurve/.style={draw=orange!85!black, line width=1.05pt},
    silucurve/.style={draw=green!35!black, opacity=0.58, line width=0.85pt},
    edgemask/.style={draw=black!35, fill=white, minimum width=0.46cm, minimum height=0.32cm, inner sep=0pt},
    callout/.style={draw=orange!80!black, line width=0.75pt}
]

% ------------------------------------------------------------------
% Layout: same width, increased height.
% ------------------------------------------------------------------
\draw[panel] (0,0) rectangle (7.15,8.65);
\draw[panel] (7.55,4.00) rectangle (13.75,8.65);
\draw[panel] (7.55,0) rectangle (13.75,3.75);

\node[title] at (3.575,8.32) {1. KAN};
\node[title] at (10.65,8.32) {2. Sparse edge function};
\node[title] at (10.65,3.42) {3. B-spline basis};

% ------------------------------------------------------------------
% Panel 1: KAN with active and inactive edges.
% ------------------------------------------------------------------
\begin{scope}[xshift=0.42cm]
\node[nodept, label=below:{\scriptsize $x_1$}] (x1) at (1.00,0.92) {};
\node[nodept, label=below:{\scriptsize $x_2$}] (x2) at (2.35,0.92) {};
\draw[omitted] (2.85,0.92) -- (4.25,0.92);
\node[nodept, label=below:{\scriptsize $x_d$}] (xd) at (4.85,0.92) {};

\node[nodept] (h11) at (0.90,2.25) {};
\node[nodept] (h12) at (2.05,2.25) {};
\draw[omitted] (2.50,2.25) -- (3.82,2.25);
\node[nodept] (h13) at (4.25,2.25) {};
\node[nodept] (h14) at (5.40,2.25) {};

\node[nodept] (h21) at (1.20,3.72) {};
\node[nodept] (h22) at (2.70,3.72) {};
\draw[omitted] (3.15,3.72) -- (4.62,3.72);
\node[nodept] (h23) at (5.05,3.72) {};

% \draw[omitted] (1.20,4.12) -- (1.20,5.20);
% \draw[omitted] (2.70,4.12) -- (2.70,5.20);
% \draw[omitted] (5.05,4.12) -- (5.05,5.20);

\node[nodept] (hm1) at (1.35,5.60) {};
\node[nodept] (hm2) at (3.20,5.60) {};
\node[nodept] (hm3) at (5.10,5.60) {};

\node[nodept] (out) at (3.20,6.82) {};
\node[tinylabel, anchor=south] at (3.20,6.94) {$f_{\boldsymbol{\theta}}(\boldsymbol{x})$};

% Inactive candidate edges.
\draw[inactiveedge] (x1) -- (h12);
\draw[inactiveedge] (x1) -- (h14);
\draw[inactiveedge] (x2) -- (h11);
\draw[inactiveedge] (x2) -- (h14);
\draw[inactiveedge] (xd) -- (h11);
\draw[inactiveedge] (xd) -- (h12);
\draw[inactiveedge] (xd) -- (h13);
\draw[inactiveedge] (h11) -- (h22);
\draw[inactiveedge] (h12) -- (h23);
\draw[inactiveedge] (h13) -- (h21);
\draw[inactiveedge] (h13) -- (h23);
\draw[inactiveedge] (h14) -- (h21);
\draw[inactiveedge] (h14) -- (h22);
\draw[inactiveedge] (h11) -- (h23);
\draw[inactiveedge] (h12) -- (h21);
\draw[inactiveedge] (h12) -- (h22);
\draw[inactiveedge] (h13) -- (h23);
\draw[inactiveedge] (h14) -- (h23);

% Active sparse KAN edges.
\draw[activeedge] (x1) -- (h11);
\draw[activeedge] (x1) -- (h13);
\draw[activeedge] (x2) -- (h12);
\draw[activeedge] (x2) -- (h13);
\draw[activeedge] (xd) -- (h14);
\draw[activeedge] (h11) -- (h21);
\draw[activeedge] (h12) -- (h21);
\draw[activeedge] (h13) -- (h22);
\draw[activeedge] (h14) -- (h23);
\draw[activeedge] (hm1) -- (out);
\draw[activeedge] (hm2) -- (out);
\draw[activeedge] (hm3) -- (out);

\draw[omitted, shorten <=15pt, shorten >=15pt] (h21) -- (hm1);
\draw[omitted, shorten <=15pt, shorten >=15pt] (h21) -- (hm2);
\draw[omitted, shorten <=15pt, shorten >=15pt] (h21) -- (hm3);
\draw[omitted, shorten <=15pt, shorten >=15pt] (h22) -- (hm1);
\draw[omitted, shorten <=15pt, shorten >=15pt] (h22) -- (hm2);
\draw[omitted, shorten <=15pt, shorten >=15pt] (h22) -- (hm3);
\draw[omitted, shorten <=15pt, shorten >=15pt] (h23) -- (hm1);
\draw[omitted, shorten <=15pt, shorten >=15pt] (h23) -- (hm2);
\draw[omitted, shorten <=15pt, shorten >=15pt] (h23) -- (hm3);
\draw[omitted, shorten <=15pt, shorten >=15pt] (hm1) -- (hm2);
\draw[omitted, shorten <=15pt, shorten >=15pt] (hm2) -- (hm3);

% Each small box denotes a learned sparse univariate edge function.
\foreach \A/\B/\pos in {
    x1/h11/0.50,
    x2/h13/0.54,
    xd/h14/0.48,
    h12/h21/0.50,
    h13/h22/0.50,
    hm1/out/0.52,
    hm2/out/0.50,
    hm3/out/0.52,
    x1/h13/0.50,
    h11/h21/0.50,
    h14/h23/0.50
} {
    \node[edgemask] (box-\A-\B) at ($(\A)!\pos!(\B)$) {};
    \draw[black!25, line width=0.22pt] ($(box-\A-\B.center)+(-0.18,-0.08)$) -- ($(box-\A-\B.center)+(0.18,-0.08)$);
    \draw[black!25, line width=0.25pt] ($(box-\A-\B.center)+(-0.17,-0.08)$)
        .. controls ($(box-\A-\B.center)+(-0.10,0.03)$) and ($(box-\A-\B.center)+(-0.06,0.06)$) ..
        ($(box-\A-\B.center)+(0.00,-0.08)$);
    \draw[blue!65!black, line width=0.35pt] ($(box-\A-\B.center)+(-0.07,-0.08)$)
        .. controls ($(box-\A-\B.center)+(-0.02,0.06)$) and ($(box-\A-\B.center)+(0.04,0.06)$) ..
        ($(box-\A-\B.center)+(0.09,-0.08)$);
    \draw[orange!85!black, line width=0.36pt] ($(box-\A-\B.center)+(-0.17,-0.02)$)
        .. controls ($(box-\A-\B.center)+(-0.04,0.07)$) and ($(box-\A-\B.center)+(0.08,-0.06)$) ..
        ($(box-\A-\B.center)+(0.17,0.04)$);
}

% Call out one active edge function for zoom-in.
\draw[callout] (box-hm3-out.center) circle[radius=0.36];

% Width bracket, shown as an architecture-level constraint.
\draw[black!48, line width=0.42pt] (0.65,7.32) -- (0.65,7.5) -- (5.60,7.5) -- (5.60,7.32);
\node[tinylabel, fill=black!2, inner sep=1pt] at (3.225,7.71) {max width $D$};

\node[tinylabel] at (0.38,1.58) {$\Phi_0$};
\node[tinylabel] at (0.45,3.12) {$\Phi_1$};
% \draw[omitted] (0.62,4.35) -- (0.62,5.00);
\node[tinylabel] at (1.15,6.25) {$\Phi_{L-1}$};
% \node[tinylabel] at (2.72,4.02) {sum};
% \node[tinylabel] at (3.20,5.86) {sum};

% \draw[activeedge] (1.08,0.22) -- (1.52,0.22);
% \node[tinylabel, anchor=west] at (1.60,0.22) {active edge};
% \draw[inactiveedge] (3.12,0.22) -- (3.56,0.22);
% \node[tinylabel, anchor=west] at (3.64,0.22) {inactive edge};
\end{scope}

\draw[zoomflow] ($(box-hm3-out.center)+(0.36,0.02)$)
    .. controls (6.55,6.35) and (7.05,7.35) .. (7.48,7.45);

% ------------------------------------------------------------------
% Panel 2: one sparse edge function with bounded coefficients.
% ------------------------------------------------------------------
\begin{scope}[yshift=-0.24cm]
\node[formula] at (10.65,7.92)
{$\phi_{i,j}^{(l)}(u)=\theta_{i,j,0}^{(l)}\mathrm{silu}(u)+s_{i,j}^{(l)}(u)$};
\node[
    draw=orange!80!black,
    rounded corners=2pt,
    inner xsep=2.4pt,
    inner ysep=1.3pt,
    font=\scriptsize,
    align=center
] (sumterm) at (10.65,7.48)
{$\textstyle s_{i,j}^{(l)}(u)=
\sum_{k=1}^{G_l+m}\theta_{i,j,k}^{(l)}B_{k,m}(u;\boldsymbol{\xi}^{(l)})$};
% \node[tinylabel] at (10.65,6.98) {$\|\boldsymbol{\theta}_{i,j}^{(l)}\|_\infty\le B$};

\draw[black!55, line width=0.40pt] (8.75,6.42) -- (12.58,6.42);
\draw[black!35, dashed, line width=0.35pt] (8.75,6.76) -- (12.58,6.76);
\draw[black!35, dashed, line width=0.35pt] (8.75,6.08) -- (12.58,6.08);
\node[tinylabel, anchor=west] at (12.70,6.76) {$B$};
\node[tinylabel, anchor=west] at (12.70,6.42) {$0$};
\node[tinylabel, anchor=west] at (12.70,6.08) {$-B$};

\foreach \i in {0,...,8} {
    \pgfmathsetmacro{\xx}{9.00+0.35*\i}
    \draw[black!25, line width=0.35pt] (\xx,6.35) -- (\xx,6.49);
    \node[coefzero, minimum width=0.20cm, minimum height=0.20cm, inner sep=0pt] at (\xx,6.42) {};
}
\draw[black!72, line width=1.15pt] (9.75,6.42) -- (9.75,6.72);
\draw[black!72, line width=1.15pt] (10.75,6.42) -- (10.75,6.14);
\draw[black!72, line width=1.15pt] (11.75,6.42) -- (11.75,6.73);
% \node[tinylabel, fill=black!2, inner sep=1pt] at (10.65,5.70) {$\|\boldsymbol{\theta}_{i,j}^{(l)}\|_0\le S$};

\draw[black!60, line width=0.40pt] (8.90,4.65) -- (12.55,4.65);
\foreach \c in {9.25,9.75,10.25,10.75,11.25,11.75,12.25} {
    \draw[basis] (\c-0.38,4.65)
        .. controls (\c-0.18,4.85) and (\c-0.13,5.16) .. (\c,5.16)
        .. controls (\c+0.13,5.16) and (\c+0.18,4.85) .. (\c+0.38,4.65);
}
\foreach \c in {9.75,10.75,11.75} {
    \draw[basisselect] (\c-0.38,4.65)
        .. controls (\c-0.18,4.85) and (\c-0.13,5.16) .. (\c,5.16)
        .. controls (\c+0.13,5.16) and (\c+0.18,4.85) .. (\c+0.38,4.65);
}
\draw[silucurve, domain=8.95:12.55, samples=120, smooth, variable=\x]
    plot ({\x},{4.60 + 0.10*(\x-8.95) + 0.38/(1+exp(-2.0*(\x-10.95)))});
\node[tinylabel, text=green!35!black, opacity=0.70] at (12.30,5.50) {silu};
\draw[resultcurve, domain=8.95:12.55, samples=160, smooth, variable=\x]
    plot ({\x},{4.60 + 0.10*(\x-8.95) + 0.38/(1+exp(-2.0*(\x-10.95)))
        + 0.50*pow(max(0,1-pow((\x-9.75)/0.42,2)),2)
        - 0.46*pow(max(0,1-pow((\x-10.78)/0.50,2)),2)
        + 0.44*pow(max(0,1-pow((\x-11.75)/0.42,2)),2)});
\end{scope}

% Zoom from the B-spline sum term to the basis panel.
\draw[zoomflow] (sumterm.east)
    .. controls (13.70,7.48) and (13.70,4.18) .. (13.28,3.82);

% ------------------------------------------------------------------
% Panel 3: fixed B-spline basis on [-H,H].
% ------------------------------------------------------------------
\node[formula] at (10.65,2.98)
{$\xi_{-m}^{(l)}=-H,\quad \xi_{G_l+m}^{(l)}=H,\qquad
\{B_{k,m}\}_{k=1}^{G_l+m}$};

\draw[black!70, line width=0.45pt] (8.85,1.67) -- (12.55,1.67);
\foreach \i in {0,...,9} {
    \pgfmathsetmacro{\xx}{8.85+0.41*\i}
    \draw[black!70, line width=0.45pt] (\xx,1.60) -- (\xx,1.74);
}
\node[tinylabel, anchor=north] at (8.85,1.53) {$-H$};
\node[tinylabel, anchor=north] at (9.67,1.53) {$\xi_0^{(l)}$};
\node[tinylabel, anchor=north] at (11.72,1.53) {$\xi_{G_l}^{(l)}$};
\node[tinylabel, anchor=north] at (12.55,1.53) {$H$};
\node[tinylabel] at (10.75,1.41) {$\cdots$};

\foreach \c in {9.25,9.66,10.07,10.48,10.89,11.30,11.71,12.12} {
    \draw[basis] (\c-0.41,1.67)
        .. controls (\c-0.20,1.95) and (\c-0.14,2.35) .. (\c,2.35)
        .. controls (\c+0.14,2.35) and (\c+0.20,1.95) .. (\c+0.41,1.67);
}

\node[formula] at (10.65,0.55)
{$\displaystyle s_{i,j}^{(l)}(u)=
\sum_{k=1}^{G_l+m}\theta_{i,j,k}^{(l)}B_{k,m}(u;\boldsymbol{\xi}^{(l)})$};

\end{tikzpicture}
% \caption{Sparse fixed-knot B-spline KAN construction. The left panel shows a sparse KAN with layer widths bounded by $D$; active edges carry learnable univariate functions and inactive edges are shown as dashed gray lines. The upper-right panel zooms in on one edge function with coefficient-level sparsity, the uniform coefficient bound $B$, and a SiLU component. The lower-right panel zooms further into the B-spline basis used in the spline part of the edge function.}
\caption{Sparse fixed-knot B-spline KAN construction. The left panel shows a KAN with layer widths bounded by $D$; edges carry learnable univariate functions. The upper-right panel zooms in on one edge function with coefficient-level sparsity, the uniform coefficient bound $B$, and a SiLU component. The lower-right panel zooms further into the B-spline basis used in the spline part of the edge function.}
\label{fig:sparse-kan-dag}
\end{figure}

\subsection{Anisotropic Besov spaces}
In this section, we introduce anisotropic Besov spaces, the function class to which  we assume the true function belongs. The anisotropic modulus of smoothness is defined similarly to that of  the isotropic Besov space case, but by taking each direction into account \citep{nikol2012approximation}.  For a function $f: \bbR^d \rightarrow \bbR$,  $r$th-order finite difference of $f$ in direction $\bh \in \bbR^d$ is defined by $\Delta_h^r (f) (\bx) = \sum_{j=0}^r \binom{r}{j} (-1)^{r - j} f(\bx + j \bh)$ if $\bx + r \bh \in [0, 1]^d $ or $\Delta_h^r (f) (\bx) = 0$ otherwise. Given a directional radius $\bt = (t_1, \ldots, t_d) \in \bbR^d_{++}$, the anisotropic modulus of smoothness of a function $f \in L^p$ is defined by $w_{r,p} (f, t) = \sup_{\bh \in \bbR^d : |h_i| \leq t_i } \| \Delta^r_h (f) \|_{L^p}$.
\begin{definition}[Anisotropic Besov Space: $\cB^{\bfs}_{p,q}( \lbrack 0,1 \rbrack^d)$]
For $p,q \in (0,\infty]$ and $\bfs=(s_1,\ldots,s_d)\in\bbR_{++}^d$, let $r := \max_i \lfloor s_i \rfloor + 1$. For a function $f \in L^p$, the anisotropic Besov seminorm $| \cdot |_{\cB^{\bfs}_{p,q}}$ is
\begin{align*}
    | f |_{\cB^{\bfs}_{p,q}} :=
    \begin{cases}
        \left( \sum_{k=0}^\infty [ 2^k w_{r,p} (f, (2^{-k/s_1}, \ldots, 2^{-k/s_d})) ]^q \right)^{1/q},     &   q < \infty  \\
        \sup_{k \geq 0} [ 2^k w_{r,p} (f, (2^{-k/s_1}, \ldots, 2^{-k/s_d})) ],       &   q = \infty.
    \end{cases}
\end{align*}
The norm on the space $\cB^\bfs_{p,q} ([0,1]^d)$ is defined by $\| f \|_{\cB^{\bfs}_{p,q}} := \| f \|_{L^p} + | f |_{\cB^{\bfs}_{p,q}}$, and $\cB^\bfs_{p,q} ([0,1]^d) := \{f \in L^p : \| f \|_{\cB^{\bfs}_{p,q}} < \infty \}$.
\end{definition}
The anisotropic smoothness vector $\bfs$
can be summarized by a single \emph{intrinsic smoothness} index
$\tbfs := ( \sum_{j=1}^d s_j^{-1} )^{-1},$
which captures the joint regularity of $f$ across all coordinates and governs the minimax convergence rate. Let  $\underline{\bfs} := \min_j s_j $ be the smallest smoothness among all directions, and $\overline{\bfs} := \max_j s_j$ the largest smoothness. Intuitively, a small $s_j$
indicates that $f$ is rough along the $j$th coordinate, while a large $s_j$ indicates higher regularity  or smoothness in that direction. The isotropic case $s_1 = \cdots = s_d = s$ recovers the standard isotropic Besov space, in which case $\tbfs = s / d$.

\begin{remark}[Intrinsic smoothness and intrinsic dimension]
\label{rem:intrinsic-dimension}
The intrinsic smoothness $\tbfs$ can be equivalently expressed through the anisotropy-induced dimension $d_{\mathrm{int}}(\bfs):= \underline{\bfs}/{\tbfs}= \sum_{j=1}^d \underline{\bfs}/{s_j}\le d.$
Indeed, the classical anisotropic minimax rate \citep{kerkyacharian2001nonlinear, hoffman2002random} satisfies $n^{-\tbfs/(2\tbfs+1)}=n^{-\underline{\bfs}/(2\underline{\bfs}+d_{\mathrm{int}}(\bfs))}.$
Thus $d_{\mathrm{int}}(\bfs)$ is the dimension index induced by anisotropy: a roughest coordinate contributes one full dimension, while a smoother coordinate contributes only the fraction $\underline{\bfs}/s_j$. In the isotropic case $d_{\mathrm{int}}(\bfs)=d$, and the layerwise quantity $t^{*(j)}=\underline{\bfs}^{(j)}/\tilde{s}^{(j)}$ below is the same index for each compositional component.
\end{remark}

\section{Main results}  \label{sec:main_results}

In this section, we present posterior contraction rate results for sparse Bayesian KANs over anisotropic Besov spaces.
For the main results, we impose the following assumptions.
\paragraph{Assumptions.}
\begin{enumerate}[
    label=\textbf{(A\arabic*)},
    ref=A\arabic*
]
\item \label{assump:A1}
For some $0<p,q\le\infty$ and $\mathbf{s}\in\mathbb{R}_{++}^d$, the true function satisfies $f_0 \in U(\cB_{p,q}^{\mathbf{s}})\cap\UB$ and $(1/p-1/2)_+< \tbfs$.
% For the Bayesian model, assume clipped KAN output 
In the likelihood,
\[
Y_i | X_i \sim N(\operatorname{clip} \circ f(X_i), \sigma^2),
\]
where $\operatorname{clip} (x) := \min \{1, \max\{-1, x\} \}$.

\item \label{assump:A2}
The design distribution $P_X$ is supported on $[0,1]^d$ and admits a bounded density function $p_X$, i.e. there exists a constant $C_X >0$ such that $\|p_X\|_{L^\infty} \le C_X$.

\item \label{assump:A3}
There exist constants $ 0<\underline\sigma<\overline\sigma<\infty $
such that $ \underline\sigma<\sigma_0<\overline\sigma $.

\item \label{assump:A4}
The spline degree of KANs satisfies $m \geq 2$ and $0 < \overline{\bfs} < \min \{ m, m - 1 + 1/p \} $.
\end{enumerate}

Assumption \eqref{assump:A1} is a standard condition adopted in theoretical studies over Besov spaces.
It means that the true function $f_0$ belongs to a bounded Besov class, and the condition $\tbfs > (1/p-1/2)_+$ is the minimal condition for $L^2$-integrability and for the embedding $\cB^{\tbfs}_{p,q} \hookrightarrow L^2$. Meanwhile, \eqref{assump:A2}--\eqref{assump:A3} are frequently used in the analysis of posterior contraction for Bayesian models.
Assumption \eqref{assump:A2} is natural once the inputs $X$ have been normalized to $[0,1]^d$ in the fixed-dimensional regime  $d = O(1)$.
Assumption \eqref{assump:A3} requires only mild prior knowledge of the range in which  $\sigma_0$ lies.  

\begin{remark}
    The uniform boundedness condition in \eqref{assump:A1} is imposed with bound $1$ only for simplicity. More generally, one may take any sufficiently large $F > 0$, define $\UB_F := \{f: [0, 1]^d \rightarrow \bbR : \| f \|_\infty \leq F \}$ and $\operatorname{clip}_F (x) := \min \{ F, \max \{-F, x \} \}$. We set $F = 1$ without loss of generality. The convention applies to \eqref{assump:D1}, to be introduced later.
\end{remark}

\paragraph{Proof strategy.}
All results below are proved through a common pipeline. Let $A_n := \{(f,\sigma) : \norm{f - f_0}_{L^2(P_X)} \le \epsilon_n/2,\ \abs{\sigma - \sigma_0} \le \epsilon_n/2\}$. By \cref{lem:sufficient}, the posterior contracts to $(f_0,\sigma_0)$ at rate $\epsilon_n$, i.e., for any $M_n \to \infty$ it asymptotically assigns no mass outside $M_n\epsilon_n$-neighborhoods of $(f_0,\sigma_0)$, provided that there exists a set of functions $\cF_n$ (a sieve) such that, for a sufficiently large $C > 0$,
\begin{equation*}
    \text{(i)}~\log \cN\bigl(\epsilon_n, \cF_n, \norm{\cdot}_{L^2(P_X)}\bigr) \lesssim n\epsilon_n^2,
    \quad
    \text{(ii)}~-\log \Pi_n(A_n) \lesssim n\epsilon_n^2,
    \quad
    \text{(iii)}~\Pi_n(\cF_n^c) = o\bigl(e^{-Cn\epsilon_n^2}\bigr).
\end{equation*}
We construct $\cF_n$ to consist of KANs of fixed depth whose parameters are sparse and bounded. For a model-size parameter $N$, \cref{lem:approx} constructs such a KAN with $O(N)$ nonzero parameters that exactly reproduces a B-spline approximation of $f_0$ with $L^2$-error $O(N^{-\tbfs})$. For (ii), \cref{lem:approx} and the Lipschitz bound in \cref{cor:KANs_lipschitz} ensure that every parameter in a small $\ell_\infty$-ball around that of the approximating KAN yields $f_{\btheta}$ within $\epsilon_n/2$ of $f_0$; the prior mass of this ball then bounds $-\log \Pi_n(A_n)$ by a multiple of $N \log n$. For (i), the same Lipschitz bound yields the covering bound of order $N \log n$ in \cref{lem:entropy}, so (i) holds whenever $N \log n \lesssim n\epsilon_n^2$. Condition (iii) follows from tail bounds of the prior on the parameters. Requiring both $N^{-\tbfs} \lesssim \epsilon_n$ and $N \log n \lesssim n\epsilon_n^2$ leads to $N \asymp n^{1/(2\tbfs+1)}$ and $\epsilon_n = n^{-\tbfs/(2\tbfs+1)}(\log n)^{1/2}$.

\subsection{Spike-and-slab prior}
The result in this section can be summarized as follows. 
For any function $f_0$ satisfying \eqref{assump:A1}, there exists an $m$-degree sparse KAN approximator
$ f^\star = f_{\btheta^\star} \in KAN (L_0, D_n, G_n, H_n ; m) $
%with architecture $(L_0, D_n, G_n, H_n)$ and with 
with parameter $\btheta^\star$ satisfying
$\| \btheta^\star \|_\infty \leq B_n^\star$ and $\| \btheta^\star \|_0 \leq S_n$,  attaining approximation error $\epsilon_n \asymp n^{- \tbfs/ (2 \tbfs + 1)}$.  Under a spike-and-slab prior, the posterior distribution contracts toward $f_0$ at a rate that matches this approximation error up to a logarithmic factor.  The optimal architecture of the given sparse KANs is as follows:
\begin{equation}    \label{eq:sas_arch}
    L_0 = 3 + 2 \lceil \log_2 d \rceil, \quad D_n = 2d N_n  , \quad G_n = 4 H_n - 2m, \quad H_n = \lceil 2B_n^\star \rceil
\end{equation}
with sparsity $S_n := \lceil S_0 N_n \rceil$ for some $S_0 > 0$, where $N_n := \lceil n^{ 1 / (2 \tbfs + 1)} \rceil$, $B_n^\star = B_0 N_n^{\beta}$ for some $B_0 > 0$, and 
$
    \beta = 
    (1 + \kappa) \max \{(1 / p - \tbfs)_+, \tbfs /\underline{\bfs} \}
$
with $\kappa = \frac{2 \omega}{\tbfs - \omega}$ if $p < 2$ and $\kappa = 0$ if $p \geq 2$ and $\omega = ( 1/p - 1/2 )_+$.

Equation \eqref{eq:sas_arch} highlights two structural features of the construction.
First, the network depth is held fixed at the finite value $L_0$, while approximation power is acquired through the width $D_n$, the grid size $G_n$, and the grid range $H_n$.
The induced grid spacing $2 H_n / (G_n + 2m) =1/2$
remains bounded away from zero, so the construction does not rely on increasingly fine grids. 
Second, the entire architecture parameters $(L_0, D_n, G_n, H_n, S_n)$ are governed by a single model-size parameter $N_n$, yielding a particularly parsimonious specification.
The weight bound $B^\star_n$ of the oracle network $f^\star$ also scales with  $N_n$;
since $B_n^\star$ enters the prior conditions stated below, the slab distribution must be adjusted accordingly.

Let  $T_n := T(L_0, D_n, G_n)$ be the total number of model parameters under the architecture \cref{eq:sas_arch}. The spike-and-slab (SAS) prior places a joint distribution on the weight vector $\btheta = (\theta_t)_{t=1}^{T_n} \in \bbR^{T_n}$, the sparsity index $\bgamma = (\gamma_t)_{t=1}^{T_n} \in \{0, 1 \}^{T_n}$  and the noise variance $\sigma^2$: 
\begin{align}    \label{eq:sas_prior}
\begin{split}
\pi(\btheta \mid \bgamma) &= \prod_{t=1}^{T_n}
\Bigl[
\gamma_t \widetilde{\pi}_{SL}(\theta_t)
+ (1-\gamma_t)\delta_0(\theta_t)
\Bigr],\\
\pi(\bgamma) &= \binom{T_n}{S_n}^{-1}
\bone\{\bgamma\in\{0,1\}^{T_n},\ \norm{\bgamma}_0=S_n\},\\
\pi(d\sigma^2) &=
\pi_{\sigma^2}(\sigma^2)\,d\sigma^2
\end{split}
\end{align}
Here, $\widetilde{\pi}_{SL}$ is a continuous slab density that governs the nonzero  coordinates of $\theta$ and $\pi_{\sigma^2}$ is the prior density function of the noise variance. The prior on $\bgamma$
 is uniform over the set of $T_n$ dimensional binary vectors with exactly $S_n$ active coordinates.

For the spike-and-slab prior \eqref{eq:sas_prior} to support our contraction results, we require
\begin{enumerate}[
    label=\textbf{(B\arabic*)},
    ref=B\arabic*
]
\item \label{assump:B1}
There exists a constant $c_1>0$ such that
$
    \log\inf_{\abs{u}\le B_n^\star + 1}\widetilde{\pi}_{SL}(u)\ge -c_1\log n
$
for all sufficiently large $n$.
\item \label{assump:B2}
There exist constants $c_2,c_3,\alpha>0$ and a sequence $\tau_n \lesssim B_n^\star$ such that for all $t>0$,
$
    \int_{\abs{u}>t}\widetilde{\pi}_{SL}(u)\,du
    \le
    c_2 \exp\!\left[-c_3\left(\frac{t}{\tau_n}\right)^\alpha\right].
$

\item \label{assump:B3}
$\pi_{\sigma^2}$ is supported on $[\underline{\sigma}^2,\overline{\sigma}^2]$, and there exists a constant $c_\sigma$ such that
$
    \inf_{\sigma\in[\underline{\sigma},\overline{\sigma}]}\pi_{\sigma^2}(\sigma^2)\ge c_\sigma>0.
$
\end{enumerate}

Conditions \eqref{assump:B1}–\eqref{assump:B3} reflect the prior properties needed for the general posterior contraction framework of \citet{ghosal2000convergence, ghosal2007convergence}.

\begin{theorem}[Spike-and-slab prior]\label{thm:sas} \label{thm:main}
Suppose Assumptions \eqref{assump:A1}–\eqref{assump:A4} hold, and that the architecture and prior of the sparse Bayesian KAN are specified as in \eqref{eq:sas_arch} and \eqref{eq:sas_prior}, with the prior satisfying \eqref{assump:B1}–\eqref{assump:B3}.  Then the posterior distribution contracts toward the true value $(f_0, \sigma_0^2)$ at the rate
\begin{equation}    \label{eq:rate}
    \epsilon_n = n^{- \frac{\tbfs}{2 \tbfs + 1}} (\log n )^{1/2}.
\end{equation}
In other words, for any sequence $M_n \rightarrow \infty$,
\begin{equation}    \label{eq:convergence}
    \Pi \left( (\btheta,\sigma) : \norm{ \operatorname
{clip} \circ f_{\btheta} -f_0}_{L^2(P_X)}+\abs{\sigma^2 - \sigma_0^2}
    > M_n\epsilon_n \;\middle|\; \cD_n \right)
    \to 0
\end{equation}
in $P_{f_0,\sigma_0}^{(n)}$-probability as $n \rightarrow \infty$.
\end{theorem}

\begin{proof}
    See \Cref{app:main-proof}.
\end{proof}

\begin{remark}
\cref{thm:sas} is a near-optimal result, it differs from known minimax convergence rate $n^{-\tbfs/(2\tbfs+1)}$ over anisotropic Besov function spaces only by the logarithmic factor $(\log n)^{1/2}$ \citep{kerkyacharian2001nonlinear, hoffman2002random}.
\end{remark}

The following are standard families of slab priors that satisfy \eqref{assump:B1}--\eqref{assump:B2}. In each case the slab scale  $\tau_n$ is tuned in accordance with the oracle weight bound  $B_n^\star \asymp n^{\beta/(2 \tbfs + 1)}$.

\begin{example}[Uniform slab distribution]
    Slab distribution $\tilde\pi_{SL} = U(-\tau_n , \tau_n)$ where $\tau_n = C_\tau \, n^{\beta / (2\tbfs + 1)} \geq B_n^\star + 1$ for some $C_\tau > 0$.
\end{example}

\begin{example}[Sub-Weibull slab distribution]
$\tilde\pi_{SL} = \mathrm{subW} ( 1/\alpha )$, $\alpha > 0$ \citep{vladimirova2020sub}.   When the density is given by
$
    \tilde\pi_{SL} (u) \propto \frac{1}{\tau_n} e^{- \left| \frac{u}{\tau_n} \right|^\alpha}$, 
$
    \tau_n = C_\tau \, n^{\frac{\beta}{2\tbfs + 1}}
$
for some $C_\tau > 0$, the prior distribution satisfies both \eqref{assump:B1}--\eqref{assump:B2}. The case $\alpha = 2$ corresponds to the Gaussian distribution, and the case $\alpha = 1$ corresponds to the Laplace distribution.
\end{example}

\subsection{Bernoulli spike-and-slab prior}
We now consider a Bernoulli variant of the spike-and-slab prior, in which the inclusion indicators  $\gamma_t$
 are i.i.d. Bernoulli draws rather than uniform on a fixed-cardinality support. Let $\rho_n \in (0, 1)$ denote the coordinatewise inclusion probability. 

 The Bernoulli spike-and-slab prior is then defined by
\begin{align}    \label{eq:ber_sas_prior}
\begin{split}
\pi(\btheta \mid \bgamma) &= \prod_{t=1}^{T_n}
\Bigl[
\gamma_t \widetilde{\pi}_{SL}(\theta_t)
+ (1-\gamma_t)\delta_0(\theta_t)
\Bigr],\\
\pi(\bgamma) &= \prod_{t=1}^{T_n}\rho_n^{\gamma_t}(1-\rho_n)^{1-\gamma_t},\\
\pi(d\sigma^2) &=
\pi_{\sigma^2}(\sigma^2)\,d\sigma^2
\end{split}
\end{align}
Specific admissible choices of $\rho_n$ are discussed in \cref{eq:ber_sas_rho} and \cref{rem:ber_sas_rho} below.

\begin{theorem}[Bernoulli spike-and-slab]\label{thm:main-ber}
 Suppose Assumptions \eqref{assump:A1}--\eqref{assump:A4} hold, that the architecture and prior of the sparse Bayesian KAN are specified as in \eqref{eq:sas_arch} and \eqref{eq:ber_sas_prior} with the slab and variance prior satisfying \eqref{assump:B1}--\eqref{assump:B3},  and that the inclusion probability $\rho_n$ satisfies
\begin{equation}    \label{eq:ber_sas_rho}
    T_n^{-1}\lesssim \rho_n,    \quad   \log \left( \frac{S_n}{T_n\rho_n} \right) \gtrsim \log n
\end{equation}
Then the posterior achieves the same rate in \eqref{eq:rate} and the result \eqref{eq:convergence} holds.
\end{theorem}
\begin{proof}
    See \Cref{app:main-ber-proof}.
\end{proof}
\begin{remark}  \label{rem:ber_sas_rho}
% Recall that under architecture \eqref{eq:sas_arch} we have  $S_n \asymp N_n$, $T_n \asymp N_n^{2+\beta}$, and
%     $N_n \asymp n^{1/(2\tbfs+1)}$.
Condition \cref{eq:ber_sas_rho} is satisfied, for instance, by inclusion probabilities $\rho_n$ such that, for some
    $0<c_\rho<1/(2\tbfs+1)$,
$
        N_n^{-(2+\beta)}
        \lesssim
        \rho_n
        \lesssim
        N_n^{-(1+\beta)} n^{-c_\rho}.
$
    On the other hand, $\rho_n\asymp T_n^{-1}$ is a particularly simple choice.
    In this case, $T_n\rho_n\asymp1$ and
$
        \log\frac{S_n}{T_n\rho_n}
        \asymp
        \log S_n
        \asymp
        \log N_n
        \asymp
        \log n,
$
    so \cref{eq:ber_sas_rho} is satisfied.
\end{remark}

\subsection{Adaptive spike-and-slab prior}
A key question in adaptive nonparametric estimation is whether a single prior can automatically select a model  suited to the unknown smoothness while still attaining the near-minimax contraction rate on each smoothness class.
The architecture \eqref{eq:sas_arch} depends on $\bfs$ and, through $\beta$, on $p$ as well. So the corresponding sparse Bayesian KAN requires oracle knowledge of the regularity of $f_0$.  We now relax this by placing a hyperprior on a single model-size parameter $N$  and controlling the remaining architectural quantities through conservative upper bounds that cover a prespecified smoothness envelope.

Let $\mathfrak{S} \subset \bbR_{++}^d$ denote the smoothness envelope over which adaptation is desired.
Define
$
    \tilde s_{\min} := \inf_{\bfs\in\mathfrak{S}} ( \sum_{j=1}^d s_j^{-1} )^{-1} > \omega
$
where $\omega = (1/p - 1/2)_+$, and set $\beta_{ad} := 1 + \kappa_{ad}$, where
$\kappa_{ad}$ is fixed arbitrary as $\kappa_{ad} > 2 \omega / (\tilde s_{\min} - \omega)$ if $p < 2$ and $\kappa_{ad} > 0$ if $p \geq 2$.
Since $\tbfs \le \underline{\bfs}$, $(1/p - \tbfs)_+ < 1/2$ and $\kappa \le 2\omega/(\tilde s_{\min} - \omega)$,
the exponent $\beta$ satisfies $\beta \le 1 + \kappa < \beta_{ad}$ for every $\bfs \in \mathfrak{S}$.

Now for each $N \in \bbN$, which is a model size parameter, we define the adaptive sparse KANs architecture  by
\begin{equation*}
    L_0 = 3 + 2 \lceil \log_2 d \rceil, \quad D(N) = 2d N, \quad G(N) = 4 H(N) - 2m, \quad H(N) = \lceil 2B^\star (N) \rceil.
\end{equation*}
where $B^\star(N) := B_{ad} N^{\beta_{ad}}$ with $B_{ad} > 0$ chosen large enough. Also, we denote the sparsity by $S(N) := \lceil S_0 N \rceil$ for $S_0 > 0$. And let the total number of parameters be $T(N) := T(L_0, D(N), G(N))$.

Building on the previous settings, the adaptive spike-and-slab  prior is defined as: 
\begin{align}   \label{eq:adaptive_sas_prior}
    \begin{split}
        \pi_N (N) &\propto \exp (- \lambda_N N \log N)  \\
        \pi (\btheta \mid \bgamma, N) &= \prod_{t=1}^{T(N)} \Bigl[ \gamma_t \widetilde{\pi}_{SL}(\theta_t \mid N) + (1-\gamma_t)\delta_0(\theta_t) \Bigr]   \\
        \pi (\bgamma \mid N) &= \binom{T(N)}{S(N)}^{-1} \bone\{\bgamma\in\{0,1\}^{T(N)},\ \norm{\bgamma}_0=S(N)\}   \\
        \pi(d\sigma^2) &= \pi_{\sigma^2}(\sigma^2)\,d\sigma^2,
    \end{split}
\end{align}
where $\lambda_N > 0$ is fixed.
Assume that the following conditions hold. These are conditional versions of \eqref{assump:B1}-\eqref{assump:B2}.
\begin{enumerate}[
    label=\textbf{(C\arabic*)},
    ref=C\arabic*
]
    \item   \label{assump:C1} There exists a constant $c_1>0$ such that for all $N$,
$
        \log\inf_{\abs{u}\le B^\star(N)+1}\widetilde{\pi}_{SL}(u | N)\ge -c_1\log N.
$

    \item \label{assump:C2} There exist constants $c_2, c_3 > 0$ and a sequence $\tau_N\lesssim N^{\beta_{\mathrm{ad}}}$ such that for all $t > 0$,
$
        \int_{\abs{u}>t}\widetilde{\pi}_{SL}(u \mid N)\,du \le c_2\exp\! \left[-c_3\left(\frac{t}{\tau_N}\right)^\alpha\right]
$.

    \item \label{assump:C3} $\pi_{\sigma^2}$ satisfies \eqref{assump:B3}.
\end{enumerate}

\begin{theorem}[Adaptive spike-and-slab prior]\label{thm:main-adaptive}
Suppose that Assumptions \eqref{assump:A1}--\eqref{assump:A4} hold for some unknown $s \in \mathfrak{S}$, and that the prior in \cref{eq:adaptive_sas_prior} satisfies \eqref{assump:C1}--\eqref{assump:C3}. Then the induced posterior contracts at the same rate $\epsilon_n$ in \eqref{eq:rate}; in particular, \eqref{eq:convergence} holds for $\epsilon_n$.
\end{theorem}

\begin{proof}
See \Cref{app:main-adaptive-proof}.
\end{proof}

\subsection{Extension to compositional anisotropic Besov spaces}

We now extend the main results to compositional anisotropic Besov spaces, following the function-space definition of \citet{suzuki2021deep,lee2025posterior}. In this setting, the regression function admits a layered compositional structure in which each layer depends only on a low-dimensional subspace of its inputs \citep{schmidt2020nonparametric}, so that the effective approximation complexity is determined by the \emph{layerwise smoothness and intrinsic dimension} rather than by the ambient dimension $d$. Structures of this kind arise naturally in regression. The additive model $f(\bx) = \sum_{i=1}^d g_i(x_i)$ with $g_i \in \cB^{s_i}_{p,q}([0,1])$ and $s_i > 1/p$, and the low-dimensional projection model $f(\bx) = g(\bA\bx + \bb)$ with $g \in \cB^{\bfs}_{p,q}([0,1]^t)$, $\tilde{s} > 1/p$ and $\bA \in \bbR^{t \times d}$, are special cases corresponding to $t^{(j)} = 1$ and $J = 2$, respectively \citep{lee2025posterior}.

\begin{definition}[Compositional anisotropic Besov space]
\label{def:comp_besov}
Let $J \in \bbN$ be the compositional depth and let
$
    \mathbf{d} = (d^{(0)}, d^{(1)}, \ldots, d^{(J)}) \in \bbN^{J+1}, \;
    d^{(0)} = d, \; d^{(J)} = 1,
$
be a sequence of dimensions. Let $\mathbf{t} = (t^{(1)}, \ldots, t^{(J)})$ be a
sequence of effective dimensions with $1 \leq t^{(j)} \leq d^{(j-1)}$, and let
$\bfs^{(1)}, \ldots, \bfs^{(J)}$ be layerwise smoothness vectors with
$\bfs^{(j)} \in \bbR_{++}^{t^{(j)}}$. The compositional anisotropic Besov space
$\cB^{\mathbf{d},\mathbf{t},\bfs}_{p,q}([0,1]^d)$ is the class of functions of
the form
$
    f = f_J \circ f_{J-1} \circ \cdots \circ f_1 \;:\; [0,1]^d \to \bbR
$
such that, for $j = 1, \ldots, J-1$,
$
    f_j = (f_{j,1}, \ldots, f_{j,d^{(j)}})
    \;:\; [0,1]^{d^{(j-1)}} \to [0,1]^{d^{(j)}}
$
and $f_J : [0,1]^{d^{(J-1)}} \to \bbR$. Moreover, for every $j \in [J]$ and
$k \in [d^{(j)}]$, there exist a subset $I_{j,k} \subset [d^{(j-1)}]$ with
$|I_{j,k}| = t^{(j)}$ and a function
$\tilde{f}_{j,k} \in \cB^{\bfs^{(j)}}_{p,q}([0,1]^{t^{(j)}})$ such that
$
    f_{j,k}(\bx) = \tilde{f}_{j,k}(\bx_{I_{j,k}}),
$
where $\bx_I$ denotes the subvector of $\bx$ indexed by $I$.
\end{definition}

\begin{figure}[t]
    \centering
    \includegraphics[width=\linewidth]{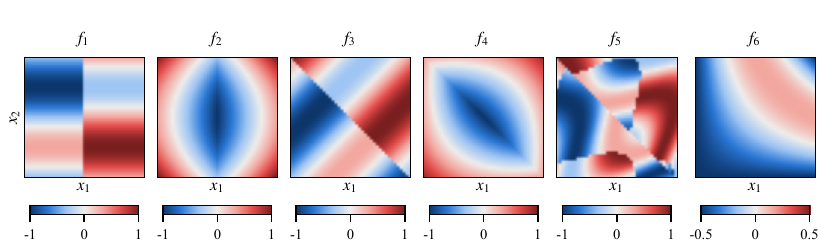}
    \caption{Target functions of the simulation study in \Cref{app:simulation}, normalized to $[-1,1]$: $f_1$ and $f_2$ from \citet{lee2025posterior}, their $45^\circ$ rotations $f_3$ and $f_4$, the composition $f_5$, and Friedman's function $f_6$ on $[0,1]^{10}$, shown on $(x_1, x_2)$ with $x_3 = \dots = x_{10} = 1/2$.}
    \label{fig:targets}
\end{figure}

\Cref{fig:targets} shows examples of the function classes considered in this paper.
The functions $f_1$ (a jump in $x_1$ plus a periodic term in $x_2$) and $f_2$ (a kink in $x_1$ plus a quadratic term in $x_2$) are anisotropic, with smoothness differing across coordinates.
The functions $f_3$ and $f_4$ are $f_1$ and $f_2$ composed with a rotation, so that the direction of low regularity is no longer aligned with a coordinate axis, and $f_5= f_3\bigl((f_3+1)/2, (f_4+1)/2\bigr)$ composes $f_3$ and $f_4$ into a two-layer function with the structure of \Cref{def:comp_besov} above.
Friedman's function $f_6$ \citep{friedman1991multivariate} depends on only five of its ten inputs through components of at most two variables, so its effective dimension is much smaller than the ambient dimension $d = 10$.

For each layer $j \in [J]$, define the layerwise intrinsic smoothness and minimum
smoothness by
$
    \tilde{s}^{(j)} :=
    ( \sum_{i=1}^{t^{(j)}} (s_i^{(j)} )^{-1} )^{-1}
$,
$
    \underline{\bfs}^{(j)} := \min_{1 \leq i \leq t^{(j)}} s_i^{(j)}
$, 
$
    \overline{\bfs}^{(j)} := \max_{1 \leq i \leq t^{(j)}} s_i^{(j)}
$.
The rate of posterior contraction is governed by an intrinsic smoothness that
accounts for the propagation of approximation errors across layers.
For each $j \in [J]$, define
$
    t^{*(j)} := \frac{\underline{\bfs}^{(j)}}{\tilde{s}^{(j)}}
$,
$
    \alpha_j := \underline{\bfs}^{(j)} - t^{\ast (j)} / p
$,
$
    \tilde{s}^{*(j)} := \tilde{s}^{(j)}
    \prod_{k=j+1}^{J}
    \{ ( \underline{\bfs}^{(k)} - \frac{t^{*(k)}}{p} ) \wedge 1
    \}
$,
and let
$
    j^* := \argmin_{j \in [J]} \tilde{s}^{*(j)}
$, 
$
    t^* := t^{*(j^*)}
$, 
$
    \tilde{s}^* := \tilde{s}^{*(j^*)}
$
denote the intrinsic layer, intrinsic dimension, and intrinsic smoothness,
respectively.

\paragraph{Assumptions.}
We impose the following assumption on the compositional structure.
\begin{enumerate}[
    label=\textbf{(D\arabic*)},
    ref=D\arabic*
]
\item \label{assump:D1}
The true function satisfies
$f_0 \in \cB^{\mathbf{d},\mathbf{t},\bfs}_{p,q}([0,1]^d) \cap \UB$.
Furthermore,
$
    \tilde{s}^{(1)} > (1/p - 1/2)_+
$,
$
    \tilde{s}^{(j)} > 1/p \quad \text{for } j = 2, \ldots, J
$
and assume
$Y_i | X_i \sim N(\operatorname{clip} \circ f(X_i), \sigma^2)$.
\item \label{assump:D2}
For each $ j\in[J] $,
$
0<  \overline{\bfs}^{(j)} <\min\{m,m-1+1/p\}
$,
and if $q > 1$, $\alpha_j \neq 1$ for $j = 2, \ldots, J$.
\end{enumerate}
For $j \geq 2$, the condition $\tilde{s}^{(j)} > 1/p$ in \eqref{assump:D1} guarantees, via the embedding $\cB^{\bfs^{(j)}}_{p,q} \hookrightarrow \cB^{\gamma_j \bfs^{(j)}}_{\infty,q}$ with $\gamma_j := 1 - 1/(p\tilde{s}^{(j)}) > 0$ \citep[Proposition~1]{suzuki2021deep}, that the components of the $j$-th layer are H\"older continuous with exponent $\alpha_j \wedge 1$ (\cref{lem:holder_comp}). An approximation error of order $N^{-\tilde{s}^{(j)}}$ incurred at an inner layer is raised to the power of the H\"older exponent of each outer layer it passes through, so that it becomes $N^{-\tilde{s}^{*(j)}}$ at the output; this is precisely what the definition of $\tilde{s}^{*(j)}$ captures.

Under \eqref{assump:D1}, \eqref{assump:D2}, \eqref{assump:A2}, and \eqref{assump:A3},
there exist constants $L_{\mathrm{cp}} \in \bbN$ and
$D_{\mathrm{cp}}, S_{\mathrm{cp}}, B_{\mathrm{cp}}, \beta_{\mathrm{cp}} > 0$
such that the compositional architecture sequence is given by
\begin{equation}    \label{eq:comp_arch}
\begin{aligned}
    N_n^* &:= \bigl\lceil n^{1/(2\tilde{s}^*+1)} \bigr\rceil,
    \quad
    D_n^{\mathrm{cp}} := D_{\mathrm{cp}} N_n^*,
    \quad
    G_n^{\mathrm{cp}} := 4 H_n^{\mathrm{cp}} - 2m,  \\
    B_n^{\mathrm{cp},*} &:= B_{\mathrm{cp}} (N_n^*)^{\beta_{\mathrm{cp}}},
    \quad
    H_n^{\mathrm{cp}} := \lceil 2 B_n^{\mathrm{cp},*} \rceil,
    \quad
    S_n^{\mathrm{cp}} := \bigl\lceil S_{\mathrm{cp}} N_n^* \bigr\rceil.
\end{aligned}
\end{equation}
We use the spike-and-slab prior \cref{eq:sas_prior} with the compositional
architecture sequence in \cref{eq:comp_arch}, where the slab density satisfies
\eqref{assump:B1}--\eqref{assump:B2} with $B_n^\star = B_n^{\mathrm{cp},*}$
and $\tau_n \lesssim (N_n^*)^{\beta_{\mathrm{cp}}}$, and the variance prior
satisfies \eqref{assump:B3}.

\begin{theorem}[Compositional anisotropic Besov spaces]
\label{thm:composite}

Suppose \eqref{assump:D1}, \eqref{assump:D2}, \eqref{assump:A2}, and \eqref{assump:A3}, and that the architecture of the sparse Bayesian KAN is specified as in \eqref{eq:comp_arch} and its prior in \eqref{eq:sas_prior} satisfies \eqref{assump:B1}--\eqref{assump:B3}.
Then, for
$\epsilon_n^* := n^{-\tilde{s}^*/(2\tilde{s}^*+1)} (\log n)^{1/2}$
and any sequence $M_n \to \infty$,
\[
    \Pi\!\left( (\btheta,\sigma) :
    \| \operatorname{clip} \circ f_{\btheta} - f_0\|_{L^2(P_X)} + |\sigma^2 - \sigma_0^2|
    > M_n \epsilon_n^* \;\middle|\; \cD_n \right)
    \to 0
\]
in $P_{f_0,\sigma_0}^{(n)}$-probability as $n \to \infty$.
\end{theorem}

\begin{proof}
    See \Cref{app:composite-proof}.
\end{proof}

\begin{remark}
The rate $\epsilon_n^* = n^{-\tilde{s}^*/(2\tilde{s}^*+1)}(\log n)^{1/2}$ depends
only on the intrinsic smoothness $\tilde{s}^*$, not on the ambient dimension $d$.
When $t^{(j)} \ll d^{(j-1)}$ for each layer, the curse of dimensionality is
substantially mitigated.
\end{remark}

\begin{remark}[Numerical illustration]
For the target functions in \Cref{fig:targets}, \Cref{app:simulation} reports the estimation results of sparse Bayesian KANs together with a Bayesian ReLU network, BART \citep{chipman2010bart}, a Gaussian process, and a random forest, for sample sizes $n$ from $250$ to $4{,}000$.
\end{remark}

\section{Conclusion and discussion}
In this paper, we established posterior contraction rates for sparse Bayesian KANs over anisotropic Besov spaces.
%providing a rigorous statistical foundation for this emerging class of neural network architectures. 
We showed that sparse Bayesian KANs equipped with spike-and-slab priors achieve near-minimax posterior contraction rates, with the rate governed by the intrinsic anisotropic smoothness of the underlying function.
Moreover, by introducing a hyperprior on a single model-size parameter, we demonstrated that the posterior can adapt to unknown smoothness while still attaining the corresponding near-optimal rate.
We further extended our analysis to compositional Besov spaces, showing that sparse Bayesian KANs attain near-minimax rates determined by the layerwise smoothness and effective dimensions of the underlying compositional structure. 

A key insight of our analysis is that, in contrast to standard ReLU-based architectures, the depth of KANs can be kept fixed, while approximation complexity is controlled through spline-grid range and size, width, and sparsity. This highlights KANs' mechanism for achieving statistical efficiency: rather than relying on increasingly deep architectures, they leverage flexible spline-based edge functions and sparse representations to capture anisotropic regularity. Our results show that this structural advantage translates into optimal statistical efficiency in a Bayesian framework.

There are several future directions to explore. First, while our analysis focuses on fixed-knot spline constructions, it would be of interest to study adaptive or data-driven knot placement, which could further enhance empirical performance. Second, it would be of interest to develop a theory for  generative models in which the generator or the score function is modeled by the KAN architecture.  Third, although our theory captures statistical optimality, understanding the computational aspects of posterior inference for sparse KANs, including scalable variational or MCMC methods, remains an important problem.

\begin{ack}
    Jaeyong Lee was supported by the National Research Foundation of Korea (NRF) grant funded by the Korea government (MSIT) (No. 2023R1A2C1003050 and RS-2024-00406127). Lizhen Lin would like to acknowledge the generous support of NSF grant DMS 2503119. Kyeongwon Lee was supported by the National Research Foundation of Korea (NRF) funded by the Korea government (MSIT) (RS-2026-25486448) and INHA UNIVERSITY Research Grant (2026).
\end{ack}

\newpage

% \section*{References}

% References follow the acknowledgments in the camera-ready paper. Use unnumbered first-level heading for
% the references. Any choice of citation style is acceptable as long as you are
% consistent. It is permissible to reduce the font size to \verb+small+ (9 point)
% when listing the references.
% Note that the Reference section does not count towards the page limit.
% \medskip

{
\small
\bibliographystyle{custom}
\bibliography{refs}
}

%%%%%%%%%%%%%%%%%%%%%%%%%%%%%%%%%%%%%%%%%%%%%%%%%%%%%%%%%%%%
\clearpage

\appendix

\section{Preliminaries}
We introduce the notation and definitions needed to understand the proofs in this appendix. Let $(L, \bD, \bG, H)$ be the structure of KANs as defined in \cref{eq:KANs_architecture} where $L \in \bbN$, $\bD \in \bbN^{L + 1}$, $\bG \in \bbN^{L}$, $H > 0$, and with parameters $\boldsymbol{\theta} = (\theta_{i,j,k}^{(l)})_{i,j,k,l}$.
\begin{equation*}
KAN_c(L, \bD, \bG, H ; m)
    := \left\{ f_{\boldsymbol{\theta}} \in KAN(L, \bD, \bG, H ; m)
        \mid
        \theta_{i,j,0}^{(l)} = 0 \; \text{for all } l, i, j
        \right\}
\end{equation*}
The class $KAN_c$ consists of KANs whose activation functions are constructed only from the spline components, without the SiLU activation. In this work, we analyze posterior contraction of the model under common width and grid size. That is, when the width and grid are given by $\bD = (d, D, D, \ldots, 1)$ and $\bG = (G_0, G, \ldots, G)$ where $G_0$ and the corresponding knot vector $\bxi^{(0)} \in [a_0, b_0]^{G_0 + 2m + 1}$ s.t. $[\xi^{(0)}_0, \xi^{(0)}_{G_0}] \supseteq [0, 1]^d$, $\Delta_0 =\xi^{(0)}_{k + 1} - \xi^{(0)}_k$ are fixed, respectively, we write, without boldface notation,
\begin{align*}
\begin{split}
    KAN(L, D, G, H; m) &:= KAN(L, \bD, \bG, H ; m), \\
    KAN_c(L, D, G, H; m) &:= KAN_c (L, \bD, \bG, H ; m)
\end{split}
\end{align*}
for simplicity.

Next, we describe the properties of B-splines. Let $\boldsymbol{\xi} = (\xi_{-m}, \cdots, \xi_{G + m} ) \in [a, b]^{G + 2m + 1}$, with $a < b$, be a knot vector. We assume that its coordinates are strictly increasing without repetition, i.e., $\xi_k < \xi_{k + 1}$ for all $k$, and equally spaced with $\Delta = \xi_{k + 1} - \xi_k$. These B-splines satisfy the nonnegativity and local support properties
\begin{equation*}
    B_{k,m} (x) \geq 0 \; \text{ for all } x, \quad B_{k,m}(x) = 0 \text{ if } x \notin [\xi_{k - m - 1}, \xi_{k}], \quad ( k = 1, \ldots, G + m ).
\end{equation*}
They also satisfy the partition of unity property
\begin{align}    \label{eq:partitionOfUnity}
\begin{split}
    &\sum_{k = 1}^{G + m} B_{k,m} (x) = 1, \quad x \in [\xi_0, \xi_G],   \\
    0 \leq & \sum_{k = 1}^{G + m} B_{k,m} (x) < 1, \quad \text{otherwise}.
\end{split}
\end{align}
A B-spline curve is a function $s_m :\bbR \rightarrow \bbR$ given by a linear combination of the basis functions $\{B_k\}_{k=1}^{G + m}$. We denote the coefficients of the basis functions by $\bw = (w_1, \ldots, w_{G + m}) \in \bbR^{G + m}$, that is,
\begin{equation*}
    s_m(x) := \sum_{k = 1}^{G + m} w_k B_{k,m} (x).
\end{equation*}
Since the shape of each curve depends on $(\bw, \bxi)$, we write $s_m(x; \bw, \bxi)$ when we need to emphasize this dependence. We introduce one more property of the B-spline curve $s_m$: its derivative $s_m^\prime$ is again a B-spline curve of degree $m - 1$. Specifically, following \citet{de2001practical},
\begin{align}   \label{eq:bSpline_deriv}
    \begin{split}
        s_m^\prime (x) &= \sum_{k= 0}^{G + m} w_k^\prime B_{k, m - 1} (x), \\
        w_k^\prime &:= \frac{m ( w_{k + 1} - w_k )}{\xi_{k} - \xi_{k - m}} = \frac{w_{k+1} - w_k}{\Delta}, \quad k = 0, \ldots, G + m,
    \end{split}
\end{align}
where $w_0 = w_{G+m+1} := 0$, and $B_{k,m-1}(\cdot\,; \bxi)$, $k = 0, \ldots, G+m$, are the B-splines of degree $m-1$ generated by the same knot vector, with $B_{k,m-1}$ supported on $[\xi_{k-m}, \xi_k]$. The boundary terms $k = 0$ and $k = G+m$ vanish on $[\xi_0, \xi_G]$, but they are needed for the formula to hold on all of $\bbR$.
Furthermore, writing $\bw^\prime = (w_0^\prime, \ldots, w_{G + m - 1}^\prime, w_{G + m}^\prime)$ and applying \cref{eq:partitionOfUnity} to the B-splines of degree $m-1$, we obtain
\begin{equation}    \label{eq:bSpline_deriv_bound}
    | s_m^\prime(x) | \leq \| \bw^\prime \|_\infty \leq \frac{2 \| \bw \|_\infty}{\Delta}, \quad \forall x.
\end{equation}

Finally, when we emphasize the spline part of the edge function $\phi_{i,j}^{(l)}$ in the layer $\Phi^{(l)}$, we write
\[
\phi_{i,j}^{(l)} (x) = s_m (x; \btheta_{i,j,-0}^{(l)} ) +  \theta_{i,j,0}^{(l) } \, \operatorname{SiLU}(x),
\]
where $\btheta_{i,j}^{(l)} = (\theta_{i,j, 0}^{(l)}, \btheta_{i,j,-0}^{(l)} ) \in \bbR^{G_l + m + 1}$.

%%%%%%%%%%%%%%%%%%%%%%%%%%%%%%%%%%%%%%%%%%%%%%%%%%%%%%%%%%%%

\section{Auxiliary results}
\label{sec:auxiliary_results}

\subsection{Exact B-spline representations}
\label{subsec:exact_representations}

For an equally spaced knot vector $\bxi = (\xi_{-m}, \ldots, \xi_{G+m})$ with spacing $\Delta$, $B_{k,m}(\cdot\,;\bxi)$ is supported on $[\xi_{k-m-1}, \xi_k]$, and we write $s_m(\cdot\,; \bw, \bxi) = \sum_{k=1}^{G+m} w_k B_{k,m}(\cdot\,; \bxi)$.

\begin{lemma}[Local polynomial reproduction]\label{lem:marsden}
Let $\bxi$ be an equally spaced knot vector with spacing $\Delta > 0$, $I = [u_-, u_+] \subseteq [\xi_0, \xi_G]$, $P(u) = \sum_{r=0}^m a_r u^r$, and $m \geq 1$. Then there exists $\bw \in \bbR^{G+m}$ such that
\begin{enumerate}[label=(\roman*)]
    \item $s_m(u; \bw, \bxi) = P(u)$ for $u \in I$;
    \item $w_k = 0$ whenever $(\xi_{k-m-1}, \xi_k) \cap I = \emptyset$; in particular, $\norm{\bw}_0 \leq \lceil (u_+ - u_-)/\Delta \rceil + m + 1$;
    \item $\norm{\bw}_\infty \leq \sum_{r=0}^m \abs{a_r} R^r$, where $R := \max\{\abs{u_-}, \abs{u_+}\} + (m+1)\Delta$.
\end{enumerate}
\end{lemma}

\begin{proof}
Consider the Marsden identity \citep{de2001practical}
\[
    (u - y)^m = \sum_{k=1}^{G+m} \prod_{i=k-m}^{k-1} (\xi_i - y) \, B_{k,m}(u; \bxi),
    \qquad u \in [\xi_0, \xi_G],\ y \in \bbR.
\]
Regarding both sides as polynomials in $y$ and comparing the coefficients of $(-y)^{m-r}$, we obtain
\begin{equation}\label{eq:marsden-monomial}
    u^r = \sum_{k=1}^{G+m} \lambda_{k,r} B_{k,m}(u; \bxi),
    \qquad
    \lambda_{k,r} := \binom{m}{r}^{-1} e_r(\xi_{k-m}, \ldots, \xi_{k-1}),
    \qquad u \in [\xi_0, \xi_G],
\end{equation}
where $e_r$ denotes the $r$th elementary symmetric polynomial.
We set $w_k := \sum_{r} a_r \lambda_{k,r}$ if $(\xi_{k-m-1}, \xi_k) \cap I \neq \emptyset$, and $w_k := 0$ otherwise.
\begin{itemize}
    \item (i): The omitted $B_{k,m}$ vanish on $I$.
    \item (ii): $(\xi_{k-m-1}, \xi_k) \cap I \neq \emptyset \iff \xi_k \in (u_-,\ u_+ + (m+1)\Delta)$, and this open interval contains at most $\lceil (u_+ - u_-)/\Delta \rceil + m + 1$ knots.
    \item (iii): If $w_k \neq 0$, then $[\xi_{k-m-1}, \xi_k]$ meets $I$, so $\abs{\xi_i} \leq R$ for $k-m \leq i \leq k-1$, and hence $\abs{\lambda_{k,r}} \leq R^r$. \qedhere
\end{itemize}
\end{proof}

Since the Marsden identity holds only on the basic interval $[\xi_0, \xi_G]$, the first layer requires $[\xi_0^{(0)}, \xi_{G_0}^{(0)}] \supseteq [0, 1]$. The localization in (ii) is the key point. The global coefficients grow as large as $\abs{\lambda_{k,r}} \asymp \max_i\abs{\xi_i}^r \asymp H^r$, but an edge whose input lies in $[0,1]$ uses only the $O(m)$ basis functions whose supports meet $[0,1]$, so its coefficients do not depend on $H$.

Let $\psi_m := \bone_{[0,1)}^{*(m+1)}$ be the cardinal B-spline of degree $m$. It has knots $0, 1, \ldots, m+1$ and support $[0, m+1]$, and coincides with the cardinal B-spline of order $m$ (i.e., of degree $m$) in \citet{suzuki2021deep}.

\begin{lemma}[Cardinal B-splines on a half-integer grid]\label{lem:cardinal}
Let $H \in \bbN$ with $H \geq m+1$, let $G := 4H - 2m$, and let $\bxi$ be the equally spaced knot vector on $[-H, H]$. Then
\begin{enumerate}[label=(\roman*)]
    \item $\Delta = 1/2$, $\xi_k = -H + (k+m)/2$, and $[0,1] \subseteq [\xi_0, \xi_G] = [-H + m/2,\ H - m/2]$;
    \item there exists $\bw$ with $\norm{\bw}_0 = m+2$ and $\norm{\bw}_\infty \leq 2$ such that $s_m(\cdot\,; \bw, \bxi) = \psi_m$ on $\bbR$;
    \item $0 \leq \psi_m \leq 1$.
\end{enumerate}
\end{lemma}

\begin{proof}
(i) We have $\Delta = 2H/(G+2m) = 1/2$, and since $H \geq m+1$, it follows that $\xi_0 = -H + m/2 \leq 0$ and $\xi_G = H - m/2 \geq 1$.

(ii) Since the knots are equally spaced, $B_{k,m}(u; \bxi) = \psi_m(2u - 2\xi_{k-m-1})$, and hence $\psi_m(2\,\cdot - r) = B_{k_r,m}(\cdot\,;\bxi)$ with $k_r := 2H + r + 1 \in [G+m]$ for $r = 0, \ldots, m+1$. By the two-scale relation \citep{chui1992introduction}
\begin{equation}\label{eq:two-scale}
    \psi_m(u) = 2^{-m} \sum_{r=0}^{m+1} \binom{m+1}{r} \psi_m(2u - r),
\end{equation}
we may take $w_{k_r} = 2^{-m}\binom{m+1}{r} \leq 2$ and $w_k = 0$ otherwise.

(iii) Since $\psi_m$ is a convolution of nonnegative functions, $\psi_m \geq 0$, and by the partition of unity, $\psi_m \leq \sum_{\ell} \psi_m(\cdot - \ell) = 1$.
\end{proof}

\begin{lemma}[Exact multiplication by fixed-knot $KAN_c$]\label{lem:exact_product}
Let $m \geq 2$ and $d \geq 2$, let $(H, G, \bxi)$ be as in \cref{lem:cardinal}, and let $L_\times := 2\lceil \log_2 d \rceil$. There exist $KAN_c$ layers $\Phi^\times_0, \ldots, \Phi^\times_{L_\times - 1}$ with knot vector $\bxi$ and zero SiLU weights, input dimension $d$, output dimension $1$, and intermediate widths $n_\ell \leq \lfloor 3d/2 \rfloor$, such that $\Pi_d := \Phi^\times_{L_\times - 1} \circ \cdots \circ \Phi^\times_0$ satisfies the following:
\begin{enumerate}[label=(\roman*)]
    \item $\Pi_d(\mathbf{u}) = u_1 \cdots u_d$ for $\mathbf{u} \in [0,1]^d$;
    \item if $\mathbf{u} \in [0,1]^d$, then all intermediate outputs lie in $[0,1]^{n_\ell}$;
    \item the number of nonzero weights is at most $\bigl(7(d-1) + 2\lceil \log_2 d \rceil\bigr)(m+3)$;
    \item every weight is bounded in absolute value by $(m+3)^2/2$.
\end{enumerate}
\end{lemma}

\begin{proof}
Applying \cref{lem:marsden} with $I = [0,1]$ and $\Delta = 1/2$, so that $R = (m+3)/2$, we obtain B-spline curves $\phi_{\mathrm{id}}$, $\phi_{\mathrm{half}}$, and $\phi_{\mathrm{sq}}$ such that, for $t \in [0,1]$,
\begin{equation}\label{eq:prod-edges}
    \varphi_{\mathrm{id}}(t) = t,\quad
    \varphi_{\mathrm{half}}(t) = t/2,\quad
    \varphi_{\mathrm{sq}}(t) = 2t^2.
\end{equation}
Note that $\varphi_{\mathrm{sq}}$ requires degree $m \geq 2$. Each curve has at most $m + 3$ nonzero coefficients, whose magnitudes are bounded by $R$, $R/2$, and $2R^2$, respectively.

This allows us to implement the binary product $\Pi_2$ through addition and composition of B-splines. That is, for $(u_1, u_2) \in [0,1]^2$,
\begin{equation}\label{eq:P2}
\begin{aligned}
    \phi^{(1)}(u_1, u_2) &:= \bigl( \varphi_{\mathrm{half}}(u_1) + \varphi_{\mathrm{half}}(u_2),\ \varphi_{\mathrm{half}}(u_1),\ \varphi_{\mathrm{half}}(u_2) \bigr)
    = \Bigl( \tfrac{u_1 + u_2}{2}, \tfrac{u_1}{2}, \tfrac{u_2}{2} \Bigr) \in [0,1]^3,\\
    \phi^{(2)}(v_1, v_2, v_3) &:= \varphi_{\mathrm{sq}}(v_1) - \varphi_{\mathrm{sq}}(v_2) - \varphi_{\mathrm{sq}}(v_3),\\
    \Pi_2 &:= \phi^{(2)} \circ \phi^{(1)},\qquad \Pi_2(u_1,u_2) = 2\bigl(\tfrac{u_1 + u_2}{2}\bigr)^2 - 2\bigl( \tfrac{u_1}{2} \bigr)^2 - 2\bigl( \tfrac{u_2}{2} \bigr)^2 = u_1 u_2.
\end{aligned}
\end{equation}
Since each output coordinate is a sum of univariate edges, both maps are of the form \cref{eq:KANs_layer}. That is, $\Pi_2$ is implemented by $\phi^{(2)} \circ \phi^{(1)}$, which has depth $2$, width $3$, and $7$ nonzero edges. Since the arguments are halved, all intermediate values lie in $[0,1]$.

We now construct a $KAN_c$ implementing $\Pi_d$ by repeating this composition along a binary tree. The proof proceeds by induction. Let $\ell = 0, \ldots, L_\times / 2 - 1$, and let the initial width be $n_0 := d$. First, for $\ell = 0$, the layers $\Phi^\times_1 \circ \Phi^\times_{0}$ can be constructed as follows:
\begin{align*}
    \Phi^\times_0 (u_1, \ldots, u_d) &:= \big( \phi^{(1)} (u_1, u_2), \; \phi^{(1)} (u_3, u_4), \; \ldots, \; \phi^{(1)} (u_{d - 1}, u_{d} ) \big),  \\
    \Phi^\times_1 \circ \Phi^\times_0 (u_1, \ldots, u_d) &:= \big( \phi^{(2)} \circ \phi^{(1)} (u_1, u_2), \; \phi^{(2)} \circ \phi^{(1)} (u_3, u_4), \; \ldots, \; \phi^{(2)} \circ \phi^{(1)} (u_{d - 1}, u_{d})  \big)  \\
    &= \big( u_1 u_2, \; u_3 u_4, \; \ldots, \; u_{d - 1} u_{d} \big).
\end{align*}
If $d$ is odd, we additionally append $\varphi_{\mathrm{id}}$ as the last component of each layer. The width of the output layer of $\Phi^\times_1 \circ \Phi^\times_0$ is $n_1 = \lceil d/2 \rceil$, and the width of the intermediate layer $\Phi_0^\times$ is $3 \lfloor d / 2 \rfloor + (d \bmod 2)$. More generally, for $\ell > 0$, we define
\begin{align*}
    \Phi^\times_{2\ell} (v_1, \ldots, v_{n_\ell}) &:= \big( \phi^{(1)} (v_1, v_2), \ldots, \phi^{(1)} (v_{n_\ell - 1}, v_{ n_\ell } ) \big),    \\
    \Phi^\times_{2\ell + 1} \big( t_1, \ldots, t_{ 3 \lfloor n_\ell / 2 \rfloor} \big) &:= \big( \phi^{(2)} (t_1, t_2, t_3), \ldots, \phi^{(2)} (t_{3 \lfloor n_\ell / 2 \rfloor - 2}, t_{3 \lfloor n_\ell / 2 \rfloor - 1}, t_{3 \lfloor n_\ell / 2 \rfloor}) \big).
\end{align*}
If $n_\ell$ is odd, we again append $\phi_{\mathrm{id}}$ as the last component of each layer. The width of the output layer is $n_{\ell + 1} = \lceil n_\ell / 2 \rceil$, and the width of the intermediate layer is $3\lfloor n_\ell/2 \rfloor + (n_\ell \bmod 2)$. Repeating this up to $\ell = L_\times / 2 - 1$, the final output is $\Phi^\times_{L_\times - 1} \circ \cdots \circ \Phi^\times_0 (u_1, \ldots, u_d) = u_1 \cdots u_d$. This proves (i) and (ii).

Next, we check the depth and width of this network. Since pairing the variables two at a time to form binary products is implemented by a $KAN_c$ module of depth $2$, the total depth of the network is $L_\times = 2 \lceil \log_2 d \rceil$. The width of each intermediate layer is $3\lfloor n_\ell/2 \rfloor + (n_\ell \bmod 2) \leq \lfloor 3d / 2 \rfloor$.

The network $\Phi^\times_{L_\times - 1} \circ \cdots \circ \Phi^\times_0$ contains exactly $d - 1$ copies of the binary product module $\Pi_2$, each of which has $7$ nonzero edges. When the output width $n_\ell$ is odd, at most one additional component $\phi_{\mathrm{id}}$ is added per layer. Moreover, each edge has at most $m + 3$ nonzero coefficients. Altogether, the number of nonzero weights of the network is at most $( 7 (d - 1) + 2 \lceil \log_2 d \rceil ) (m + 3)$, which proves (iii).

Finally, for (iv), substituting the conditions of the present lemma into \cref{lem:marsden}(iii) gives $R = (m + 3)/2$, so by the coefficient bounds following \cref{eq:prod-edges}, the largest weight bound is $2R^2 = (m + 3)^2 / 2$. This completes the proof.
\end{proof}

\cref{lem:exact_product} corresponds to the product network of \citet[Lemma~A.3]{schmidt2020nonparametric}. Since ReLU networks only approximate products, their depth grows with the required accuracy, whereas B-spline edges with $m \geq 2$ reproduce quadratic maps exactly, so the depth depends only on $d$. The condition $m \geq 2$ in \eqref{assump:A4} is used at this step.

\subsection{Approximation theorem}
\label{subsec:approximation_theorem}

\begin{lemma}[Approximation by fixed-knot $KAN_c$]\label{lem:approx}
    Suppose that \eqref{assump:A1} and \eqref{assump:A4} hold. Then there exist constants
    $S_0,B_0>0$, independent of $N$, such that for every $N\in\bbN$ there exists
    \[
    f_N^\star
    \in
    KAN_c\bigl(L_0, D(N) , G(N), H(N);m\bigr),
    \qquad
    L_0 = 3 + 2\lceil \log_2 d\rceil,
    \]
    satisfying
    \begin{equation}\label{eq:approx-main}
        \norm{f_0-f_N^\star}_{L^2}
        \lesssim
        N^{-\tbfs}.
    \end{equation}
    Moreover, there exists a parameter vector $\btheta_N^\star$ realizing $f_N^\star$ such that
    \[
    \norm{\btheta_N^\star}_\infty\le B^\star(N),
    \qquad
    \norm{\btheta_N^\star}_0\le S_0N.
    \]
    Here,
    \[
        D(N) := 2d N,
        \quad
        G(N) := 4 H(N) - 2m,
        \quad
        H(N) := \lceil 2B^\star(N) \rceil,
        \quad
        B^\star(N) := B_0 N^\beta,
    \]
    and
    \begin{align}   \label{eq:beta}
    \begin{split}
        \beta
        &:=
        (1+ \kappa) \max\{(1/p-\tbfs)_+, \, \tbfs /\underline{\bfs}\},  \\
        \kappa
        &:= \begin{cases}
            \frac{2 \omega}{\tbfs - \omega},    & p < 2, \\
            0,                                  & p \geq 2,
        \end{cases}
    \end{split}
    \end{align}
    with $\omega := (1/p - 1/2)_+$.
\end{lemma}

\paragraph{Proof sketch and layer-wise structure.}
We first summarize the construction in \cref{lem:approx}.
By Lemma~2 of \citet{suzuki2021deep}, $f_0$ is approximated in $L^2$ at the rate $N^{-\tbfs}$ by a tensor-product B-spline expansion $f_N^{\mathrm{sp}} = \sum_{(k,j) \in E_N} \alpha_{k,j} M_{k,j}^{d,m}$ with $|E_N| \leq N$. Here $M_{k,j}^{d,m}(\bx) = \prod_{i=1}^d \psi_m(z_{k,j,i}(x_i))$, $z_{k,j,i}(x_i) = 2^{\lfloor k \underline{\bfs} / s_i \rfloor} x_i - j_i$, and $\psi_m$ is the cardinal B-spline of degree $m$.
The proof shows that $f_N^{\mathrm{sp}}$ is realized \emph{exactly} by a fixed-knot $KAN_c$ through the following pipeline:
\begin{align*}
    \bx
    &\xmapsto{\ \Phi_0\ } \bigl( z_{k,j,i}(x_i) \bigr)_{k,j,i}
    \xmapsto{\ \Phi_1\ } \bigl( \psi_m(z_{k,j,i}(x_i)) \bigr)_{k,j,i}
    \\
    &\xmapsto{\ \Phi_2, \ldots, \Phi_{L_0 - 2}\ } \bigl( M_{k,j}^{d,m}(\bx) \bigr)_{(k,j) \in E_N}
    \xmapsto{\ \Phi_{L_0 - 1}\ } f_N^{\mathrm{sp}}(\bx).
\end{align*}

\begin{enumerate}[label=(\roman*)]
    \item By the Marsden identity \citep{de2001practical}, B-splines of degree $m$ exactly reproduce polynomials of degree at most $m$, so each affine map $z_{k,j,i}$ is implemented by a single edge in the first layer.
    \item The choice $G(N) = 4H(N) - 2m$ yields the hidden-layer knot spacing $2H(N)/(G(N)+2m) = 1/2$, and the two-scale relation $\psi_m(u) = 2^{-m} \sum_{r=0}^{m+1} \binom{m+1}{r} \psi_m(2u - r)$ \citep{chui1992introduction} expresses $\psi_m$ in terms of $m+2$ basis functions with coefficients independent of $N$.
    \item Since $uv = 2(\frac{u+v}{2})^2 - 2(\frac{u}{2})^2 - 2(\frac{v}{2})^2$ involves only quadratic univariate maps, two KAN layers compute the product exactly when $m \geq 2$. A binary tree of such modules computes $M_{k,j}^{d,m}$ using $2\lceil \log_2 d \rceil$ layers (\cref{lem:exact_product}).
    \item The output layer forms the linear combination with coefficients $\alpha_{k,j}$.
    Since each of the $\abs{E_N} \leq N$ blocks uses $O(1)$ nodes and nonzero weights, the width and sparsity grow linearly in $N$, while the depth does not depend on $N$; see \cref{tab:approx_arch}. 
\end{enumerate}

\begin{table}[ht]
\centering
\small
\caption{Layer-wise structure of a single block $(k,j) \in E_N$ of the oracle network $f_N^\star$ in \cref{lem:approx}. The $|E_N| \leq N$ blocks are arranged in parallel, and $H = H(N)$.}
\label{tab:approx_arch}
\begin{tabular}{@{}lllll@{}}
\toprule
Layer & Operation & Domain $\to$ Range & Nonzero weights & Weight bound \\
\midrule
$\Phi_0$ & Affine map $z_{k,j,i}$ & $[0,1] \to [-H, H]$ & $\leq d(G_0 + m)$ & $\lesssim N^{(1+\kappa) \tbfs / \underline{\bfs} }$ \\
$\Phi_1$ & Cardinal B-spline $\psi_m$ & $[-H, H] \to [0,1]$ & $d(m+2)$ & $\leq 2$ \\
$\Phi_2, \ldots, \Phi_{L_0-2}$ & Product tree & $[0,1]^d \to [0,1]$ & $O(dm)$ & $O(m^2)$ \\
$\Phi_{L_0-1}$ & $t \mapsto \alpha_{k,j} t$ & $[0,1] \to \bbR$ & $\leq m+3$ & $\lesssim N^{(1+\kappa)(1/p-\tbfs)_+}$ \\
\midrule
\multicolumn{3}{@{}l}{Total: depth $L_0 = 3 + 2\lceil \log_2 d \rceil$, width $2dN$, knot spacing $1/2$} & $\leq S_0 N$ & $\leq B_0 N^\beta$ \\
\bottomrule
\end{tabular}
\end{table}

\begin{proof}
\emph{Step 0: Sparse spline approximation.}
We apply Lemma~2 of \citet{suzuki2021deep} with $r = 2$; its conditions are satisfied under \eqref{assump:A1} and \eqref{assump:A4}.
The number of terms in Eq.~(11) of that lemma is at most $C_0 2^{\|K\|_{\bfs}}$ for a constant $C_0$ depending only on $(d,m,\kappa)$, where $\|K\|_{\bfs} := \sum_i \lfloor K\underline{\bfs}/s_i \rfloor \geq K$. We therefore set $N_0 := \lceil C_0 \rceil$ and, for each $N \geq N_0$, choose the largest $K$ such that $C_0 2^{\|K\|_{\bfs}} \leq N$. Then $2^K \leq N$, and with $K^\ast = \lceil (1+\kappa) K\rceil$, there exist an index set $E_N\subset\{(k,j) : 0 \leq k \leq K^\ast,\ j \in \bbZ^d\}$ and coefficients $(\alpha_{k,j})$ such that the function
\begin{equation}\label{eq:app-approx-spline}
\begin{aligned}
f_N^{\mathrm{sp}}
&:=
\sum_{(k,j)\in E_N}\alpha_{k,j}M_{k,j}^{d,m},
\quad
\abs{E_N}\le N,
\quad
\max_{(k,j)\in E_N}\abs{\alpha_{k,j}}
\lesssim
N^{(1+\kappa)(1/p-\tbfs)_+},    \\
M_{k,j}^{d,m}(\bx)
&:=
\prod_{i=1}^d
\psi_m\bigl(z_{k,j,i}(x_i)\bigr),
\quad
z_{k,j,i}(x_i) := 2^{a_{k,i}}x_i-j_i,
\quad
a_{k,i} := \lfloor k \underline{\bfs} / s_i \rfloor,
\end{aligned}
\end{equation}
with $j_i \in \{-m, \ldots, 2^{a_{k,i}}\}$, satisfies
\begin{equation}\label{eq:app-approx-rate}
\norm{f_0-f_N^{\mathrm{sp}}}_{L^2}
\lesssim
N^{-\tbfs}.
\end{equation}
We now construct, layer by layer, $f_N^\star = \Phi^\star_{L_0-1} \circ \cdots \circ \Phi^\star_0 \in KAN_c$ that exactly realizes $f_N^{\mathrm{sp}}$ on $[0,1]^d$. We denote the parameters of this network by $\btheta^\star_N$. In each layer, all edges not explicitly specified are set to $0$.

For $N < N_0$, we set $f_N^\star := 0$ (i.e., $\btheta_N^\star = \mathbf 0$), which gives $\norm{f_0 - f_N^\star}_{L^2} \lesssim 1 \leq N_0^{\tbfs} N^{-\tbfs}$. Hence, in what follows, we prove the claim for $N \geq N_0$.

We choose $B_0 \geq (m+1)/2$ so that $H(N) \geq m+1$ ($B_0$ will be enlarged finitely many more times below).
Since the hidden-layer knot vector is equally spaced on $[-H(N), H(N)]$, \cref{lem:cardinal,lem:exact_product} apply with $H = H(N)$.
Since $\|K\|_{\bfs} \geq K\underline{\bfs}/\tbfs - d$ and $2^{\|K\|_{\bfs}} \leq N$, we have $2^K \leq (2^d N)^{\tbfs/\underline{\bfs}}$.
Therefore, from $K^\ast \leq (1+\kappa)K + 1$ and $\tbfs \leq \underline{\bfs}$, we obtain, with $c_1 := 2^{1+d(1+\kappa)}$,
\begin{equation}\label{eq:app-scale-bound}
    \max_{(k,j) \in E_N,\ i \in [d]} 2^{a_{k,i}}
    \leq
    2^{K^\ast}
    \leq
    c_1 N^{(1 + \kappa) \tbfs / \underline{\bfs}}.
\end{equation}

\emph{Step 1: Affine layer $\Phi^\star_0$.}
Choose an injection $\nu_0: E_N \times [d] \rightarrow [ D (N) ]$, which is possible since $d\abs{E_N} \leq D(N)$. Let $\bxi^{(0)}$ be the first-layer knot vector with spacing $\Delta_0 > 0$ satisfying $[\xi^{(0)}_0, \xi^{(0)}_{G_0}] \supseteq [0, 1]$. Applying \cref{lem:marsden} to $\bxi^{(0)}$ with $I = [0,1]$ and $P(u) = 2^{a_{k,i}} u - j_i$, we obtain
\[
 (\Phi^{\star}_0 (\bx) )_{\nu_0 (k,j,i)} = \phi^{\star(0)}_{\nu_0(k,j,i), i} (x_i) = z_{k,j,i} (x_i),
 \quad x_i \in [0,1].
\]
We denote the coefficients of each edge by $\big( \theta_{\nu_0(k,j,i), i, l}^{\star (0)} \big)_l$. Since $\abs{j_i} \leq m + 2^{a_{k,i}} \leq (m+1)2^{a_{k,i}}$ and $z_{k,j,i}([0,1]) = [-j_i,\ 2^{a_{k,i}} - j_i] \subseteq [-2^{a_{k,i}},\ 2^{a_{k,i}} + m]$, it follows from \cref{eq:app-scale-bound} that, with $R_0 := 1 + (m+1)\Delta_0$,
\begin{equation}    \label{eq:app-affine-range}
\begin{aligned}
    \max_{l} \abs{\theta_{\nu_0(k,j,i), i, l}^{\star (0)}}
    &\leq R_0 2^{a_{k,i}} + \abs{j_i}
    \leq \big((R_0 + 1) c_1 + m \big) N^{(1 + \kappa) \tbfs / \underline{\bfs}},\\
    \max_{(k,j) \in E_N,\ i \in [d]} \sup_{x_i \in [0,1]} \abs{z_{k,j,i}(x_i)}
    &\leq (c_1 + m) N^{(1 + \kappa) \tbfs / \underline{\bfs}}.
\end{aligned}
\end{equation}
Each edge has at most $G_0 + m$ coefficients.

\emph{Step 2: Cardinal B-spline layer $\Phi^\star_1$.}
Using \cref{lem:cardinal}, we choose the edges $\phi^{\star(1)}_{\nu_0(k,j,i), \nu_0(k,j,i)}$ of the layer $\Phi^\star_1$ such that
\begin{equation*}
 \phi^{\star(1)}_{\nu_0(k,j,i), \nu_0(k,j,i)} = \psi_m \quad \text{on } \mathbb{R}.
\end{equation*}
The coefficients of each edge satisfy $\| \btheta^{\star (1)}_{\nu_0(k,j,i)} \|_0 \leq m +2$ and $\| \btheta^{\star (1)}_{\nu_0(k,j,i)} \|_\infty \leq 2$, and the corresponding knot vector $\bxi^{(1)}$ is equally spaced on $[-H(N), H(N)]$ with spacing $1/2$. Therefore,
\[
    \bigl( \Phi^{\star}_1 \circ \Phi^{\star}_0 (\bx) \bigr)_{\nu_0(k,j,i)} = \psi_m(z_{k,j,i}(x_i)) \in [0,1].
\]

\emph{Step 3: Product layers $\Phi^\star_2, \ldots, \Phi^\star_{L_\times + 1}$.}
If $d = 1$, then $L_\times = 0$ and this step is void. Let $d \geq 2$,
let $\Pi_d = \Phi^\times_{L_\times - 1} \circ \cdots \circ \Phi^\times_0$ be the product network of \cref{lem:exact_product}, and let $n^\times_\ell$ denote the output width of $\Phi^\times_{\ell-1}$.
That is, $n^\times_0 = d$, $n^\times_{L_\times} = 1$, and $n^\times_\ell \leq \lfloor 3d/2 \rfloor$ otherwise.
For each $\ell = 0, \ldots, L_\times$, we choose pairwise disjoint node sets
$J_\ell(k,j) \subseteq [D(N)]$ with $\abs{J_\ell(k,j)} = n^\times_\ell$ for $(k,j) \in E_N$,
where $J_0(k,j) := \{\nu_0(k,j,1), \ldots, \nu_0(k,j,d)\}$, and identify the elements of each $J_\ell(k,j)$ with $[n^\times_\ell]$ in increasing order. Each layer uses at most
$\lfloor 3d/2 \rfloor \abs{E_N} \leq D(N)$ nodes.

For $\ell = 0, \ldots, L_\times - 1$, we define the parallel combination
$\bigoplus_{(k,j) \in E_N} \Phi^\times_\ell : \bbR^{D(N)} \to \bbR^{D(N)}$ of $\Phi^\times_\ell$
as the layer obtained by copying the corresponding edges of $\Phi^\times_\ell$ only to the edges with $p \in J_\ell(k,j)$ and $q \in J_{\ell+1}(k,j)$, and setting all remaining edges, including those between different blocks, to $0$. That is,
\begin{equation*}
    \Bigl( \bigoplus_{(k,j) \in E_N} \Phi^\times_\ell (\mathbf v) \Bigr)_{J_{\ell+1}(k,j)}
    = \Phi^\times_\ell \bigl( \mathbf v_{J_\ell(k,j)} \bigr),
    \qquad (k,j) \in E_N,
\end{equation*}
which is a $KAN_c$ layer of the form \cref{eq:KANs_layer}. We now set
\begin{equation*}
    \Phi^\star_{\ell+2} := \bigoplus_{(k,j) \in E_N} \Phi^\times_\ell,
    \qquad \ell = 0, \ldots, L_\times - 1.
\end{equation*}
Let $\mathbf u^{(k,j)}(\bx) := \bigl( \psi_m(z_{k,j,i}(x_i)) \bigr)_{i=1}^d \in [0,1]^d$.
Since the blocks are computed independently of one another, induction on $\ell$ gives
\begin{equation*}
    \bigl( \Phi^\star_{\ell+2} \circ \cdots \circ \Phi^\star_0 (\bx) \bigr)_{J_{\ell+1}(k,j)}
    = \Phi^\times_\ell \circ \cdots \circ \Phi^\times_0 \bigl( \mathbf u^{(k,j)}(\bx) \bigr).
\end{equation*}
In particular, for $\ell = L_\times - 1$, $J_{L_\times}(k,j)$ is a single node, and by \cref{lem:exact_product}(i), its value is
$\Pi_d(\mathbf u^{(k,j)}(\bx)) = M^{d,m}_{k,j}(\bx)$.
By \cref{lem:exact_product}(ii), all intermediate values lie in $[0,1]$, and both the number and the magnitude of the nonzero weights per block depend only on $(d,m)$.

\emph{Step 4: Output layer $\Phi^\star_{L_\times + 2}$.}
For each block $(k,j)$, we define the edge from its output node to the network output by \cref{lem:marsden} with $I = [0,1]$ and $P(t) = \alpha_{k,j} t$; this edge has at most $m+3$ coefficients, each bounded in absolute value by $\abs{\alpha_{k,j}}(m+3)/2$. Then $f_N^\star(\bx) = f_N^{\mathrm{sp}}(\bx)$ for $\bx \in [0,1]^d$, and \cref{eq:app-approx-rate} yields \cref{eq:approx-main}.

\emph{Step 5: Architecture and parameter bounds.}
From Steps 1--4 and \cref{lem:exact_product}(iii)(iv), we obtain
\begin{equation}\label{eq:app-arch-summary}
\begin{aligned}
    \text{Depth}&: L_0 = 1 + 1 + L_\times + 1 = 3 + 2\lceil \log_2 d \rceil, \\
    \text{Width}&: \max\{d, \lfloor 3d/2 \rfloor\}\abs{E_N} \leq 2dN = D(N), \\
    \text{Sparsity}&: \norm{\btheta_N^\star}_0 \leq S_0 \abs{E_N} \leq S_0 N,\\
    &\qquad S_0 := d(G_0 + m) + d(m+2) + \bigl(7(d-1) + 2\lceil \log_2 d \rceil\bigr)(m+3) + (m+3),\\
    \text{Weights}&: \norm{\btheta_N^\star}_\infty \leq
    \max\Bigl\{ \bigl( (R_0 + 1) c_1 + m \bigr) N^{(1 + \kappa) \tbfs / \underline{\bfs}},\ 2,\ \tfrac{(m+3)^2}{2},\ \tfrac{m+3}{2} \max_{(k,j)\in E_N} \abs{\alpha_{k,j}} \Bigr\}\\
    &\qquad\leq B_0 N^\beta.
\end{aligned}
\end{equation}

The last inequality holds for sufficiently large $B_0$, since $\beta \geq (1+\kappa) \tbfs / \underline{\bfs}$ and $\beta \geq (1+\kappa)(1/p-\tbfs)_+$ by \cref{eq:beta}, together with \cref{eq:app-approx-spline}.
As for the ranges of the layers, the second bound in \cref{eq:app-affine-range} does not exceed the first, which is in turn bounded by $B^\star(N)$ by \cref{eq:app-arch-summary}. Hence,
\[
    \max_{k,j,i} \sup_{x_i \in [0,1]} \abs{z_{k,j,i}(x_i)} \leq B^\star(N) \leq H(N),
\]
that is, the inputs to $\Phi^\star_1$ lie in $[-H(N), H(N)]$, and the inputs to all subsequent layers lie in $[0,1] \subseteq [-H(N), H(N)]$. Moreover, since $H(N) \leq 2B_0N^\beta + 1$, we have $H(N) \asymp N^\beta$ and $G(N) = 4H(N) - 2m \geq 2m+4$.
Therefore, $f_N^\star \in KAN_c(L_0, 2dN, 4H(N) - 2m, H(N); m)$, and all the claims hold.
\end{proof}

\begin{remark}[Enlarging the architecture]\label{rem:approx_monotone}
The proof above uses the architecture only through the following three properties: (a) $H(N) \in \bbN$ with $H(N) \geq \max\{m+1,\ \max_{k,j,i}\sup_{x_i}\abs{z_{k,j,i}(x_i)}\}$; (b) $G(N) = 4H(N) - 2m$, so that the hidden-layer knot spacing is $1/2$; and (c) all weights are bounded by $B^\star(N)$. The representations in \cref{lem:marsden,lem:cardinal,lem:exact_product} use only half-integer knots in $[-(m+1)/2,\ m+1]$. Hence, for any $B_0' \geq B_0$ and $\beta' \geq \beta$, with $H'(N) := \lceil 2B_0' N^{\beta'} \rceil$, the same network is realized in $KAN_c(L_0, D(N), 4H'(N) - 2m, H'(N); m)$, up to a reindexing of the hidden-layer basis functions. In particular, $\kappa$ in \cref{eq:beta} may be replaced by any $\kappa' \geq \kappa$, and for $p \geq 2$, $\kappa = 0$ is the smallest admissible value.
\end{remark}

\subsection{Compositional approximation}
\label{subsec:comp_approx}

\begin{lemma}[Layer-wise approximation by fixed-knot $KAN_c$]
\label{lem:comp_layer_approx}
For $t \in \bbN$, $\bfs \in \bbR_{++}^t$, and $r \in \{2, \infty\}$, let
\[
    \tbfs :=
    \Bigl(\sum_{i=1}^{t} s_i^{-1}\Bigr)^{-1},
    \qquad
    \underline{\bfs} := \min_{i} s_i,
    \qquad
    \overline{\bfs} := \max_{i} s_i,
    \qquad
    \omega_r := (1/p - 1/r)_+.
\]
Assume that $\omega_r < \tbfs$ and $0 < \overline{\bfs} < \min\{m,\, m - 1 + 1/p\}$, and set
\[
\begin{aligned}
    \kappa_r &:= \frac{2\omega_r}{\tbfs - \omega_r},
    &
    \beta_r &:=
    (1 + \kappa_r) \max\bigl\{(1/p - \tbfs)_+,\, \tbfs / \underline{\bfs}\bigr\},\\
    B_r^*(N) &:= B_r N^{\beta_r},
    &
    H_r^*(N) &:= \lceil 2B_r^*(N) \rceil
\end{aligned}
\]
with $\kappa_r = 0$ if $\omega_r = 0$.
Then there exist constants $L_r \in \bbN$ and $D_r, S_r, B_r > 0$, depending only on $t, m, p, \bfs, r$ and an upper bound on $\| g_0 \|_{\cB_{p,q}^\bfs ([0,1]^t)} \vee \| g_0 \|_\infty$ such that for every $N \in \bbN$ and every bounded function $g_0$ with $g_0|_{[0,1]^t} \in \cB^{\bfs}_{p,q}([0,1]^t)$, there exists $g_N$, realized by a parameter vector $\btheta_N$, satisfying
\[
\begin{aligned}
    &g_N \in KAN_c\bigl(L_r,\,  D_r N ,\,  4H_r^*(N) - 2m ,\, H_r^*(N);\, m\bigr),\\
    &\|\btheta_N\|_\infty \leq B_r^*(N),\qquad \|\btheta_N\|_0 \leq S_r N,
    \qquad
    \|g_0 - g_N\|_{L^r([0,1]^t)} \lesssim N^{-\tbfs}.
\end{aligned}
\]
\end{lemma}

\begin{proof}
For $r = 2$, the claim follows by applying \cref{lem:approx} to the $t$-variate function $g_0$.
For $r = \infty$, Lemma~2 of \citet{suzuki2021deep} applies since $\omega_\infty = 1/p < \tbfs$. With $N_0$ and $K$ chosen as in Step~0 of the proof of \cref{lem:approx}, now with $\|K\|_{\bfs} = \sum_{i=1}^t \lfloor K \underline{\bfs}/s_i \rfloor$, it yields for $N \geq N_0$ an expansion $g_N^{\mathrm{sp}}$ of the form \cref{eq:app-approx-spline} with $\abs{E_N} \leq N$, $k \leq K^\ast = \lceil (1+\kappa_\infty)K \rceil$, $\max\abs{\alpha_{k,j}} \lesssim N^{(1+\kappa_\infty)(1/p - \tbfs)_+}$, and $\norm{g_0 - g_N^{\mathrm{sp}}}_{L^\infty([0,1]^t)} \lesssim N^{-\tbfs}$; in particular, \cref{eq:app-scale-bound} holds with $(d, \kappa)$ replaced by $(t, \kappa_\infty)$. Steps 1--5 of that proof use only these properties, not the norm in which the error is measured, so the same construction realizes $g_N^{\mathrm{sp}}$ with depth $3 + 2\lceil \log_2 t \rceil$, width $2tN$, sparsity $S_r N$, and weights bounded by $B_r N^{\beta_r}$. For $N < N_0$, set $g_N = 0$ as in Step~0.
\end{proof}

\begin{lemma}[Approximate clipping]\label{lem:range_correction}
Let $m \geq 2$, $H \in \bbN$ with $H \geq 2m + 3$, $G := 4H - 2m$, and let $\bxi$ be the knot vector of \cref{lem:cardinal}. For $a \geq 1$, let $\eta_a := 2/(ea)$. There exist KAN layers $\Phi^{\mathrm{rc}}_0 : \bbR \to \bbR^3$ and $\Phi^{\mathrm{rc}}_1 : \bbR^3 \to \bbR$ with knot vector $\bxi$ such that $R_a := \Phi^{\mathrm{rc}}_1 \circ \Phi^{\mathrm{rc}}_0$ satisfies the following:
\begin{enumerate}[label=(\roman*)]
    \item $R_a(\bbR) \subseteq [0,1]$;
    \item $\abs{R_a(u) - v} \leq \abs{u - v} + 2\eta_a$ for all $u \in \bbR$ and $v \in [0,1]$;
    \item $R_a$ has at most $18m + 40$ nonzero weights, each bounded in absolute value by $(m+3)a/2$.
\end{enumerate}
\end{lemma}

\begin{proof}
Write $\operatorname{clip}_{[\alpha,\beta]}(w) := \min\{\beta, \max\{\alpha, w\}\}$ and $F_a(x) := a^{-1}\operatorname{SiLU}(ax)$. Since $\abs{\operatorname{SiLU}(x) - x_+} = \abs{x}/(1+e^{\abs{x}}) \leq e^{-1}$,
\begin{equation}\label{eq:rc-silu}
    \sup_{x \in \bbR} \abs{F_a(x) - x_+} \leq (ea)^{-1} = \eta_a/2 .
\end{equation}
We approximate $\operatorname{clip}_{[0,1]}(w) = w_+ - (w-1)_+$ by replacing $x_+$ with $F_a$. Since the spline edges vanish outside $[\xi_{-m}, \xi_{G+m}]$ and the SiLU edges cannot produce a constant, the first layer maps $u$ into a bounded interval by a spline $C$ and produces the constant $1$ by a second spline $S$.

Let $\varrho := (m+1)/2$, $I := [-1-2\varrho,\ 2+2\varrho]$, $\cK := \{k : (\xi_{k-m-1}, \xi_k) \cap I \neq \emptyset\}$, and $\xi^\ast_k := (\xi_{k-m} + \cdots + \xi_{k-1})/m$, which equals $\lambda_{k,1}$ in \cref{eq:marsden-monomial}. Define
\[
    C(u) := \sum_{k \in \cK} \operatorname{clip}_{[-\varrho,1+\varrho]}(\xi_k^\ast)\, B_{k,m}(u; \bxi),
    \qquad
    S(u) := \sum_{k \in \cK} B_{k,m}(u; \bxi),
    \qquad u \in \bbR.
\]
Since $H \geq 2m+3$, we have $I \subseteq [\xi_0, \xi_G]$. We claim that
\begin{enumerate}[label=(\alph*)]
    \item $0 \leq S \leq 1$ on $\bbR$, and $S = 1$ on $I$;
    \item $C(\bbR) \subseteq [-\varrho, 1+\varrho]$;
    \item $\abs{C(u) - v} \leq \abs{u - v}$ for $u \in I$ and $v \in [0,1]$.
\end{enumerate}
Claims (a) and (b) follow from $B_{k,m} \geq 0$ and the partition of unity on $[\xi_0, \xi_G]$: every $B_{k,m}$ that does not vanish on $I$ has $k \in \cK$, and $C(u)$ is a nonnegative combination of values in $[-\varrho, 1+\varrho] \ni 0$ with total weight $S(u) \leq 1$.
For (c), we first show that $C(u) = u$ on $[0,1]$. If $B_{k,m}(u) \neq 0$, then $\abs{\xi_k^\ast - u} \leq (m+1)\Delta = \varrho$, so the clipping is inactive, and $C(u) = \sum_k \xi_k^\ast B_{k,m}(u) = u$ by \cref{eq:marsden-monomial}.
Next, $C$ is $1$-Lipschitz on $I$: by \cref{eq:bSpline_deriv}, for $u \in I$, $C'(u)$ is a nonnegative combination, with total weight at most $1$, of the difference quotients $\{\operatorname{clip}_{[-\varrho,1+\varrho]}(\xi_{k+1}^\ast) - \operatorname{clip}_{[-\varrho,1+\varrho]}(\xi_k^\ast)\}/\Delta \in [-1,1]$ with $k, k+1 \in \cK$ (the boundary terms $k = \min\cK - 1$ and $k = \max\cK$ vanish on $I$, since $B_{k,m-1}$ is supported on $[\xi_{k-m}, \xi_k]$ while $\xi_{\min\cK - 1} \leq \inf I$ and $\xi_{\max\cK - m} \geq \sup I$ by the definition of $\cK$).
Hence, for $u \in I$ and $v \in [0,1] \subseteq I$, $\abs{C(u) - v} = \abs{C(u) - C(v)} \leq \abs{u - v}$.
Let $\varphi_{\mathrm c}$ be the spline edge of \cref{lem:marsden} that equals the constant $\eta_a/(1+2\eta_a)$ on $[-\varrho, 1+\varrho]$, and set
\begin{equation}\label{eq:rc-layers}
\begin{aligned}
    \Phi^{\mathrm{rc}}_0(u) &:= \bigl(C(u),\ aC(u),\ a\{C(u) - S(u)\}\bigr),\\
    \Phi^{\mathrm{rc}}_1(z, y_1, y_2) &:= \varphi_{\mathrm{c}}(z) + \frac{\operatorname{SiLU}(y_1) - \operatorname{SiLU}(y_2)}{a(1+2\eta_a)}.
    \end{aligned}
\end{equation}
Each coordinate of $\Phi^{\mathrm{rc}}_0$ is a single spline edge, and $\Phi^{\mathrm{rc}}_1$ consists of one spline edge and two SiLU weights. By (b), $\varphi_{\mathrm c}(C(u)) = \eta_a/(1+2\eta_a)$, so with $w := C(u)$ and $s := S(u)$,
\[
    R_a(u) = \frac{F_a(w) - F_a(w - s) + \eta_a}{1 + 2\eta_a}.
\]
Fix $u \in \bbR$ and let $c := w_+ - (w - s)_+ = \operatorname{clip}_{[0,s]}(w) \in [0,1]$. By \cref{eq:rc-silu}, $F_a(w) - F_a(w - s) = c + \delta$ with $\abs{\delta} \leq \eta_a$, and therefore
\[
    R_a(u) = \frac{c + \delta + \eta_a}{1+2\eta_a} \in [0,1],
    \qquad
    \abs{R_a(u) - c} = \frac{\abs{\delta + \eta_a(1 - 2c)}}{1 + 2\eta_a} \leq 2\eta_a,
\]
which proves (i). For (ii), let $v \in [0,1]$. If $u \in I$, then $s = 1$ by (a), so $c = \operatorname{clip}_{[0,1]}(C(u))$ and $\abs{c - v} \leq \abs{C(u) - v} \leq \abs{u - v}$ by (c). If $u \notin I$, then $\abs{u - v} \geq 1 + 2\varrho \geq 1 \geq \abs{c - v}$. In both cases, $\abs{R_a(u) - v} \leq \abs{c - v} + 2\eta_a \leq \abs{u - v} + 2\eta_a$.
By \cref{lem:marsden}(ii), $\abs{\cK} \leq \lceil \abs{I}/\Delta \rceil + m + 1 = 5m + 11$ and $\norm{\varphi_{\mathrm c}}_0 \leq 3m + 5$, so the number of nonzero weights is at most $3(5m+11) + (3m+5) + 2 = 18m + 40$. The weights of $\Phi^{\mathrm{rc}}_0$ are bounded by $a(1+\varrho) = (m+3)a/2$, and those of $\Phi^{\mathrm{rc}}_1$ by $1$.
\end{proof}

\begin{lemma}[H\"older continuity of the outer layers]\label{lem:holder_comp}
Assume \eqref{assump:D1} and \eqref{assump:D2}, and for $k = 2, \ldots, J$, let $\gamma_k := 1 - 1/(p\tilde{s}^{(k)})$ and $\alpha_k := \gamma_k \underline{\bfs}^{(k)} = \underline{\bfs}^{(k)} - t^{*(k)}/p$. Then, for all $k \geq 2$ and $i \in [d^{(k)}]$, $\tilde{f}_{k,i}$ has a continuous representative, and there exist constants $A_k, \Lambda_j > 0$ such that
\[
\begin{aligned}
    \abs{\tilde{f}_{k,i}(\by) - \tilde{f}_{k,i}(\by')} &\leq A_k \norm{\by - \by'}_\infty^{\alpha_k \wedge 1},\\
    \abs{F_{0,j}(\mathbf{u}) - F_{0,j}(\bv)} &\leq \Lambda_j \norm{\mathbf{u} - \bv}_\infty^{\Gamma_j},
    \qquad
    \Gamma_j := \prod_{k=j+1}^{J} (\alpha_k \wedge 1),
\end{aligned}
\]
for all $\by, \by' \in [0,1]^{t^{(k)}}$, $\mathbf{u}, \bv \in [0,1]^{d^{(j)}}$, and $j \in [J]$. Here, $F_{0,j} := f_{0,J} \circ \cdots \circ f_{0,j+1}$, $F_{0,J} := \mathrm{id}$, and $\Gamma_J := 1$.
\end{lemma}

\begin{proof}
Fix $k \geq 2$ and $i$, and write $f := \tilde{f}_{k,i}$, $t := t^{(k)}$, and $s_\ell' := \gamma_k s^{(k)}_\ell$, so that $\min_\ell s_\ell' = \alpha_k$. Since $\tilde{s}^{(k)} > 1/p$, we have $\gamma_k \in (0,1]$, and by \citet[Proposition~1, item~2]{suzuki2021deep} with $(p_1, p_2) = (p, \infty)$, $f \in \cB^{\bfs'}_{\infty,q}([0,1]^t)$. Moreover, by \citet[Proposition~1, item~4]{suzuki2021deep}, $\cB^{\bfs^{(k)}}_{p,q}([0,1]^t) \hookrightarrow C^0([0,1]^t)$, so $f$ has a continuous representative; we identify $f$ with it throughout the proof.

Let $\omega_{r,\ell}(f, \tau) := \sup_{\abs{h} \leq \tau} \norm{\Delta^r_{h \bfe_\ell}(f)}_{L^\infty}$ with $r = \max_\ell \lfloor s_\ell' \rfloor + 1$. Restricting the supremum in the anisotropic modulus to $\bh = h\bfe_\ell$ gives $\omega_{r,\ell}(f, 2^{-n/s_\ell'}) \leq 2^{-n} \abs{f}_{\cB^{\bfs'}_{\infty,q}}$, and by monotonicity,
\begin{equation}\label{eq:holder-dir-mod}
    \omega_{r,\ell}(f, \tau) \leq 2\abs{f}_{\cB^{\bfs'}_{\infty,q}}\, \tau^{s_\ell'},
    \qquad \tau \in (0,1].
\end{equation}

Fix $\ell$ and the remaining coordinates $\by_{-\ell}$, and let $g(u) := f(\by_{-\ell}, u)$ for $u \in [0,1]$, where $u$ occupies the $\ell$th coordinate. Denote by $\omega_r(g, \cdot)$ the univariate $r$th modulus of smoothness of $g$ on $[0,1]$. Since $f$ is continuous, $\omega_r(g, \tau) \leq \omega_{r,\ell}(f, \tau)$ and $\norm{g}_{L^\infty} \leq \norm{f}_{L^\infty}$ for every $\by_{-\ell}$, while $\omega_{1,\ell}(f, \tau) \leq \sup_{\by_{-\ell}} \omega_1(g, \tau)$. Applying the Marchaud inequality \citep[Ch.~2]{devore1993constructive} to $g$, whose constant does not depend on $\by_{-\ell}$, taking the supremum over $\by_{-\ell}$, and substituting \cref{eq:holder-dir-mod}, we obtain
\begin{equation}\label{eq:holder-marchaud}
    \omega_{1,\ell}(f, \tau) \lesssim \tau \Bigl( \int_\tau^1 \frac{\omega_{r,\ell}(f, u)}{u^2} \, du + \norm{f}_{L^\infty} \Bigr)
    \lesssim
    \begin{cases}
        \tau^{s_\ell' \wedge 1}, & s_\ell' \neq 1, \\
        \tau \log(e/\tau), & s_\ell' = 1,
    \end{cases}
    \qquad \tau \in (0,1].
\end{equation}
\begin{itemize}
    \item $\alpha_k < 1$: Since $s'_\ell \geq \alpha_k$ and $\tau\log(e/\tau) \lesssim \tau^{\alpha_k}$, we have $\omega_{1,\ell}(f,\tau) \lesssim \tau^{\alpha_k}$ for all $\ell$.
    \item $\alpha_k > 1$: Since $s'_\ell > 1$ for all $\ell$, we have $\omega_{1,\ell}(f,\tau) \lesssim \tau$.
    \item $\alpha_k = 1$: By \eqref{assump:D2}, $q \leq 1$. For $\ell$ with $s'_\ell = 1$, the definition of the quasi-norm and $q \leq 1$ yield
    \[
        \int_\tau^1 \frac{\omega_{r,\ell}(f, u)}{u^2} \, du
        \leq \sum_{n \geq 0} 2^{n} \omega_{r,\ell}(f, 2^{-n})
        \leq \Bigl(\sum_{n \geq 0} \{2^n \omega_{r,\ell}(f, 2^{-n})\}^q\Bigr)^{1/q}
        \leq \abs{f}_{\cB^{\bfs'}_{\infty,q}},
    \]
    and hence the first inequality in \cref{eq:holder-marchaud} gives $\omega_{1,\ell}(f,\tau) \lesssim \tau$.
\end{itemize}
In all cases, $\omega_{1,\ell}(f, \tau) \lesssim \tau^{\alpha_k \wedge 1}$.

Changing one coordinate at a time gives $\abs{f(\by) - f(\by')} \leq \sum_{\ell} \omega_{1,\ell}(f, \abs{y_\ell - y_\ell'}) \leq A_k \norm{\by - \by'}_\infty^{\alpha_k \wedge 1}$. Since $f_{0,k,i}(\mathbf{u}) = \tilde{f}_{k,i}(\mathbf{u}_{I_{k,i}})$, we have $\norm{f_{0,k}(\mathbf{u}) - f_{0,k}(\bv)}_\infty \leq A_k \norm{\mathbf{u} - \bv}_\infty^{\alpha_k \wedge 1}$. Backward induction starting from $j = J$ (with $\Gamma_J = 1$) then yields
\[
\begin{aligned}
    \abs{F_{0,j-1}(\mathbf{u}) - F_{0,j-1}(\bv)}
    &= \abs{F_{0,j}(f_{0,j}(\mathbf{u})) - F_{0,j}(f_{0,j}(\bv))}
    \leq \Lambda_j \norm{f_{0,j}(\mathbf{u}) - f_{0,j}(\bv)}_\infty^{\Gamma_j}\\
    &\leq \Lambda_j A_j^{\Gamma_j} \norm{\mathbf{u} - \bv}_\infty^{(\alpha_j \wedge 1)\Gamma_j}
    = \Lambda_j A_j^{\Gamma_j} \norm{\mathbf{u} - \bv}_\infty^{\Gamma_{j-1}},
\end{aligned}
\]
where the induction hypothesis applies because $f_{0,j}$ maps $[0,1]^{d^{(j-1)}}$ into $[0,1]^{d^{(j)}}$.
\end{proof}

We elaborate on the roles of \eqref{assump:D1} and \eqref{assump:D2}. Since H\"older exponents multiply under composition, an error of order $N^{-\tilde{s}^{(j)}}$ in the $j$th layer becomes $N^{-\tilde{s}^{(j)} \prod_{k>j} (\alpha_k \wedge 1)} = N^{-\tilde{s}^{*(j)}}$ at the output. If $\alpha_k > 1$, Lipschitz continuity suffices, which explains the truncation $\wedge 1$. For the innermost layer, only the $L^2$-approximation condition $\tilde{s}^{(1)} > (1/p - 1/2)_+$ is required. The last condition in \eqref{assump:D2} excludes only the boundary case $\alpha_j = 1$ when $q > 1$. In that case, elements of $\cB^{\gamma_j\bfs^{(j)}}_{\infty,q}$ need not be Lipschitz (e.g., $u \mapsto u\log u$ for $q = \infty$), so the factor $\alpha_j \wedge 1 = 1$ would incur an additional logarithmic loss through the term $\tau\log(e/\tau)$.

% \begin{remark}[Details of the unit-cube normalization]\label{rem:unit_cube_details}
% Suppose that, in \cref{def:comp_besov}, $f_j$ maps into a bounded rectangle $Q \subset \bbR^{d^{(j)}}$ that is the domain of $f_{j+1}$. Let $A_j : Q \to [0,1]^{d^{(j)}}$ be a coordinatewise affine bijection. Then the map $(f_j, f_{j+1}) \mapsto (A_j \circ f_j,\ f_{j+1} \circ A_j^{-1})$ leaves the composition unchanged, and a fixed coordinatewise affine change of variables preserves anisotropic Besov regularity with the same smoothness vector while changing the norm only by a constant factor. Moreover, the coordinate set $I_{j,k}$ on which each component depends is unchanged. Even if the unnormalized representation is used as is, $A_j$ and $A_j^{-1}$ are exactly realized by fixed-knot KAN edges by \cref{lem:marsden}, so absorbing them into the oracle network changes neither the contraction rate nor the order of the architecture.
% \end{remark}

\begin{lemma}[Compositional approximation by fixed-knot KAN]
\label{lem:comp_approx}
Assume \eqref{assump:D1} and \eqref{assump:D2}. There exist constants $L_{\mathrm{cp}} \in \bbN$ and $D_{\mathrm{cp}}, S_{\mathrm{cp}}, B_{\mathrm{cp}}, \beta_{\mathrm{cp}} > 0$ such that for every $N \in \bbN$, there exists $f_N^{\mathrm{cp},*}$, realized by a parameter vector $\btheta_N^{\mathrm{cp}}$, satisfying
\[
\begin{aligned}
    &f_N^{\mathrm{cp},*}
    \in KAN\bigl(
        L_{\mathrm{cp}},\,
        D_{\mathrm{cp}} N,\,
        4H_{\mathrm{cp}}^*(N) - 2m,\,
        H_{\mathrm{cp}}^*(N);\,
        m
    \bigr),\\
    &\|\btheta_N^{\mathrm{cp}}\|_\infty \leq B_{\mathrm{cp}}^*(N),
    \qquad
    \|\btheta_N^{\mathrm{cp}}\|_0 \leq S_{\mathrm{cp}} N,
\end{aligned}
\]
and
\begin{equation}    \label{eq:comp_approx_rate}
    \|f_0 - f_N^{\mathrm{cp},*}\|_{L^2(P_X)} \lesssim N^{-\tilde{s}^*}.
\end{equation}
Here, $B_{\mathrm{cp}}^*(N) := B_{\mathrm{cp}} N^{\beta_{\mathrm{cp}}}$,
$H_{\mathrm{cp}}^*(N) := \lceil 2B_{\mathrm{cp}}^*(N) \rceil$, and $L_{\mathrm{cp}} = \sum_{j=1}^J (3 + 2\lceil \log_2 t^{(j)} \rceil) + 2(J-1)$.
\end{lemma}

\begin{proof}
If $J = 1$, the claim reduces to \cref{lem:approx}; hence we assume $J \geq 2$.

Let $r^{(1)} := 2$ and $r^{(j)} := \infty$ for $j \geq 2$. By \eqref{assump:D1} and \eqref{assump:D2}, the hypotheses of \cref{lem:comp_layer_approx} hold for $(t^{(j)}, \bfs^{(j)}, r^{(j)})$, namely $\omega_{r^{(j)}} < \tilde{s}^{(j)}$ and $\overline{\bfs}^{(j)} < \min\{m,\, m-1+1/p\}$. Let $L_j, D_j, S_j, B_j, \beta_j$, and $H_j^*(N)$ denote its constants, so that $L_j = 3 + 2\lceil \log_2 t^{(j)} \rceil$. Since each $\tilde{f}_{j,k}$ is bounded (it takes values in $[0,1]$ for $j < J$, and for $j = J$ it is continuous by \cref{lem:holder_comp} or equals $f_0 \in \UB$) and there are finitely many $(j,k)$, the constants may be taken uniform in $k$. Thus, for all $j \in [J]$ and $k \in [d^{(j)}]$, there exists
\[
    h_{j,k,N} \in KAN_c\bigl(L_j,\, D_j N,\, 4H_j^*(N) - 2m,\, H_j^*(N);\, m\bigr),
\]
realized by $\btheta_{j,k,N}$ with $\|\btheta_{j,k,N}\|_\infty \leq B_j N^{\beta_j}$ and $\|\btheta_{j,k,N}\|_0 \leq S_j N$, such that
\begin{equation}    \label{eq:comp_layer_err}
    \|\tilde{f}_{j,k} - h_{j,k,N}\|_{L^{r^{(j)}}([0,1]^{t^{(j)}})}
    \lesssim N^{-\tilde{s}^{(j)}}.
\end{equation}
For $j \geq 2$, both $\tilde{f}_{j,k}$ (\cref{lem:holder_comp}) and $h_{j,k,N}$ are continuous, so \cref{eq:comp_layer_err} holds pointwise on $[0,1]^{t^{(j)}}$.

For $j \geq 2$, the first layer of $h_{j,k,N}$ becomes a hidden layer of the composite network: its knot vector is $\bxi$ of \cref{lem:cardinal} instead of $\bxi^{(0)}$, and its input lies in $[0,1] \subseteq [\xi_0, \xi_G]$ as shown below. Step~1 in the proof of \cref{lem:approx} applies to this layer with $\Delta_0 = 1/2$, so that $R_0 = (m+3)/2$; let $B_j' \geq B_j$ denote the resulting weight bound, and set $B_1' := B_1$. To place all components on a common grid, set
\[
    \beta_{\mathrm{cp}} := \max\Bigl\{\max_{j \in [J]} \beta_j,\ \max_{j \in [J-1]} \tilde{s}^{(j)}\Bigr\},
    \qquad
    B_{\mathrm{cp}} := \max\Bigl\{m+2,\ \max_{j \in [J]} B_j'\Bigr\},
\]
and let $H := H_{\mathrm{cp}}^*(N)$ and $G := 4H - 2m$. Then $H \geq 2(m+2) > 2m+3$, so \cref{lem:range_correction} applies on the grid $(G,H)$, and since $B_{\mathrm{cp}} \geq B_j'$ and $\beta_{\mathrm{cp}} \geq \beta_j$, \cref{rem:approx_monotone} realizes each $h_{j,k,N}$ on the grid $(G,H)$ with the same weights. We extend $h_{j,k,N}$ to $d^{(j-1)}$ input variables by setting the edges from the coordinates outside $I_{j,k}$ to zero, and place the $d^{(j)}$ components, all of depth $L_j$, in parallel to obtain $h_{j,N} := (h_{j,1,N}, \ldots, h_{j,d^{(j)},N})$.

For $j \in [J-1]$, let $a_{j,N} := N^{\tilde{s}^{(j)}}$, $g_{j,k,N} := R_{a_{j,N}} \circ h_{j,k,N}$, and $g_{j,N} := (g_{j,1,N}, \ldots, g_{j,d^{(j)},N})$; let $g_{J,N} := h_{J,N}$. By \cref{lem:range_correction}(i), $g_{j,N}(\bbR^{d^{(j-1)}}) \subseteq [0,1]^{d^{(j)}}$ for $j \in [J-1]$. Since $f_{0,j,k}$ takes values in $[0,1]$, \cref{lem:range_correction}(ii) gives, for $j \in [J-1]$ and $\by \in [0,1]^{d^{(j-1)}}$,
\begin{equation}\label{eq:comp_rc_pw}
    \abs{g_{j,k,N}(\by) - f_{0,j,k}(\by)}
    \leq
    \abs{h_{j,k,N}(\by_{I_{j,k}}) - \tilde{f}_{j,k}(\by_{I_{j,k}})} + 2\eta_{a_{j,N}},~
    2\eta_{a_{j,N}} = \frac{4}{e} N^{-\tilde{s}^{(j)}}.
\end{equation}
For $2 \leq j \leq J-1$, \cref{eq:comp_rc_pw} and the pointwise form of \cref{eq:comp_layer_err} give the first bound below; for $j = J$, it is \cref{eq:comp_layer_err} itself:
\begin{equation}\label{eq:comp_rc_sup}
    \sup_{\by \in [0,1]^{d^{(j-1)}}} \norm{g_{j,N}(\by) - f_{0,j}(\by)}_\infty
    \lesssim N^{-\tilde{s}^{(j)}},
    \qquad j = 2, \ldots, J.
\end{equation}
For $j = 1$, the marginal density of $X_{I_{1,k}}$ is bounded by $C_X$ by \eqref{assump:A2}, so \cref{eq:comp_rc_pw} and \cref{eq:comp_layer_err} give
\begin{equation}\label{eq:comp_rc_l2}
    \norm{g_{1,k,N} - f_{0,1,k}}_{L^2(P_X)}
    \leq C_X^{1/2} \norm{h_{1,k,N} - \tilde{f}_{1,k}}_{L^2([0,1]^{t^{(1)}})} + 2\eta_{a_{1,N}}
    \lesssim N^{-\tilde{s}^{(1)}}.
\end{equation}

Let $f_N^{\mathrm{cp},*} := g_{J,N} \circ \cdots \circ g_{1,N}$. The $j$th block $h_{j,N}$ has depth $L_j$, and each of the $J-1$ range-correction modules adds two layers, so the depth of $f_N^{\mathrm{cp},*}$ is $\sum_{j=1}^J L_j + 2(J-1) = L_{\mathrm{cp}}$. The $d^{(j)}$ parallel components of $h_{j,N}$ have width at most $D_j N$ each, and the range-correction module following it has width $3d^{(j)}$, so padding every hidden layer by zero edges to the width $D_{\mathrm{cp}} N$ with $D_{\mathrm{cp}} := \max_j d^{(j)}(D_j + 3)$ suffices. The number of nonzero parameters is at most $\sum_{j=1}^J d^{(j)}(S_j + 18m + 40)\, N =: S_{\mathrm{cp}} N$. Every weight of $h_{j,N}$ is bounded in absolute value by $B_j' N^{\beta_j} \leq B_{\mathrm{cp}} N^{\beta_{\mathrm{cp}}}$, and every weight of a range-correction module by $(m+3) a_{j,N}/2 \leq (m+2) N^{\beta_{\mathrm{cp}}} \leq B_{\mathrm{cp}} N^{\beta_{\mathrm{cp}}}$, since $a_{j,N} = N^{\tilde{s}^{(j)}} \leq N^{\beta_{\mathrm{cp}}}$. Hence $f_N^{\mathrm{cp},*}$ belongs to the class stated in the lemma.

Let $G_{0,N} := \mathrm{id}$ and $G_{j,N} := g_{j,N} \circ \cdots \circ g_{1,N}$. As shown above, $G_{j,N}([0,1]^d) \subseteq [0,1]^{d^{(j)}}$ for $j \in [J-1]$. Since $F_{0,j}$ is evaluated only on $[0,1]^{d^{(j)}}$, the telescoping identity
\[
    f_0 - f_N^{\mathrm{cp},*}
    = \sum_{j=1}^J \Bigl[
        F_{0,j} \circ f_{0,j} \circ G_{j-1,N}
        - F_{0,j} \circ g_{j,N} \circ G_{j-1,N}
    \Bigr]
\]
is well defined on $[0,1]^d$. Writing $\by := G_{j-1,N}(\bx)$, \cref{lem:holder_comp}, \cref{eq:comp_rc_sup}, and \cref{eq:comp_rc_l2} give
\begin{align*}
    \abs{F_{0,j}(f_{0,j}(\by)) - F_{0,j}(g_{j,N}(\by))}
    &\leq \Lambda_j \norm{f_{0,j}(\by) - g_{j,N}(\by)}_\infty^{\Gamma_j}
    \lesssim N^{-\tilde{s}^{(j)}\Gamma_j} = N^{-\tilde{s}^{*(j)}},
    \qquad j \geq 2,\\
    \norm{F_{0,1} \circ f_{0,1} - F_{0,1} \circ g_{1,N}}_{L^2(P_X)}^2
    &\leq \Lambda_1^2 \Bigl( \int \norm{f_{0,1} - g_{1,N}}_\infty^2 \, dP_X \Bigr)^{\Gamma_1}\\
    &\leq \Lambda_1^2 \Bigl( \sum_{k} \norm{f_{0,1,k} - g_{1,k,N}}_{L^2(P_X)}^2 \Bigr)^{\Gamma_1}
    \lesssim N^{-2\tilde{s}^{*(1)}},
\end{align*}
where, for $j = 1$, we applied Jensen's inequality to the concave map $u \mapsto u^{\Gamma_1}$ ($\Gamma_1 \leq 1$). Summing over $j$ and using $\tilde{s}^* = \min_j \tilde{s}^{*(j)}$ yields \cref{eq:comp_approx_rate}.
\end{proof}

%%%%%%%%%%%%%%%%%%%%%%%%%%%%%%%%%%%%%%%%%%%%%%%%%%%%%%%%%%%%

\subsection{Model complexity}
\label{subsubsec:model_complexity}

%%%%%%%%%%%%%%%%%%%%%%%%%%%%%%%%%%%%%%%%%%%%%%%%%%%%%%%%%%%%
\subsubsection{Layerwise bounds for KANs}
\label{subsubsec:layerwise_bounds}

\begin{lemma}   \label{lem:KANs_lipschitz_bound}
      Let $L, m \in \bbN$ $(m \geq 2)$, $\bD \in \bbN^{L + 1}$, $\bG \in \bbN^L$, and $H > 0$. For any $f_{\btheta} \in KAN(L, \bD, \bG, H; m)$, write $f_{\btheta} = \Phi_{L - 1} \circ \cdots \circ \Phi_0$.
    Let its parameters be denoted by $\btheta = (\bW^{(l)}) = (\theta_{i,j,k}^{(l)})$.
    Also, let the minimum spacing of the knots in layer $l$ of $f_{\btheta}$ be
    \begin{equation*}
        \Delta^{(l)} = \min_{k} | \xi_{k - m - 1}^{(l)} - \xi_{k - m}^{(l)}|, \quad l = 1, \ldots, L - 1
    \end{equation*}
    Then the following holds. 
    For each $l = 1, \ldots, L - 1$, there exists a constant $C_l > 0$ such that
    \begin{equation*}
        | (\Phi_{L - 1} \circ \cdots \circ \Phi_{l}) (\bx) - (\Phi_{L - 1} \circ \cdots \circ \Phi_{l} ) (\bx^\ast) | \leq C_l \cdot \| \bx - \bx^\ast \|_\infty, \quad \text{for all } \bx, \bx^\ast \in \bbR^{D_l}.
    \end{equation*}
    Here,
    \begin{equation*}
        C_l := \prod_{k=l}^{L - 1} \left( \frac{2 m}{\Delta^{(k)}}  + m \right) \| \bW^{(k)} \|_\infty D_{k} 
    \end{equation*}
\end{lemma}

\begin{proof}
    % We prove (i).
    We prove the claim by induction on the number of composed functions, $h = L - l$. First, suppose that $h = 1$. Fix $\bx, \bx^\ast \in \bbR^{D_{L - 1}}$. By the derivative properties \cref{eq:bSpline_deriv} and \cref{eq:bSpline_deriv_bound} of the B-spline curve $s_m$, together with the mean value theorem,
    \begin{align}   \label{eq:KANs_Lipschitz1}
    \begin{split}
        &| \phi_{1,j}^{(L - 1)} (x_j) - \phi_{1,j}^{(L - 1)} (x_j^\ast) |    \\
        &\leq | s_m (x_j; \btheta_{1,j,-0}^{(L - 1)} ) - s_m (x_j^\ast; \btheta_{1,j, -0}^{(L - 1)} ) | + | \theta_{1,j,0}^{(L - 1)} | | \operatorname{SiLU} (x_j) - \operatorname{SiLU} (x_j^\ast) | \\
        &\leq \left( \frac{2 \| \btheta_{1, j, -0}^{(L - 1)} \|_\infty}{\Delta^{(L - 1)}}  + c_S | \theta_{1,j,0}^{(L - 1)} | \right) \cdot |x_j - x_j^\ast|  \\
        &\leq \big( 2 m ( \Delta^{(L - 1)} )^{-1} + m \big) \cdot \| \bW^{(L - 1)} \|_\infty \cdot | x_j - x_j^\ast |  \\
    \end{split}
    \end{align}
    for $j = 1, \ldots, D_{L - 1}$. The second inequality follows from \cref{eq:bSpline_deriv_bound} and the mean value theorem applied to $\operatorname{SiLU}$, where $c_{\mathrm{S}} := \sup_{z} \abs{\operatorname{SiLU}^\prime(z)} \in (1, 1.1)$ (note that $\operatorname{SiLU}$ is not $1$-Lipschitz). The last inequality uses $2 \leq 2m$ and $c_{\mathrm{S}} < 2 \leq m$. Summing over $j = 1, \ldots, D_{L-1}$ yields
    \begin{equation}    \label{eq:KANs_Lipschitz2}
        | \Phi_{L - 1} (\bx) - \Phi_{L - 1} (\bx^\ast) | \leq  \big( 2 m ( \Delta^{(L - 1)} )^{-1} + m \big) \| \bW^{(L - 1)} \|_\infty D_{L - 1} \cdot \| \bx - \bx^\ast \|_\infty
    \end{equation}
    Now suppose that the conclusion holds for some $h \geq 1$. We show that the inequality also holds for $h + 1$. Let $\bx, \bx^\ast \in \bbR^{D_{L - h - 1}}$. By the induction hypothesis, we have
    \begin{align}
    \begin{split}   \label{eq:KANs_Lipschitz3}
        &| \Phi_{L - 1} \circ \cdots \circ \Phi_{L - h - 1} (\bx) - \Phi_{L - 1} \circ \cdots \circ \Phi_{L - h - 1} (\bx^\ast) |   \\
        &\leq \left( \prod_{k = L - h}^{L - 1} \big( 2 m ( \Delta^{(k)} )^{-1} + m \big) \cdot \| \bW^{(k)} \|_\infty \cdot D_{k} \right)    \cdot \| \Phi_{L - h - 1}(\bx) - \Phi_{L - h - 1} (\bx^\ast) \|_\infty.  \\
    \end{split}
    \end{align}
    Moreover, for $i=1, \ldots, D_{L - h}$, by the same argument as in \cref{eq:KANs_Lipschitz1}, we have
    \begin{equation*}
        | \phi_{i,j}^{(L - h - 1)} (x_j) - \phi_{i,j}^{(L - h - 1)} (x_j^\ast) | \leq \big( 2 m (\Delta^{(L - h - 1)})^{-1} + m \big) \| \bW^{(L - h - 1)} \|_\infty \cdot |x_j - x_j^\ast|
    \end{equation*}
    for $j = 1, \ldots, D_{L - h - 1}$. Applying this for all $i$, as in \cref{eq:KANs_Lipschitz2}, gives
    \begin{equation*}
        \| \Phi_{L - h - 1} (\bx) - \Phi_{L - h - 1} (\bx^\ast) \|_\infty \leq  \big( 2 m ( \Delta^{(L - h - 1)} )^{-1} + m \big) \cdot \| \bW^{(L - h - 1)} \|_\infty \cdot D_{L - h - 1} \cdot \| \bx - \bx^\ast \|_\infty.
    \end{equation*}
    Substituting this into \cref{eq:KANs_Lipschitz3} completes the proof.
\end{proof}

\begin{lemma}   \label{lem:activation_bound}
    Let $L, m \in \bbN$, $\bD \in \bbN^{L + 1}$, $\bG \in \bbN^L$, and $H > 0$.
    Let $f_{\btheta} \in KAN(L, \bD, \bG, H ; m)$, $f_{\btheta} = \Phi_{L - 1} \circ \cdots \circ \Phi_0$, and write $\btheta = (\bW^{(l)})_l = (\theta_{i,j,k}^{(l)})_{i,j,k,l}$.
    Assume that $|a_0| \vee |b_0| \geq 1$. Then the following holds:
    % \begin{enumerate}[label=(\roman*)]
    % \item 
    For each $l = 0, \ldots, L - 1$,
    \begin{equation*}
        \| \Phi_l \circ \cdots \circ \Phi_0 (\bx) \|_\infty \leq 2 \cdot (|a_0| \vee |b_0| ) \cdot \prod_{k=0}^l D_k \cdot \big( \sum_{h=0}^l \prod_{k = l - h}^l \| \bW^{(k)} \|_\infty \big)  \quad \text{for all } \bx \in [a_0, b_0]^{D_0}.
    \end{equation*}
\end{lemma}

\begin{proof}
% We prove (i).
Fix $\bx = (x_1, \ldots, x_{D_0}) \in [a_0, b_0]^{D_0}$. When $l = 0$, for $i = 1, \ldots, D_1$,
\begin{align*}
\begin{split}
    | \phi_{i,j}^{(0)} (x_j) | &\leq | s_m (x_j ; \btheta_{i,j,-0}^{(0)} , \bxi^{(0)}) | + | \theta_{i,j,0}^{(0)} \cdot \operatorname{SiLU} (x_j) | \\
    &\leq \| \bW^{(0)} \|_\infty + (|a_0| \vee |b_0| ) \cdot \| \bW^{(0)} \|_\infty    \\
    &\leq 2 \cdot (|a_0| \vee |b_0| ) \cdot \| \bW^{(0)} \|_\infty \quad (j = 1, \ldots, D_0)
\end{split}
\end{align*}
The second inequality follows from the partition of unity property \cref{eq:partitionOfUnity} of B-splines and the bound $|\operatorname{SiLU} (z) | = | z / (1 + e^{-z}) | \leq | z | \leq |a_0| \vee |b_0|$ for all $a_0 \leq z \leq b_0$. Hence,
\begin{equation*}
    \| \Phi_0 (\bx) \|_\infty \leq 2 \cdot (|a_0| \vee |b_0|  ) \cdot D_0 \cdot \| \bW^{(0)} \|_\infty \quad (i = 1, \ldots, D_1)
\end{equation*}
and the conclusion follows.

Next, suppose that the conclusion holds for some $l\geq 0$. We show that it also holds for $l + 1$. Denote the output of the previous layer $\Phi_{l}: \bbR^{D_l} \rightarrow \bbR^{D_{l + 1}}$ by $z_j := [\Phi_{l} \circ \cdots \circ \Phi_0 (\bx) ]_j$, $j = 1, \ldots, D_{l + 1}$. Then, by an argument similar to the case $l=0$ and by the induction hypothesis,
\begin{align*}
\begin{split}
    | \phi_{i,j}^{(l + 1)} (z_j) | &\leq \| \bW^{(l + 1)} \|_\infty + \| \bW^{(l + 1)} \|_\infty \cdot \sup_{\bx} | z_j |   \\
    &\leq 2 \cdot (|a_0| \vee |b_0|  ) \cdot \prod_{k=0}^{l} D_k \cdot \bigg[ \| \bW^{(l + 1)} \|_\infty + \| \bW^{(l + 1)} \|_\infty \cdot \big( \sum_{h=0}^{l} \prod_{k = l - h}^{l} \| \bW^{(k)} \|_\infty \big) \bigg],   \\
    &\leq  2 \cdot (|a_0| \vee |b_0|  ) \cdot \prod_{k=0}^{l} D_k \cdot \big( \sum_{h=0}^{l + 1} \prod_{k = l + 1 - h}^{l + 1} \| \bW^{(k)} \|_\infty \big)
\end{split}
\end{align*}
for $j = 1, \ldots, D_{l + 1}$. In the second inequality, we used $2 \cdot ( |a_0| \vee |b_0| ) \cdot \prod_{k=0}^{l} D_k \geq 1$. Finally, since $| [ \Phi_{l + 1} \circ \cdots \circ \Phi_0 (\bx) ]_i | \leq D_{l + 1} \cdot \max _j | \phi_{i,j}^{(l + 1)} (z_j) |$, substituting the above bound and rearranging completes the proof.
\end{proof}

\begin{lemma}   \label{lem:KANs_layerwise_error}
    Let $L, m \in \bbN$, $\bD \in \bbN^{L + 1}$, $\bG \in \bbN^L$, and $H > 0$. Let $f$, $f^\ast \in KAN(L, \bD, \bG, H ; m)$ be denoted by $f = \Phi_{L - 1} \circ \cdots \circ \Phi_0$ and $f^\ast = \Phi^\ast_{L - 1} \circ \cdots \circ \Phi_0^\ast$, and denote each edge by $( \phi_{i,j}^{(l)} )$ and $(\phi_{i,j}^{\ast (l)})$.
    Their respective parameters are denoted by $\boldsymbol{\theta} := ( \bW^{(l)} ) = (\theta_{i,j,k}^{(l)})$ and $\boldsymbol{\theta}^\ast := ( \bW^{\ast (l)} ) = (\theta_{i,j,k}^{\ast (l)})$.
    Further assume that $|a_0| \vee |b_0| \geq 1$.
    Then the following holds:
    
    Given $\epsilon > 0$, if $\| \boldsymbol{\theta} - \boldsymbol{\theta}^\ast \|_\infty < \epsilon $, then the following holds.
    For each $l = 1, \ldots, L - 1$, there exists a constant $C_l^\ast > 0$ such that
    \begin{equation*}
        \| \Phi_l \circ ( \Phi_{l-1}^\ast \circ \dots \Phi_0^\ast) (\bx) - \Phi_l^\ast \circ ( \Phi_{l - 1}^\ast \circ \dots \Phi_0^\ast ) (\bx) \|_\infty  \leq \epsilon \cdot C_l^\ast \quad \text{for all } \bx \in [a_0, b_0]^{D_0}
    \end{equation*}
    Here,
    \begin{equation*}
        C_l^\ast := 2 \cdot ( | a_0 | \vee | b_0 | ) \cdot \prod_{k=0}^{l} D_k \cdot \left( 1 + \sum_{h=0}^{l - 1} \prod_{k = l - 1 - h}^{l - 1} \| \bW^{\ast (k)} \|_\infty \right)
    \end{equation*}
    Moreover, if we define the case $l = 0$, then
    \begin{equation*}
        \| \Phi_0 (\bx) - \Phi_0^\ast (\bx) \|_\infty \leq \epsilon \cdot C_0^\ast \quad \text{for all } \bx \in [a_0, b_0]^{D_0}
    \end{equation*}
    where the constant is
    \begin{equation*}
        C_0^\ast := 2 \cdot (| a_0 | \vee | b_0 | ) \cdot D_0
    \end{equation*}
\end{lemma}

\begin{proof}
    Let $\epsilon > 0$ be given, and suppose that $\| \boldsymbol{\theta} - \boldsymbol{\theta^\ast} \|_\infty < \epsilon$. Also fix $\bx = (x_1, \ldots, x_{D_0}) \in [a_0, b_0]^{D_0}$.
    First, we handle the case $l = 0$ separately. By the property \cref{eq:partitionOfUnity} and the bound $| \operatorname{SiLU}(x) | \leq |x| \leq |a_0| \vee |b_0|$, $\forall x \in [a_0, b_0]$, for $i = 1, \ldots, D_1$,
    \begin{align}
    \begin{split}   \label{eq:layerwise_error1}
        | \phi_{i,j}^{(0)} (x_j) - \phi_{i,j}^{\ast (0)} (x_j)  | &\leq \| \bW^{(0)} - \bW^{\ast (0)} \|_\infty ( 1 + | x_j | )   \\
        &\leq \epsilon \cdot 2 \cdot ( | a_0 | \vee | b_0 | )
    \end{split}
    \end{align}
    for every $j = 1, \ldots, D_0$. Therefore, substituting this bound into $\| \Phi_0 (\bx) - \Phi^\ast_0 (\bx) \|_\infty \leq D_0 \max_{i,j} | \phi_{i,j}^{(0)} (x_j) - \phi_{i,j}^{\ast (0)} (x_j) |$ completes the proof for this case.

    Now let $l = 1, \ldots, L - 1$ be arbitrary, and we prove the stated conclusion. Denote the output of the previous layer $\Phi_{l - 1}^\ast : \bbR^{D_{l - 1}} \rightarrow \bbR^{D_l}$ by $\bz^\ast = (z_1^\ast, \ldots, z_{D_l}^\ast) \in \bbR^{D_l}$ s.t. $z_j^\ast := [\Phi_{l - 1}^\ast \circ \cdots \circ \Phi_0^\ast (\bx) ]_j$, $j = 1, \ldots, D_{l} $. Then, by an argument similar to \cref{eq:layerwise_error1}, for $i = 1, \ldots, D_{l + 1}$ and $j = 1, \ldots, D_{l}$,
    \begin{align*}
        | ( \phi_{i,j}^{(l)} - \phi_{i,j}^{\ast (l)} ) (z_j^\ast) | &\leq \| \bW^{(l)} - \bW^{\ast (l)} \|_\infty \cdot ( 1 + \sup_{\bx} | z_j^\ast | ) \\
        &\leq \epsilon \cdot (1 + \sup_\bx |z_j^\ast| )
    \end{align*}
    Applying Lemma~\ref{lem:activation_bound}, we have
    \begin{equation*}
        |z_j^\ast | \leq 2 (|a_0| \vee |b_0|) \cdot \prod_{k = 0}^{l - 1} D_k \cdot \left( \sum_{h=0}^{l - 1} \prod_{k = l - 1 - h}^{l - 1} \| \bW^{\ast(k)} \|_\infty \right)
    \end{equation*}
    Combining this with $1 \leq 2(|a_0| \vee |b_0|) \cdot \prod_{k=0}^{l - 1} D_k$ gives
    \begin{equation*}
        \| \Phi_l(\bz^\ast) - \Phi_l^\ast(\bz^\ast) \|_\infty \leq \epsilon \cdot 2 (|a_0| \vee |b_0|) \cdot \prod_{k = 0}^{l} D_k \cdot \left( 1 + \sum_{h=0}^{l - 1} \prod_{k = l - 1 - h}^{l - 1} \| \bW^{\ast(k)} \|_\infty \right)
    \end{equation*}
    which completes the proof.
\end{proof}

%%%%%%%%%%%%%%%%%%%%%%%%%%%%%%%%%%%%%%%%%%%%%%%%%%%%%%%%%%%%

\subsubsection{Parameter-to-function Lipschitz inequality}
\label{subsubsec:param_to_function_lipschitz}

\begin{lemma}   \label{lem:KANs_lipschitz}
Let $m \in \bbN$ with $m \geq 2$ be given, and let $L \in \bbN$, $\bD \in \bbN^{L + 1}$, $\bG \in \bbN^L$, $H > 0$, and $B > 0$. For given functions $f, f^\ast \in KAN (L, \bD, \bG, H ; m)$, write $f = \Phi_{L - 1} \circ \cdots \circ \Phi_0$ and $f^\ast = \Phi_{L - 1}^\ast \circ \cdots \circ \Phi_0^\ast$. Suppose also that their respective parameters $\boldsymbol{\theta} = (\bW^{(l)})_{l=0}^{L - 1}$ and $\boldsymbol{\theta}^\ast = (\bW^{\ast (l)})_{l=0}^{L - 1}$ satisfy
\begin{equation*}
    \| \btheta \|_\infty \leq B, \; \| \btheta^\ast \|_\infty \leq B.
\end{equation*}
Then the following holds:

For any $\epsilon > 0$, if $\| \boldsymbol{\theta} - \boldsymbol{\theta}^\ast \|_\infty < \epsilon$, then there exists a constant $K (L, \bD, \bG, H, B ; a_0, b_0, m) $ such that
    \begin{equation*}
        \| f - f^\ast \|_{L^\infty} \leq \epsilon K (L, \bD, \bG, H, B ; a_0, b_0, m)
    \end{equation*}
    Here,
    \begin{align*}
         &K (L, \bD, \bG, H, B \, ; \, a_0, b_0, m) \\
         &:= 2 (|a_0| \vee |b_0|) L m^{L - 1} \left[ \prod_{l=1}^{L - 1} \left( \frac{G_l + 2m}{H} + 1 \right) \right] \cdot \left( \prod_{l=0}^{L - 1} D_l \right) \cdot \left( \sum_{l=0}^{L - 1} B^l \right)
    \end{align*}
\end{lemma}

\begin{proof}
    For the given $\epsilon > 0$, assume that $\| \boldsymbol{\theta} - \boldsymbol{\theta}^\ast \|_\infty < \epsilon$.
    For the network representation $f=\Phi_{L-1}\circ\cdots\circ\Phi_0$, define the following partial composition notation. For each $l = 1, \ldots, L - 1$, let
    \begin{align*}
        T^+_l (f) (\cdot) &:= \Phi_{l - 1} \circ \cdots \circ \Phi_0 (\cdot), \\
        T^-_l (f) (\cdot) &:= \Phi_{L - 1} \circ \cdots \circ \Phi_l (\cdot).
    \end{align*}
    Also define $T^+_0 (f) (\bx) = \bx$ when $l = 0$, and define $T^-_L(f) (\bx) = \bx$ when $l = L$.
    Now fix $\bx \in [a_0, b_0]^{D_0}$. Then, by a telescoping argument that replaces one layer at a time and by the triangle inequality,
    \begin{equation}    \label{eq:KANs_lipschitz_lemma1}
        |f(\bx) - f^\ast(\bx)| 
        \leq 
        \sum_{l = 0}^{L - 1} 
        \left| 
        T^-_{l + 1} (f) \circ \Phi_l \circ T^+_{l} (f^\ast) (\bx) 
        - 
        T^-_{l + 1} (f) \circ \Phi^\ast_l \circ T^+_{l} (f^\ast) (\bx) 
        \right|
    \end{equation}
    Applying \cref{lem:KANs_lipschitz_bound} to the right-hand side of \cref{eq:KANs_lipschitz_lemma1}, we obtain
    \begin{equation}    \label{eq:KANs_lipschitz_lemma2}
    \begin{split}
        |f(\bx) - f^\ast(\bx)|
        \leq 
        &\sum_{l = 0}^{L - 2}
        C_{l + 1} 
        \left\| 
        \Phi_l \circ T^+_l (f^\ast) (\bx) 
        - 
        \Phi_l^\ast \circ T^+_l (f^\ast) (\bx) 
        \right\|_\infty \\
        &\quad +
        \left| 
        \Phi_{L - 1} \circ T^+_{L - 1} (f^\ast) (\bx) 
        - 
        \Phi_{L - 1}^\ast \circ T^+_{L - 1} (f^\ast) (\bx) 
        \right|,
    \end{split}
    \end{equation}
    where, for $l^\prime = 1, \ldots, L - 1$,
    \begin{equation*}
        C_{l^\prime}
        :=
        \prod_{k=l^\prime}^{L - 1} 
        \left(\frac{2m}{\Delta^{(k)}} + m \right) 
        \| \bW^{(k)} \|_\infty D_k .
    \end{equation*}

    Next, applying Lemma~\ref{lem:KANs_layerwise_error} to \cref{eq:KANs_lipschitz_lemma2}, we obtain the following upper bound:
    \begin{equation}    \label{eq:KANs_lipschitz_lemma3}
        |f(\bx)-f^\ast(\bx)|
        \leq
        \left( \sum_{l = 1}^{L - 1} C_{l} C_{l - 1}^\ast + C_{L - 1}^\ast \right) \epsilon .
    \end{equation}
    Here,
    \begin{equation*}
        C^\ast_0 
        :=
        2 (|a_0| \vee |b_0|) D_0
    \end{equation*}
    and, for $l^\prime = 1, \ldots, L - 1$,
    \begin{equation*}
        C^\ast_{l^\prime} 
        :=
        2 (|a_0| \vee |b_0|) 
        \prod_{k=0}^{l^\prime} D_k 
        \left(
        1 + \sum_{k^\prime = 0}^{l^\prime - 1} 
        \prod_{k = l^\prime - 1 - k^\prime}^{l^\prime - 1} 
        \| \bW^{\ast (k)} \|_\infty 
        \right).
    \end{equation*}

    Since
    \begin{equation*}
        \Delta^{(k)}
        =
        \frac{b_k - a_k}{G_k + 2m}
        =
        \frac{2H}{G_k + 2m},
        \qquad
        \| \bW^{(k)} \|_\infty \leq B,
        \qquad
        \| \bW^{\ast (k)} \|_\infty \leq B
    \end{equation*}
    we have $2m/\Delta^{(k)} + m = m\{(G_k + 2m)/H + 1\}$, and hence, for $l=1,\ldots,L-1$,
    \begin{align*}
        C_l
        &\leq
        \prod_{k=l}^{L-1}
        \left\{
        m\left(\frac{G_k+2m}{H}+1\right)
        \right\}
        B D_k \\
        &\leq
        m^{L-l} B^{L-l}
        \left[
        \prod_{k=l}^{L-1}
        \left(\frac{G_k+2m}{H}+1\right)
        \right]
        \left(
        \prod_{k=l}^{L-1} D_k
        \right).
    \end{align*}
    Moreover,
    \begin{align*}
        C_{l-1}^\ast
        &\leq
        2 (|a_0| \vee |b_0|)
        \left(
        \prod_{k=0}^{l-1}D_k
        \right)
        \left(
        \sum_{r=0}^{l-1}B^r
        \right).
    \end{align*}
    Therefore,
    \begin{align*}
        C_l C_{l-1}^\ast
        &\leq
        2 (|a_0| \vee |b_0|)
        m^{L-l}
        \left[
        \prod_{k=l}^{L-1}
        \left(\frac{G_k+2m}{H}+1\right)
        \right]
        \left(
        \prod_{k=0}^{L-1}D_k
        \right)
        B^{L-l}
        \left(
        \sum_{r=0}^{l-1}B^r
        \right) \\
        &\leq
        2 (|a_0| \vee |b_0|)
        m^{L-1}
        \left[
        \prod_{k=1}^{L-1}
        \left(\frac{G_k+2m}{H}+1\right)
        \right]
        \left(
        \prod_{k=0}^{L-1}D_k
        \right)
        \left(
        \sum_{r=0}^{L-1}B^r
        \right),
    \end{align*}
    where the last inequality uses $m\geq 1$ and
    \[
        B^{L-l}\sum_{r=0}^{l-1}B^r
        =
        \sum_{r=L-l}^{L-1}B^r
        \leq
        \sum_{r=0}^{L-1}B^r
    \]
    The term corresponding to the last layer can be bounded in the same way as
    \begin{equation*}
        C_{L-1}^\ast
        \leq
        2 (|a_0| \vee |b_0|)
        \left(
        \prod_{k=0}^{L-1}D_k
        \right)
        \left(
        \sum_{r=0}^{L-1}B^r
        \right)
    \end{equation*}
    and hence it is also bounded by the same common constant above. Therefore, the right-hand side of \cref{eq:KANs_lipschitz_lemma3} is bounded by
    \begin{equation*}
        \epsilon \cdot
        2 (|a_0| \vee |b_0|)
        L m^{L - 1}
        \left[
        \prod_{l=1}^{L - 1}
        \left( \frac{G_l + 2m}{H} + 1 \right)
        \right]
        \left(
        \prod_{l=0}^{L - 1} D_l
        \right)
        \left(
        \sum_{l=0}^{L - 1} B^l
        \right)
    \end{equation*}
    Since $\bx \in [a_0,b_0]^{D_0}$ was arbitrary,
    \[
        \|f-f^\ast\|_{L^\infty}
        \leq
        \epsilon K (L,\bD,\bG,H,B\,;\,a_0,b_0,m)
    \]
    This completes the proof.
\end{proof}

\begin{corollary}   \label{cor:KANs_lipschitz}
    Let $L, D, G, m \in \bbN$, $H > 0$ with $m \geq 2$, and $B \geq 1$. For $f_{\btheta}$, $f_{\btheta^\ast} \in KAN(L,D,G,H ;m)$, assume that
    \begin{equation*}
        \| \btheta \|_\infty \vee \| \btheta^\ast \|_\infty \leq B
    \end{equation*}
    For any $\epsilon > 0$, if $\| \btheta - \btheta^\ast \|_\infty < \epsilon$, then there exists $K (L, D, G, H, B ; a_0, b_0, d, m)$ such that
    \begin{equation*}
        \| f_{\btheta} - f_{\btheta^\ast} \|_{L^\infty} \leq \epsilon K (L, D, G, H, B; a_0, b_0, d, m)
    \end{equation*}
    Here,
    \begin{equation*}
        K (L, D, G, H, B ; a_0, b_0, d, m) = 2 (|a_0| \vee |b_0|) L^2 m^{L - 1} d \left( \frac{G + 2m}{H} + 1 \right)^{L - 1} D^{L - 1} B^L.
    \end{equation*}
\end{corollary}

\begin{proof}
    The result follows directly by substituting $\bD = (d, D, \ldots, D, 1)$ and $\bG = (G_0, G \ldots, G)$ into $K (L, \bD, \bG, H, B ; a_0, b_0, d, m)$ in \cref{lem:KANs_lipschitz}.
\end{proof}

%%%%%%%%%%%%%%%%%%%%%%%%%%%%%%%%%%%%%%%%%%%%%%%%%%%%%%%%%%%%

\subsubsection{Covering number bound for sparse KANs}
\label{subsubsec:covering_sparse_kans}

\begin{lemma}[Covering number bound]\label{lem:entropy}
    Let $L, D, G, m \in \bbN$ with $m \geq 2$, $H > 0$, $B \geq 1$, and $1 \leq S \leq T(L,D,G)$, where $T(L, D, G) = d D (G_0 + m  + 1) + ((L - 2)D^2 + D) (G + m + 1 )$. Then, for any $\epsilon > 0$, the following holds:
    \begin{equation}\label{eq:entropy}
    \begin{aligned}
    &\log \cN\Bigl(
    \epsilon,
    \Bigl\{
    f_{\btheta}\in KAN(L,D,G,H;m):
    \norm{\btheta}_\infty\le B,\ 
    \norm{\btheta}_0\le S
    \Bigr\},
    \norm{\cdot}_{L^\infty}
    \Bigr)
    \\
    &\lesssim \;
    S\log\frac{eT(L,D,G)}{S}+
    S\log\left(1+\frac{2 B K (L, D, G, H, B ; a_0, b_0, d, m)}{\epsilon}\right)
    \end{aligned}
    \end{equation}
    where
    \begin{equation*}
        K (L, D, G, H, B ; a_0, b_0, d, m) = 2 (|a_0| \vee |b_0|) L^2 m^{L - 1} d \left( \frac{G + 2m}{H} + 1 \right)^{L - 1} D^{L - 1} B^L
    \end{equation*}
\end{lemma}

\begin{proof}
    Write $T : = T(L, D, G)$, and define the parameter set as
    \[
    \Theta(B,S):=\{\btheta\in\mathbb{R}^{T}:\norm{\btheta}_\infty\le B,\ \norm{\btheta}_0\le S\}.
    \]
    Write $\delta := \epsilon / K (L, D, G, H, B ; a_0, b_0, d, m)$. By \cref{cor:KANs_lipschitz}, every $\delta$-cover of $\Theta(B, S)$ in $\norm{\cdot}_\infty$ induces an $\epsilon$-cover of the KAN class in \cref{eq:entropy}. Therefore, the left-hand side of \cref{eq:entropy} is bounded above by
    \[
    \log \cN\left( \delta ,\Theta(B,S),\norm{\cdot}_\infty\right).
    \]
    We now decompose $\Theta(B,S)$ according to the support set $\text{supp}(\btheta) = \{ i : \theta_i \neq 0 \}$. For $I \subset \{ 1, \ldots, T \}$, define
    \[
    \Theta_I(B):=\{\btheta\in\mathbb{R}^{T}:\mathrm{supp}(\btheta)\subset I,\ \norm{\btheta}_\infty\le B\}.
    \]
    Then
    \[
    \Theta(B,S)
    =
    \bigcup_{s=0}^{S}\ \bigcup_{\substack{I\subset\{1,\dots,T\}\\ \abs{I}=s}}\Theta_I(B).
    \]
    For a fixed $I$ s.t. $|I| = s$, covering the set $\Theta_I(B)$ under the $\ell_\infty$ metric is equivalent to covering $[-B, B]^s$ under the $\ell_\infty$ metric. Hence,
    \[
    \cN\bigl(\delta,\Theta_I(B),\norm{\cdot}_\infty\bigr)
    \le
    \left\lceil \frac{2B}{\delta} \right\rceil^s.
    \]
    Moreover, since there are $\binom{T}{s}$ possible support sets of size $s = |I|$, we have
    \begin{align*}
    \cN \left( \delta ,\Theta(B,S),\norm{\cdot}_\infty \right)
    &\le
    \sum_{s=1}^{S}\binom{T}{s} \left\lceil \frac{2B}{\delta} \right\rceil^s \\
    &\le
    \left(1+\frac{2B}{\delta}\right)^S\sum_{s=0}^{S}\binom{T}{s}    \\
    &\le
    1 + \left(1+\frac{2B}{\delta}\right)^S \left( \frac{e T}{S} \right)^S.
    \end{align*}
    In the last inequality, we used $\sum_{s=1}^{S} \binom{T}{s} \leq \left( \frac{e T}{S} \right)^S$. Finally, substituting back $\delta = \epsilon / K (L, D, G, H, B ; a_0, b_0, d, m)$ and taking logarithms completes the proof.
\end{proof}

%%%%%%%%%%%%%%%%%%%%%%%%%%%%%%%%%%%%%%%%%%%%%%%%%%%%%%%%%%%%

\section{Proofs on posterior concentration}

In this section, we present the proofs of the main theorems corresponding to \cref{sec:main_results}.
\Cref{app:sufficient} is a standard lemma used in \Cref{app:main-proof,app:main-ber-proof,app:main-adaptive-proof,app:composite-proof}. It is a reformulation of Theorem 4 of \cite{ghosal2007convergence} for independent observations in the framework of a nonparametric regression model with Gaussian errors.

\subsection{Lemmas for posterior concentration}\label{app:sufficient}
\begin{lemma}[Sufficient conditions for posterior concentration]\label{lem:sufficient}
Let $\cF_n\subset\UB$ be a sieve, and define
\[
\cM_n:=\cF_n\times[\underline{\sigma},\overline{\sigma}].
\]
Also let $\epsilon_n\to0$ and $n\epsilon_n^2\to\infty$. Define
\[
A_n
:=
\left\{
(f,\sigma):
\norm{f-f_0}_{L^2(P_X)}\le \epsilon_n/2,\ 
\abs{\sigma-\sigma_0}\le \epsilon_n/2
\right\}.
\]
Suppose that, for a sufficiently large constant $C>0$, the following three conditions hold:
\begin{align}
\log \cN\bigl(\epsilon_n,\cF_n,\norm{\cdot}_{L^2(P_X)}\bigr)
&\lesssim n\epsilon_n^2,\label{eq:suff-1}\\
-\log \Pi_n(A_n)
&\lesssim n\epsilon_n^2,\label{eq:suff-2}\\
\Pi_n(\cM_n^c)
&=
o\bigl(e^{-Cn\epsilon_n^2}\bigr).\label{eq:suff-3}
\end{align}
Then
\[
\Pi_n\!\left(
(f,\sigma):
\norm{f-f_0}_{L^2(P_X)}+\abs{\sigma^2-\sigma_0^2}
>
M_n\epsilon_n
\;\middle|\;
\cD_n
\right)
\to 0
\]
holds in $P_{f_0,\sigma_0}^{(n)}$-probability for any sequence $M_n\to\infty$.
\end{lemma}

\begin{proof}
Write the density of a single observation $(X,Y)$ as
\[
p_{f,\sigma}(x,y)
:=
p_X(x)\frac{1}{\sqrt{2\pi}\sigma}
\exp\!\left[-\frac{(y-f(x))^2}{2\sigma^2}\right].
\]
On the bounded variance interval $[\underline{\sigma},\overline{\sigma}]$, the Hellinger distance, the Kullback--Leibler divergence, and the second moment of the log-likelihood ratio for the Gaussian regression model are all controlled by
\[
\rho\bigl((f,\sigma),(g,\tau)\bigr)
:=
\norm{f-g}_{L^2(P_X)}+\abs{\sigma-\tau}.
\]
More precisely, by the calculations in \citet{lee2025posterior}, there exist constants $c_1,c_2,c_3>0$ such that, for all $(f,\sigma),(g,\tau)\in\cM_n$,
\begin{align}
h^2\!\left(p_{f,\sigma},p_{g,\tau}\right)
&\le
c_1\Bigl(\norm{f-g}_{L^2(P_X)}^2+\abs{\sigma-\tau}^2\Bigr),\label{eq:app-h-bound}\\
K\!\left(p_{f_0,\sigma_0},p_{f,\sigma}\right)
&\le
c_2\Bigl(\norm{f-f_0}_{L^2(P_X)}^2+\abs{\sigma-\sigma_0}^2\Bigr),\label{eq:app-kl-bound}\\
V\!\left(p_{f_0,\sigma_0},p_{f,\sigma}\right)
&\le
c_3\Bigl(\norm{f-f_0}_{L^2(P_X)}^2+\abs{\sigma-\sigma_0}^2\Bigr)\label{eq:app-v-bound}
\end{align}
hold. Here, $K$ denotes the Kullback--Leibler divergence, and $V$ denotes the second moment of the log-likelihood ratio.

By \cref{eq:app-h-bound}, combining an $L^2(P_X)$-cover of $\cF_n$ and an absolute-value cover of the interval $[\underline{\sigma},\overline{\sigma}]$ yields a Hellinger cover of $\cM_n$. That is, for some constant $c>0$,
\[
\log \cN\bigl(c\eta,\cM_n,h\bigr)
\le
\log \cN\bigl(\eta,\cF_n,\norm{\cdot}_{L^2(P_X)}\bigr)
+
\log \cN\bigl(\eta,[\underline{\sigma},\overline{\sigma}],\abs{\cdot}\bigr).
\]
The last term satisfies
\[
\log \cN\bigl(\eta,[\underline{\sigma},\overline{\sigma}],\abs{\cdot}\bigr)
\le
\log\left(1+\frac{\overline{\sigma}-\underline{\sigma}}{\eta}\right)
\lesssim
\log(\eta^{-1}),
\]
and since $n\epsilon_n^2\to\infty$, this is of smaller order than \cref{eq:suff-1} when $\eta=\epsilon_n$. Hence,
\[
\log \cN\bigl(c\epsilon_n,\cM_n,h\bigr)\lesssim n\epsilon_n^2.
\]

Secondly, by \cref{eq:app-kl-bound} and \cref{eq:app-v-bound}, on $A_n$ we have
\[
K\!\left(p_{f_0,\sigma_0},p_{f,\sigma}\right)
\vee
V\!\left(p_{f_0,\sigma_0},p_{f,\sigma}\right)
\lesssim
\epsilon_n^2.
\]
Therefore, $A_n$ is contained in a Kullback-Leibler neighborhood. This proves that \cref{eq:suff-2} is a sufficient condition.

Condition \cref{eq:suff-3} is precisely the standard prior probability condition outside the sieve.

Now applying the theorem in \citet{ghosal2017fundamentals} to $\cM_n$, the standard posterior contraction theorem directly implies from \cref{eq:suff-1}--\cref{eq:suff-3} that
\[
\Pi_n\!\left(
\rho\bigl((f,\sigma),(f_0,\sigma_0)\bigr)>M_n\epsilon_n
\;\middle|\;
\cD_n
\right)\to 0
\]
in $P_{f_0,\sigma_0}^{(n)}$-probability. Finally, since $\sigma,\sigma_0\in[\underline{\sigma},\overline{\sigma}]$,
\[
\abs{\sigma^2-\sigma_0^2}
=
\abs{\sigma-\sigma_0}\abs{\sigma+\sigma_0}
\le
2\overline{\sigma}\abs{\sigma-\sigma_0}.
\]
Applying Lemma B.1 of \citet{xie2020adaptive}, the desired conclusion
\[
\Pi_n\!\left(
(f,\sigma):
\norm{f-f_0}_{L^2(P_X)}+\abs{\sigma^2-\sigma_0^2}>M_n\epsilon_n
\;\middle|\;
\cD_n
\right)\to 0
\]
follows.
\end{proof}

%%%%%%%%%%%%%%%%%%%%%%%%%%%%%%%%%%%%%%%%%%%%%%%%%%%%%%%%%%%%

\subsection{Proof of \cref{thm:main}} \label{app:main-proof}

\begin{proof}
We define the sieve for applying \cref{lem:sufficient} as follows.
\begin{align}    \label{eq:sas_sieve}
\begin{split}
    \cF_n^{\mathrm{spi}} &:= \left\{ f_{\btheta} \in KAN(L_0, D_n, G_n, H_n ; m) : \| \btheta \|_\infty \leq \overline{B}_n, \, \| \btheta \|_0 \leq S_n \right\}    \\
    \widetilde{\cF}_n^{\mathrm{spi}} &:= \left\{ \operatorname{clip} \circ f_{\btheta} : f_{\btheta} \in \cF_n^{\mathrm{spi}} \right\}
\end{split}
\end{align}
and set $\widetilde{\cM}^{\mathrm{spi}}_n = \widetilde{\cF}_n^{\mathrm{spi}} \times [\underline{\sigma}, \overline{\sigma}]$. Here, we take the truncation sequence $\overline{B}_n$, which is used only for the sieve, as
\begin{equation}    \label{eq:sas_sieve_trunc}
    \overline{B}_n := \overline{B}_0 N_n^{\beta + 1 / \alpha} ( \log n )^{1 / \alpha}
\end{equation}
where $\overline{B}_0 > 0$ is a sufficiently large constant chosen so that the calculation of the sieve-complement probability closes. We now show that the three conditions \cref{eq:suff-1}--\cref{eq:suff-3} in \cref{lem:sufficient} are satisfied for the rate $\epsilon_n$ in \cref{eq:rate}. Note that \cref{lem:sufficient} is applied with $\cF_n = \widetilde{\cF}_n^{\mathrm{spi}}$ and $\cM_n = \widetilde{\cM}_n^{\mathrm{spi}}$.

First, we show \eqref{eq:suff-1}. By the density condition in \eqref{assump:A2},
\[
\norm{g}_{L^2(P_X)}
\le
\norm{p_X}_{L^\infty}^{1/2}\norm{g}_{L^2}
\le
C_X^{1/2} \norm{g}_{L^\infty}
\]
and since $\operatorname{clip}$ is 1-Lipschitz, 
an $L^\infty$-cover of $\mathcal F_n^{\mathrm{spi}}$ induces an
$L^\infty$-cover, hence an $L^2(P_X)$-cover,
of $\widetilde{\mathcal F}_n^{\mathrm{spi}}$. Therefore,
\[
\log \cN\bigl(\epsilon_n, \widetilde{\cF}_n^{\mathrm{spi}},\norm{\cdot}_{L^2(P_X)}\bigr)
\le
\log \cN\bigl(C_X^{-1/2}\epsilon_n,\cF_n^{\mathrm{spi}},\norm{\cdot}_{L^\infty}\bigr).
\]
Applying \cref{lem:entropy} to the sieve \eqref{eq:sas_sieve}, we obtain
\begin{equation}    \label{eq:main-proof1}
\log \cN\bigl(C_X^{-1/2}\epsilon_n,\cF_n^{\mathrm{spi}},\norm{\cdot}_{L^\infty}\bigr)
\lesssim
S_n\log\frac{eT_n}{S_n}
+
S_n\log\left(1+\frac{\overline{B}_n K_n}{C_X^{-1/2}\epsilon_n}\right),
\end{equation}
where $T_n := T(L_0, D_n, G_n)$ and $K_n := K (L_0, D_n, G_n, H_n, \overline{B}_n; m)$. Substituting the values of the relevant sequences gives
\begin{align} \label{eq:main-proof2}
    \begin{split}
        T_n &\lesssim N_n^{2 + \beta}   \\
        K_n &= 2 (|a_0| \vee |b_0|) L_0^2 m^{L_0 - 1} d \left( \frac{G_n + 2m}{H_n} + 1\right)^{L_0 - 1} D_n^{L_0 - 1} \overline{B}_n^{L_0} \\
            &\lesssim N_n^{L_0 (\beta + 1) - 1 + L_0/\alpha} (\log n)^{L_0 / \alpha}
    \end{split}
\end{align}
Moreover, using $S_n = \lceil S_0 N_n \rceil$ and $N_n = \lceil n^{1 / (2\tbfs + 1)} \rceil$ in \cref{eq:main-proof1}, we see that
$\log\frac{eT_n}{S_n}    \lesssim \log n$ and
$\log\left(1+\frac{\overline{B}_n K_n}{\epsilon_n}\right) \lesssim \log n$.
Therefore,
\begin{equation}    \label{eq:main-proof3}
    S_n\log\frac{eT_n}{S_n}
    +
    S_n\log\left(1+\frac{\overline{B}_n K_n}{\epsilon_n}\right)
    \lesssim
    N_n \log n \asymp n \epsilon_n^2
\end{equation}
holds.

Secondly, we prove \eqref{eq:suff-2}. Applying \cref{lem:approx} with $N = N_n$, there exists $f^\star_n \in KAN_c (L_0, D_n, G_n, H_n; m)$ with some parameter $\btheta^\star_n$ satisfying $\| \btheta^\star_n \|_\infty \leq B^\star_n$ and $\|\btheta^\star_n \|_0 \leq S_n$ such that
\begin{equation}    \label{eq:main-proof4}
    \| f_0 - f^\star_n \|_{L^2} \lesssim N_n^{-\tbfs}
\end{equation}
holds. Here, $B^\star_n := B_0 N_n^\beta$ for some $B_0 > 0$. We then define the oracle model space by
\begin{equation}    \label{eq:main-proof5}
    \cF^\star_n := \left\{ f_{\btheta} \in KAN(L_0, D_n, G_n, H_n ; m) : \| \btheta \|_\infty \leq B^\star_n, \| \btheta \|_0 \leq S_n\right\}.
\end{equation}
Since
\begin{equation}    \label{eq:main-proof6}
    \frac{\overline{B}_n}{B^\star_n} = \frac{\overline{B}_0}{B_0} N_n^{1/\alpha} ( \log n)^{1/\alpha} \rightarrow \infty,
\end{equation}
we have $B^\star_n \leq \overline{B}_n$ for all sufficiently large $n$. Thus, $f_n^\star \in \cF^\star_n \subset \cF^{\mathrm{spi}}_n$ for sufficiently large $n$. In addition, $N_n^{-\tbfs} \leq n^{- \tbfs / (2 \tbfs + 1)} = \epsilon_n (\log n)^{-1/2}$ for sufficiently large $n$. Hence, for sufficiently large $n$,
\begin{equation}    \label{eq:main-proof7}
    \norm{f_0-f_n^\star}_{L^2}
    \le
    \frac{\epsilon_n}{4C_X^{1/2}}
\end{equation}
holds, and by applying \eqref{assump:A2}, we obtain
\begin{equation}    \label{eq:main-proof8}
\norm{f_0-f_n^\star}_{L^2(P_X)}
\le
\norm{p_X}_{L^\infty}^{1/2}
\norm{f_0-f_n^\star}_{L^2}
\le
\epsilon_n/4.
\end{equation}
Since $f_0 \in \UB$, we have $\operatorname{clip} \circ f_0 = f_0$. Moreover, the clipping map $\operatorname{clip}$ is 1-Lipschitz. Therefore, for any $f \in \cF^\star_n$,
\begin{align}   \label{eq:main-proof9}
\begin{split}
    \| \operatorname{clip} \circ f - f_0 \|_{L^2(P_X)}
    &=
    \| \operatorname{clip} \circ f - \operatorname{clip} \circ f_0 \|_{L^2(P_X)}  \\
    &\leq
    \| \operatorname{clip} \circ f - \operatorname{clip} \circ f^\star_n \|_{L^2(P_X)}
    +
    \| \operatorname{clip} \circ f^\star_n - \operatorname{clip} \circ f_0 \|_{L^2(P_X)}  \\
    &\leq
    \| f - f^\star_n \|_{L^2(P_X)}
    +
    \| f^\star_n - f_0 \|_{L^2(P_X)}  \\
    &\leq
    C_X^{1/2}\| f - f^\star_n \|_{L^2}
    +
    \epsilon_n/4 .
\end{split}
\end{align}
Combining the preceding bounds, for the prior mass $\Pi_n(A_n)$ in \cref{eq:suff-2}, we have
\begin{align}  
\begin{split}   \label{eq:main-proof10}
    \Pi_n(A_n) 
    &\geq 
    \Pi_n \left(
    f \in \cF^{\mathrm{spi}}_n :
    \| f - f^\star_n \|_{L^2} \leq \frac{\epsilon_n}{4 C_X^{1/2}}
    \right) \,
    \Pi \left(\sigma: | \sigma - \sigma_0| \leq \frac{\epsilon_n}{2} \right) \\
    &\geq 
    \Pi_n \left(
    f \in \cF^\star_n :
    \| f - f^\star_n \|_{L^\infty} \leq \frac{\epsilon_n}{4 C_X^{1/2}}
    \right) \,
    \Pi \left(\sigma: | \sigma - \sigma_0| \leq \frac{\epsilon_n}{2} \right).
\end{split}
\end{align}
In the second inequality, we used the fact that the $L^\infty$ norm dominates the $L^2$ norm.
Now, for the parameter vector $\boldsymbol{\theta}_n^\star$ realizing $f_n^\star$, let
\begin{equation}    \label{eq:main-proof11}
I_n^\star:=\{t:\theta_{n,t}^\star\neq 0\},
\qquad
S_n^\star:=|I_n^\star|.
\end{equation}
Then $S_n^\star \leq S_n$.
Since the spike-and-slab prior selects exactly $S_n$ nonzero coordinates, choose any set
$J_n^\star\subset\{1,\ldots,T_n\}$ such that $I_n^\star\subset J_n^\star$ and $|J_n^\star|=S_n$ and denote
a vector $\bone_{J^\star_n} \in \{ 0, 1 \}^{T_n}$ as
\begin{equation}    \label{eq:main-proof-bone}
    (\bone_{J^\star_n})_t = \begin{cases}
        1,      &t \in J^\star_n    \\
        0,      &t \notin J^\star_n.
    \end{cases}
\end{equation}

That is, we enlarge the parameter support of the actual approximator to an $S_n$-sized support. We now apply \cref{cor:KANs_lipschitz} to the first term in the last inequality of \eqref{eq:main-proof10}. For $f_{\btheta}, f^\star_n \in \cF^\star_n$, if $\| \btheta - \btheta^\star_n \|_\infty \leq \delta_n$, then
\begin{equation}    \label{eq:main-proof12}
    \|f_{\boldsymbol{\theta}}-f_n^\star\|_{L^\infty}
    \le
    K_n^\star \|\boldsymbol{\theta}-\boldsymbol{\theta}_n^\star\|_\infty
    \le
    \frac{\epsilon_n}{4C_X^{1/2}}.
\end{equation}
where $K_n^\star := K(L_0, D_n, G_n, H_n, B^\star_n; a_0, b_0, d, m)$. Recall that $K_n^\star \asymp D_n^{L_0 - 1} (B_n^\star)^{L_0} \asymp N_n^{L_0 - 1} ( B_n^\star)^{L_0} = N_n^{(1 + \beta)L_0 - 1}$.
Define
\begin{equation}    \label{eq:main-proof13}
    \delta_n :=
    \frac{\epsilon_n}{4C_X^{1/2}K_n^\star}.
\end{equation}
Then the lower bound in \cref{eq:main-proof10} is
\begin{equation}    \label{eq:main-proof14}
\Pi_n\Bigl( \btheta:
\|\boldsymbol{\theta}-\boldsymbol{\theta}_n^\star\|_\infty\le \delta_n
\mid
\bgamma =\bone_{J_n^\star}
\Bigr)
\pi_n(\bgamma=\bone_{J_n^\star})
\Pi (\sigma : |\sigma-\sigma_0|\le \epsilon_n/2).
\end{equation}
First, since $\pi_n (\bgamma = \bone_{J_n^\star} ) = \binom{T_n}{S_n}^{-1}$,
\begin{equation*}
    - \log \pi_n (\bgamma = \bone_{J^\star_n}) 
    \le
    S_n\log\frac{eT_n}{S_n}
    \lesssim
    N_n \log n
\end{equation*}
holds. Next, since $\| \btheta^\star_n \|_\infty \leq B^\star_n$ and $\delta_n \leq 1$ for all sufficiently large $n$, combining this with \eqref{assump:B1} yields
\begin{align*}
\Pi_n\Bigl( \btheta :
\|\boldsymbol{\theta}-\boldsymbol{\theta}_n^\star\|_\infty\le \delta_n
\mid
\boldsymbol{\gamma}=\bone_{J_n^\star}
\Bigr)
&\ge
\prod_{t\in J_n^\star}
\int_{\theta_{n,t}^\star-\delta_n}^{\theta_{n,t}^\star+\delta_n}
\widetilde{\pi}_{SL}(u)\,du\\
&\ge
\left(
2\delta_n
\inf_{|u|\le B_n^\star+1}\widetilde{\pi}_{SL}(u)
\right)^{S_n}   \\
&\ge
(2\delta_n n^{-c_1})^{S_n}.
\end{align*}
Therefore, it remains to bound the right-hand side of
\begin{equation}    \label{eq:main-proof15}
-\log \Pi_n\Bigl( \btheta :
\|\boldsymbol{\theta}-\boldsymbol{\theta}_n^\star\|_\infty\le \delta_n
\mid
\boldsymbol{\gamma}=\bone_{J_n^\star}
\Bigr)
\leq S_n \log \delta_n^{-1} + c_1 S_n \log n
\end{equation}
by $n \epsilon_n^2 = n^{1 / (2\tbfs + 1)} \log n$.
The second term is immediate, so we treat the first term.
\[
\delta_n^{-1} 
\asymp K^\star_n \epsilon_n^{-1} 
= N_n^{(1 + \beta) L_0 - 1 + \tbfs} ( \log n )^{- 1/ 2}
\]
Thus, $\log \delta_n^{-1} \lesssim \log n$, which completes this part.
Finally, for the variance term in \cref{eq:main-proof14}, by \eqref{assump:B3} and $\sigma_0 \geq \underline{\sigma} > 0$, for all sufficiently large $n$,
\begin{equation}    \label{eq:main-proof16}
\Pi_n(\abs{\sigma-\sigma_0}\le \epsilon_n/2)
\ge
c_\sigma\Bigl[\bigl(\sigma_0+\epsilon_n/2\bigr)^2-\bigl(\sigma_0-\epsilon_n/2\bigr)^2\Bigr]
=
2c_\sigma\sigma_0\epsilon_n,
\end{equation}
and hence
\begin{equation}    \label{eq:main-proof17}
-\log \Pi_n(\abs{\sigma-\sigma_0}\le \epsilon_n/2)
\lesssim
\log(\epsilon_n^{-1})
\lesssim
\log n.
\end{equation}
This verifies \eqref{eq:suff-2}.

Finally, we prove \eqref{eq:suff-3}. By \eqref{assump:B3}, the prior distribution on the variance is supported on
$[\underline{\sigma}^2,\overline{\sigma}^2]$. Since the sufficient lemma is applied
with $\cF_n=\widetilde{\cF}_n^{\mathrm{spi}}$, where
\[
\widetilde{\cF}_n^{\mathrm{spi}}
:=
\left\{
\operatorname{clip}\circ f_{\btheta}: f_{\btheta}\in\cF_n^{\mathrm{spi}}
\right\},
\]
we have
\begin{equation}    \label{eq:main-proof18}
\Pi_n\bigl((\widetilde{\cM}_n^{\mathrm{spi}})^c\bigr)
=
\Pi_n\bigl((\widetilde{\cF}_n^{\mathrm{spi}})^c\bigr)
\le
\Pi_n\bigl((\cF_n^{\mathrm{spi}})^c\bigr).
\end{equation}
Moreover, since the sparsity index in \eqref{eq:sas_prior} always has exactly $S_n$ active coordinates,
\[
\Pi_n\bigl(\norm{\btheta}_0>S_n\bigr)=0.
\]
Therefore,
\[
\Pi_n\bigl((\cF_n^{\mathrm{spi}})^c\bigr)
=
\Pi_n\bigl(\norm{\btheta}_\infty>\overline{B}_n\bigr).
\]
Conditionally on $\bgamma$, the number of nonzero coordinates is $S_n$, and all zero coordinates are equal to $0$. Thus, by the union bound and \eqref{assump:B2},
\begin{align}   \label{eq:main-proof19}
\Pi_n\bigl(\norm{\btheta}_\infty>\overline{B}_n\bigr)
&=
\sum_{\substack{\bgamma\in\{0,1\}^{T_n}\\ \norm{\bgamma}_0=S_n}}
\pi_n(\bgamma)
\Pi_n\bigl(\norm{\btheta}_\infty>\overline{B}_n\mid \bgamma\bigr) \notag \\
&\le
\sum_{\substack{\bgamma\in\{0,1\}^{T_n}\\ \norm{\bgamma}_0=S_n}}
\pi_n(\bgamma)
\sum_{t:\gamma_t=1}
\Pi_n\bigl(|\theta_t|>\overline{B}_n\mid \gamma_t=1\bigr) \notag \\
&\le
S_n
\int_{\abs{u}>\overline{B}_n}
\widetilde{\pi}_{SL}(u)\,du \notag \\
&\le
c_2 S_n
\exp\left[
-c_3\left(\frac{\overline{B}_n}{\tau_n}\right)^\alpha
\right].
\end{align}
By \eqref{assump:B2}, we have $\tau_n\lesssim B_n^\star$, and by the definition of the truncation sequence,
\[
B_n^\star=B_0N_n^\beta,
\qquad
\overline{B}_n
=
\overline{B}_0N_n^{\beta+1/\alpha}(\log n)^{1/\alpha}.
\]
Therefore, there exists a constant $C_\tau>0$ such that
\[
\left(\frac{\overline{B}_n}{\tau_n}\right)^\alpha
\ge
C_\tau \overline{B}_0^\alpha N_n\log n.
\]
Substituting this into \cref{eq:main-proof18} and \cref{eq:main-proof19}, we obtain
\[
\Pi_n\bigl((\widetilde{\cM}_n^{\mathrm{spi}})^c\bigr)
\le
c_2 S_n
\exp\left[
-c_3 C_\tau \overline{B}_0^\alpha N_n\log n
\right].
\]
On the other hand, since $S_n=\lceil S_0N_n\rceil$, we have $\log S_n\lesssim \log n$. Hence,
\[
\log \Pi_n\bigl((\widetilde{\cM}_n^{\mathrm{spi}})^c\bigr)
\le
\log c_2+\log S_n
-
c_3 C_\tau \overline{B}_0^\alpha N_n\log n .
\]
Moreover,
\[
n\epsilon_n^2
=
n^{1/(2\tbfs+1)}\log n
\asymp
N_n\log n.
\]
Therefore, by choosing $\overline{B}_0>0$ sufficiently large, for the sufficiently large constant $C>0$ appearing in \cref{lem:sufficient}, we have
\[
\Pi_n\bigl((\widetilde{\cM}_n^{\mathrm{spi}})^c\bigr)
=
o\left(e^{-Cn\epsilon_n^2}\right).
\]
This proves \cref{eq:suff-3}.

\end{proof}

%%%%%%%%%%%%%%%%%%%%%%%%%%%%%%%%%%%%%%%%%%%%%%%%%%%%%%%%%%%%
 
\subsection{Proof of \cref{thm:main-ber}}   \label{app:main-ber-proof}
\begin{proof}
We define a sieve as in the proof of \Cref{app:main-proof}.
The proof follows almost the same structure. Let the truncation sequence $\overline{B}_n$ be defined as in \cref{eq:sas_sieve_trunc}, and let the sieve be as in \cref{eq:sas_sieve}, except that the sparsity level $S_n$ is adjusted.
\begin{align*}
    \cF_n^{\mathrm{ber}} &:= \left\{ f_{\btheta} \in KAN(L_0, D_n, G_n, H_n ; m) : \| \btheta \|_\infty \leq \overline{B}_n, \, \| \btheta \|_0 \leq \overline{S}_n \right\} \\
    \widetilde{\cF}_n^{\mathrm{ber}} &:= \left\{ \operatorname{clip} \circ f_{\btheta} : f_{\btheta} \in \cF_n^{\mathrm{ber}}  \right\}   \\
    \widetilde\cM_n^{\mathrm{ber}} &:= \widetilde\cF_n^{\mathrm{ber}} \times [\underline{\sigma}, \overline{\sigma}]
\end{align*}
where $\overline{S}_n := \overline{S}_0 S_n$ for some sufficiently large $\overline{S}_0 > 0$.

We first show that \eqref{eq:suff-1} holds. As in the proof of \Cref{app:main-proof}, by \eqref{assump:A2} and \Cref{lem:entropy}, we have
\begin{equation}
\log \cN\bigl(C_X^{-1/2}\epsilon_n,\cF_n^{\mathrm{ber}},\norm{\cdot}_{L^\infty}\bigr)
\lesssim
\overline{S}_n\log\frac{eT_n}{\overline{S}_n}
+
\overline{S}_n\log\left(1+\frac{\overline{B}_n K_n}{C_X^{-1/2}\epsilon_n}\right).
\end{equation}
The remaining argument is identical to the steps in \cref{eq:main-proof2}--\cref{eq:main-proof3}, and the right-hand side can be controlled by $n\epsilon_n^2$.

Next, we prove \eqref{eq:suff-2}. Following the same steps as in \cref{eq:main-proof4}--\cref{eq:main-proof13} in the proof of the prior mass condition in \Cref{app:main-proof}, we start from \cref{eq:main-proof14}.
\begin{align}
\begin{split}
    &\Pi_n (A_n)    \\
    &\ge
    \Pi_n\Bigl( \btheta:
    \|\boldsymbol{\theta}-\boldsymbol{\theta}_n^\star\|_\infty\le \delta_n
    \mid
    \bgamma =\bone_{J_n^\star}
    \Bigr)
    \pi_n(\bgamma=\bone_{J_n^\star})
    \Pi (\sigma : |\sigma-\sigma_0|\le \epsilon_n/2).
\end{split}
\end{align}
The first probability on the right-hand side is handled in the same way as in \cref{eq:main-proof15}. Next, for the second probability,
\begin{equation*}
    \pi_n (\bgamma = \bone_{J^\star_n}) = \rho_n^{S_n} (1 - \rho_n)^{T_n - S_n}
\end{equation*}
and its negative logarithm is
\begin{equation*}
    S_n \log \frac{1}{\rho_n} + (T_n - S_n) \log \frac{1}{1 - \rho_n}.
\end{equation*}
By the first condition in \cref{eq:ber_sas_rho}, there exists $c_{\rho, 1} > 0$ such that $\rho_n \geq c_{\rho, 1} T_n^{-1}$ for all sufficiently large $n$. Therefore,
\begin{equation*}
S_n\log\frac{1}{\rho_n}
\le
S_n\log\frac{T_n}{c_{\rho, 1}}
\lesssim
S_n\log n
\lesssim n \epsilon_n^2.
\end{equation*}
Next, the second condition in \cref{eq:ber_sas_rho} implies that $T_n \rho_n \le S_n n^{-c_{\rho, 2}}$ for some $c_{\rho, 2} > 0$. Since $\rho_n \rightarrow 0$, for all sufficiently large $n$,
\begin{equation}    \label{eq:main-ber-proof1}
(T_n-S_n)\log\frac{1}{1-\rho_n}
\le
2T_n\rho_n
\lesssim
S_n.
\end{equation}
Thus, the second probability is also controlled. The third probability, corresponding to the variance parameter, is handled in the same way as in \cref{eq:main-proof16}--\cref{eq:main-proof17}.

Now we show that \eqref{eq:suff-3} holds. As in \cref{eq:main-proof18}, it suffices to show that
\begin{equation*}
\Pi_n\bigl((\cF_n^{\mathrm{ber}})^c\bigr)
=
o (e^{-C n \epsilon_n^2} ).
\end{equation*}
Under the Bernoulli spike-and-slab prior in \cref{eq:ber_sas_prior}, there are two ways to leave the sieve: either the number of nonzero coordinates exceeds $\overline{S}_n$, or some nonzero coordinate exceeds $\overline{B}_n$. Thus,
\begin{equation}
\Pi_n\bigl((\cF_n^{\mathrm{ber}})^c\bigr) \leq \Pi_n(\norm{\bgamma}_0> \overline{S}_n)
+
\Pi_n\Bigl( \| \btheta \|_\infty > \overline{B}_n \Bigr).
\end{equation}
Write $\| \bgamma \|_0 \sim \text{Bin} (T_n, \rho_n)$, and denote its mean by $\mu_n = T_n \rho_n$. From \cref{eq:main-ber-proof1}, we have $\mu_n \le S_n/e$ for all sufficiently large $n$. Therefore, by the Chernoff bound for the binomial distribution,
\begin{equation*}
    \Pi_n (\| \bgamma \|_0 > \overline{S}_n ) \leq \left( \frac{e \mu_n}{\overline{S}_n} \right)^{\overline{S}_n}.
\end{equation*}
Taking negative logarithms and using the second condition in \cref{eq:ber_sas_rho}, we obtain
\begin{equation*}
    - \log \Pi_n (\| \bgamma \|_0 > \overline{S}_n ) \geq \overline{S}_n \left( \log \frac{\overline{S}_n}{\mu_n} - 1 \right) \gtrsim \overline{S}_n \log n.
\end{equation*}
Since $\overline S_n=\overline S_0S_n$ and $S_n\log n\asymp n\epsilon_n^2$, choosing $\overline S_0>0$ sufficiently large makes the preceding negative logarithm sufficiently larger than $C n\epsilon_n^2$. Therefore,
\[
\Pi_n(\|\bgamma\|_0>\overline S_n)
=
o(e^{-C n\epsilon_n^2})
\]
holds.

Next, the second probability is handled as follows:
\begin{equation*}
    \Pi_n\Bigl( \| \btheta \|_\infty > \overline{B}_n \Bigr) = \Pi_n\Bigl( \max_{1 \leq t \leq T_n} \gamma_t |\theta_t| > \overline{B}_n \Bigr) \leq \sum_{t=1}^{T_n} \Pi_n\Bigl( \gamma_t = 1, \, |\theta_t| > \overline{B}_n \Bigr).
\end{equation*}
The probability in the last sum is bounded, exactly as in \cref{eq:main-proof19}, by
\begin{align*}
    \sum_{t=1}^{T_n} \Pi_n\Bigl( \gamma_t = 1, \, |\theta_t| > \overline{B}_n \Bigr)
    &\le T_n \rho_n \int_{| u | > \overline{B}_n} \tilde\pi_{SL} (u) du
    \\
    &\le \mu_n c_2 \exp\!\left[-c_3\left(\frac{\overline{B}_n}{\tau_n}\right)^\alpha\right].
\end{align*}
Similarly, taking negative logarithms gives
\begin{equation*}
    \left( \frac{\overline{B}_n}{\tau_n} \right)^{\alpha} \gtrsim N_n \log n \asymp n \epsilon_n^2.
\end{equation*}
Therefore, by choosing $\overline{B}_0 > 0$ sufficiently large, this term is of order $o(e^{-C n \epsilon_n^2})$.
\end{proof}

\subsection{Proof of \cref{thm:main-adaptive}}
\label{app:main-adaptive-proof}

\begin{proof}
Fix the unknown smoothness vector $\bfs = (s_1, \ldots, s_d) \in \mathfrak{S}$,
and write $\tbfs$, $\underline{\bfs}$, and $\omega = (1/p - 1/2)_+$ as defined
in \cref{sec:main_results}. For brevity, throughout this proof let
\[
    N_n := \left\lceil n^{1/(2\tbfs+1)} \right\rceil,
    \qquad
    \epsilon_n := \epsilon_n(\bfs) = n^{-\tbfs/(2\tbfs+1)}(\log n)^{1/2},
\]
so that $N_n \log n \asymp n\epsilon_n^2$.

Next, we define the cutoff
\begin{equation*}
    \overline{N}_n := \left\lceil \overline{N}_0 N_n \right\rceil,
\end{equation*}
where the constant $\overline{N}_0 > 0$ will be chosen sufficiently large below.
For each $1 \leq N \leq \overline{N}_n$, let
\begin{equation*}
    \cF_n^{\mathrm{ad}}(N)
    :=
    \left\{
    f_{\btheta} \in KAN(L_0, D(N), G(N), H(N); m)
    :
    \|\btheta\|_\infty \leq \overline{B}_n(N),\;
    \|\btheta\|_0 \leq S(N)
    \right\},
\end{equation*}
where the truncation sequence is
\[
    \overline{B}_n(N)
    :=
    \overline{B}_0\, N^{\beta_{\mathrm{ad}}} (N_n \log n)^{1/\alpha},
\]
and $S(N)$ is the sparsity level used in the adaptive prior. The constant
$\overline{B}_0 > 0$ will be chosen sufficiently large when controlling the
sieve-complement probability.

Define the clipped sieve for each $N$ by
\[
    \widetilde{\cF}_n^{\mathrm{ad}}(N)
    :=
    \left\{
    \operatorname{clip}\circ f_{\btheta}
    :
    f_{\btheta}\in \cF_n^{\mathrm{ad}}(N)
    \right\}.
\]
Finally, define the full adaptive sieves by
\[
    \cF_n^{\mathrm{ad}}
    :=
    \bigcup_{N=1}^{\overline{N}_n} \cF_n^{\mathrm{ad}}(N),
    \qquad
    \widetilde{\cF}_n^{\mathrm{ad}}
    :=
    \bigcup_{N=1}^{\overline{N}_n} \widetilde{\cF}_n^{\mathrm{ad}}(N),
\]
and set
\[
    \widetilde{\cM}_n^{\mathrm{ad}}
    :=
    \widetilde{\cF}_n^{\mathrm{ad}}
    \times
    [\underline{\sigma},\overline{\sigma}].
\]
We verify conditions \cref{eq:suff-1}--\cref{eq:suff-3} of
\cref{lem:sufficient} for the rate $\epsilon_n$. Note that
\cref{lem:sufficient} is applied with
$\cF_n=\widetilde{\cF}_n^{\mathrm{ad}}$ and
$\cM_n=\widetilde{\cM}_n^{\mathrm{ad}}$.

We first verify the entropy condition \eqref{eq:suff-1}. By \eqref{assump:A2},
$\|g\|_{L^2(P_X)} \leq C_X^{1/2} \|g\|_{L^\infty}$.
Moreover, since $\operatorname{clip}$ is 1-Lipschitz, an
$L^\infty$-cover of $\cF_n^{\mathrm{ad}}(N)$ induces an
$L^\infty$-cover, hence an $L^2(P_X)$-cover, of
$\widetilde{\cF}_n^{\mathrm{ad}}(N)$. Therefore,
\begin{align}
\log \cN\!\left(
    \epsilon_n,\, \widetilde{\cF}_n^{\mathrm{ad}},\, \|\cdot\|_{L^2(P_X)}
\right)
&\lesssim
\log \sum_{N=1}^{\overline{N}_n}
\cN\!\left(
    C_X^{-1/2}\epsilon_n,\,
    \cF_n^{\mathrm{ad}}(N),\,
    \|\cdot\|_{L^\infty}
\right).    \label{eq:ad-entropy-union}
\end{align}
Applying \cref{lem:entropy} to each $\cF_n^{\mathrm{ad}}(N)$ gives
\begin{align}
\begin{split}   \label{eq:ad-entropy-one}
&\log \cN\!\left(
    C_X^{-1/2}\epsilon_n,\,
    \cF_n^{\mathrm{ad}}(N),\,
    \|\cdot\|_{L^\infty}
\right) \\
&\quad \lesssim
S(N) \log \frac{eT(N)}{S(N)}    \\
&\qquad+
S(N) \log\!\left(
    1 + \frac{\overline{B}_n(N)\,
    K(L_0, D(N), G(N), H(N), \overline{B}_n(N); a_0, b_0, d, m)}
    {C_X^{-1/2}\epsilon_n}
\right).    
\end{split}
\end{align}
Since $D(N) \asymp N$, $G(N)/H(N) \asymp 1$, $H(N) \asymp N^{\beta_{\mathrm{ad}}}$,
and $L_0$ is fixed, we have $T(N) \lesssim N^{2+\beta_{\mathrm{ad}}}$.
Moreover, uniformly over $1 \leq N \leq \overline{N}_n$,
\begin{align*}
    \log T(N) &\lesssim \log n,  \\
    \log \overline{B}_n(N) &\lesssim \log n,    \\
    \log K\bigl(L_0, D(N), G(N), H(N), \overline{B}_n(N); a_0, b_0, d, m\bigr)
    &\lesssim \log n.
\end{align*}
Hence the right-hand side of \cref{eq:ad-entropy-one} is bounded by
$C \cdot S(N) \log n \lesssim N \log n$.
Substituting into \cref{eq:ad-entropy-union},
\[
    \log \cN\!\left(
    \epsilon_n,\, \widetilde{\cF}_n^{\mathrm{ad}},\, \|\cdot\|_{L^2(P_X)}
    \right)
    \lesssim
    \log \overline{N}_n + \overline{N}_n \log n
    \lesssim N_n \log n
    \asymp n\epsilon_n^2,
\]
where the last equivalence follows from the definitions of $N_n$ and
$\epsilon_n$. This verifies \cref{eq:suff-1}.

Next, we establish the prior-mass condition \eqref{eq:suff-2}. Applying \cref{lem:approx} with $N = N_n$ and then \cref{rem:approx_monotone}, which applies here because $\beta < \beta_{ad}$ implies $B_0 N_n^{\beta} \le B^\star(N_n)$ for all sufficiently large $n$, there exists
\[
    f_{n}^\star
    \in
    KAN_c\bigl(L_0,\, D(N_n),\, G(N_n),\, H(N_n);\, m\bigr)
\]
with realizing parameter vector $\btheta_{n}^\star$ satisfying
\[
    \|\btheta_{n}^\star\|_0 \leq S(N_n),
    \qquad
    \|\btheta_{n}^\star\|_\infty \leq B^\star(N_n),
\]
and
\[
    \|f_0 - f_{n}^\star\|_{L^2}
    \lesssim N_n^{-\tbfs}
    \lesssim \epsilon_n(\log n)^{-1/2}.
\]
Therefore, for all sufficiently large $n$, by \eqref{assump:A2},
\begin{equation*} 
    \|f_0 - f_{n}^\star\|_{L^2(P_X)} \leq \epsilon_n/4.
\end{equation*}
Moreover, since $B^\star(N_n) \leq \overline{B}_n(N_n)$ for all sufficiently
large $n$, we have $f_{n}^\star \in \cF_n^{\mathrm{ad}}(N_n)$.

Let $I_{n}^\star := \{t : \theta_{n,t}^\star \neq 0\}$.
Choose $J_{n}^\star \subset \{1, \ldots, T(N_n)\}$ such that
$I_{n}^\star \subset J_{n}^\star$ and $|J_{n}^\star| = S(N_n)$,
and set $\bone_{J_{n}^\star} \in \{0, 1 \}^{T(N_n)}$ as in \cref{eq:main-proof-bone}.
Define
\[
    K_{n}^\star
    :=
    K\bigl(L_0, D(N_n), G(N_n), H(N_n), B^\star(N_n); a_0, b_0, d, m\bigr),
    \qquad
    \delta_n
    :=
    \frac{\epsilon_n}{4C_X^{1/2} K_{n}^\star}.
\]
Since $\log K_{n}^\star \lesssim \log n$, we have $\delta_n \leq 1$ and
$\log \delta_n^{-1} \lesssim \log n$ for all sufficiently large $n$.
By \cref{cor:KANs_lipschitz}, if $\|\btheta - \btheta_{n}^\star\|_\infty \leq \delta_n$,
then $\|f_{\btheta} - f_{n}^\star\|_{L^\infty} \leq \epsilon_n / (4C_X^{1/2})$,
and hence $\|f_{\btheta} - f_{n}^\star\|_{L^2(P_X)} \leq \epsilon_n/4$.
Since $f_0\in\UB$, we have $\operatorname{clip}\circ f_0=f_0$.
Moreover, $\operatorname{clip}$ is 1-Lipschitz. Therefore,
\[
\begin{split}
    \|\operatorname{clip}\circ f_{\btheta} - f_0\|_{L^2(P_X)}
    &=
    \|\operatorname{clip}\circ f_{\btheta}
    - \operatorname{clip}\circ f_0\|_{L^2(P_X)}  \\
    &\leq
    \|\operatorname{clip}\circ f_{\btheta}
    - \operatorname{clip}\circ f_n^\star\|_{L^2(P_X)}
    +
    \|\operatorname{clip}\circ f_n^\star
    - \operatorname{clip}\circ f_0\|_{L^2(P_X)}  \\
    &\leq
    \|f_{\btheta} - f_n^\star\|_{L^2(P_X)}
    +
    \|f_n^\star - f_0\|_{L^2(P_X)}
    \leq
    \epsilon_n/2.
\end{split}
\]
Let $A_n$ denote the KL neighborhood in \cref{lem:sufficient}, where
\cref{lem:sufficient} is applied with
$\cF_n=\widetilde{\cF}_n^{\mathrm{ad}}$. Then
\begin{align}
\Pi_n(A_n)
&\geq
\pi_N(N_n)
\cdot
\Pi\!\left(\bgamma = \bone_{J_n^\star} \mid N = N_n\right)
\nonumber\\
&\quad\times
\Pi\!\left(
    \|\btheta - \btheta_{n}^\star\|_\infty \leq \delta_n
    \mid \bgamma = \bone_{J_n^\star},\, N = N_n
\right)
\nonumber\\
&\quad\times
\Pi\!\left(|\sigma - \sigma_0| \leq \epsilon_n/2\right).
\label{eq:ad-prior-mass-decomp}
\end{align}
We bound each factor in turn. Since $\pi_N(N) \propto \exp(-\lambda_N N \log N)$,
\begin{equation}    \label{eq:ad-pm1}
    -\log \pi_N(N_n) \lesssim N_n \log N_n \lesssim N_n \log n.
\end{equation}

The combinatorial prior assigns
$\Pi(\bgamma = \bone_{J_n^\star} \mid N = N_n) = \binom{T(N_n)}{S(N_n)}^{-1}$,
so
\begin{equation}    \label{eq:ad-pm2}
    -\log \Pi\!\left(\bgamma = \bone_{J_n^\star} \mid N = N_n\right)
    \leq
    S(N_n) \log \frac{eT(N_n)}{S(N_n)}
    \lesssim N_n \log n.
\end{equation}

By \eqref{assump:C1}, using $\|\btheta_{n}^\star\|_\infty \leq B^\star(N_n)$
and $\delta_n \leq 1$,
\begin{align*}
&\Pi\!\left(
    \|\btheta - \btheta_{n}^\star\|_\infty \leq \delta_n
    \mid \bgamma = \bone_{J_n^\star},\, N = N_n
\right)\\
&\quad\geq
\left[
    2\delta_n \inf_{|u| \leq B^\star(N_n)+1}
    \widetilde{\pi}_{SL}(u \mid N_n)
\right]^{S(N_n)}.
\end{align*}
Therefore,
\begin{equation}    \label{eq:ad-pm3}
-\log
\Pi\!\left( \btheta:
    \|\btheta - \btheta_{n}^\star\|_\infty \leq \delta_n
    \mid \bgamma = \bone_{J_n^\star},\, N = N_n
\right)
\lesssim
S(N_n)\{\log \delta_n^{-1} + \log n\}
\lesssim N_n \log n.
\end{equation}

By \eqref{assump:C3}, the same calculation as in
\cref{eq:main-proof16}--\cref{eq:main-proof17} gives
\begin{equation}    \label{eq:ad-pm4}
    -\log \Pi\!\left(|\sigma - \sigma_0| \leq \epsilon_n/2\right)
    \lesssim \log(\epsilon_n^{-1}) \lesssim \log n.
\end{equation}
Combining \cref{eq:ad-pm1}--\cref{eq:ad-pm4} in \cref{eq:ad-prior-mass-decomp},
\[
    -\log \Pi_n(A_n)
    \lesssim N_n \log n
    \asymp n\epsilon_n^2,
\]
which verifies \cref{eq:suff-2}.

Finally, we verify the sieve-complement condition \eqref{eq:suff-3}. By \eqref{assump:C3}, the variance prior is supported on
$[\underline{\sigma}^2, \overline{\sigma}^2]$. Since \cref{lem:sufficient} is
applied with $\cF_n=\widetilde{\cF}_n^{\mathrm{ad}}$ and
$\cM_n=\widetilde{\cM}_n^{\mathrm{ad}}$, we have
\[
\begin{split}
    \Pi_n^{\mathrm{ad}}\bigl((\widetilde{\cM}_n^{\mathrm{ad}})^c\bigr)
    &=
    \Pi_n^{\mathrm{ad}}\bigl((\widetilde{\cF}_n^{\mathrm{ad}})^c\bigr) \\
    &\leq
    \pi_N(N > \overline{N}_n)
    +
    \sum_{N=1}^{\overline{N}_n}
    \pi_N(N)\,
    \Pi\!\left(\|\btheta\|_\infty > \overline{B}_n(N) \mid N\right).
\end{split}
\]
Indeed, for $N\leq \overline N_n$, the event
$\|\btheta\|_\infty\leq \overline B_n(N)$ implies
$f_{\btheta}\in \cF_n^{\mathrm{ad}}(N)$, and hence
$\operatorname{clip}\circ f_{\btheta}\in
\widetilde{\cF}_n^{\mathrm{ad}}(N)\subset
\widetilde{\cF}_n^{\mathrm{ad}}$.
For all sufficiently large $n$,
\[
    \pi_N(N > \overline{N}_n)
    \lesssim
    \sum_{N > \overline{N}_n} \exp(-\lambda_N N \log N)
    \lesssim
    \exp\!\left\{
        -\frac{\lambda_N}{2} \overline{N}_n \log \overline{N}_n
    \right\}.
\]
Since $\overline{N}_n \log \overline{N}_n \asymp \overline{N}_0 N_n \log n$,
choosing $\overline{N}_0$ sufficiently large gives
\begin{equation}    \label{eq:ad-sc1}
    \pi_N(N > \overline{N}_n) = o\!\left(e^{-Cn\epsilon_n^2}\right)
\end{equation}
for the constant $C > 0$ in \cref{lem:sufficient}.

Since conditionally on $N$ the spike-and-slab prior selects exactly $S(N)$
nonzero coordinates, by the union bound and \eqref{assump:C2},
\[
    \Pi\!\left(\|\btheta\|_\infty > \overline{B}_n(N) \mid N\right)
    \leq
    S(N) \int_{|u| > \overline{B}_n(N)}
    \widetilde{\pi}_{SL}(u \mid N)\, du
    \leq
    c_2 S(N)
    \exp\!\left[
        -c_3 \left(\frac{\overline{B}_n(N)}{\tau_N}\right)^\alpha
    \right].
\]
Since $\tau_N \lesssim N^{\beta_{\mathrm{ad}}}$, the definition of
$\overline{B}_n(N)$ implies, uniformly over $1 \leq N \leq \overline{N}_n$,
\[
    \left(\frac{\overline{B}_n(N)}{\tau_N}\right)^\alpha
    \gtrsim \overline{B}_0^\alpha N_n \log n.
\]
Therefore,
\[
    \sum_{N=1}^{\overline{N}_n}
    \pi_N(N)\,
    \Pi\!\left( \btheta: \|\btheta\|_\infty > \overline{B}_n(N) \mid N\right)
    \leq
    c_2 S(\overline{N}_n)
    \exp\!\left[-c_4 \overline{B}_0^\alpha N_n \log n\right].
\]
Since $S(\overline{N}_n) \lesssim N_n$ and $N_n \log n \asymp n\epsilon_n^2$,
choosing $\overline{B}_0 > 0$ sufficiently large gives
\begin{equation}    \label{eq:ad-sc2}
    \sum_{N=1}^{\overline{N}_n}
    \pi_N(N)\,
    \Pi\!\left(\|\btheta\|_\infty > \overline{B}_n(N) \mid N\right)
    = o\!\left(e^{-Cn\epsilon_n^2}\right).
\end{equation}
Combining \cref{eq:ad-sc1} and \cref{eq:ad-sc2} verifies \cref{eq:suff-3}.
Applying \cref{lem:sufficient} completes the proof.
\end{proof}

\subsection{Proof of \cref{thm:composite}}
\label{app:composite-proof}

\begin{proof}
Let $N_n^* := \lceil n^{1/(2\tilde{s}^*+1)} \rceil$ and
$\epsilon_n^* := n^{-\tilde{s}^*/(2\tilde{s}^*+1)}(\log n)^{1/2}$,
so that $N_n^* \log n \asymp n(\epsilon_n^*)^2$.
Introduce the auxiliary truncation sequence
\[
    \overline{B}_n^{\mathrm{cp}}
    :=
    \overline{B}_{\mathrm{cp}}
    (N_n^*)^{\beta_{\mathrm{cp}} + 1/\alpha}
    (\log n)^{1/\alpha},
\]
where $\overline{B}_{\mathrm{cp}} > 0$ is chosen sufficiently large so that the
sieve-complement calculation closes.
Define the raw sieve
\[
    \cF_n^{\mathrm{cp}}
    :=
    \left\{
        f_{\btheta} \in KAN(L_{\mathrm{cp}}, D_n^{\mathrm{cp}},
        G_n^{\mathrm{cp}}, H_n^{\mathrm{cp}}; m)
        :
        \|\btheta\|_\infty \leq \overline{B}_n^{\mathrm{cp}},\;
        \|\btheta\|_0 \leq S_n^{\mathrm{cp}}
    \right\}
\]
and the corresponding clipped sieve
\[
    \widetilde{\cF}_n^{\mathrm{cp}}
    :=
    \left\{
        \operatorname{clip}\circ f_{\btheta}
        :
        f_{\btheta}\in \cF_n^{\mathrm{cp}}
    \right\}.
\]
Set
\[
    \widetilde{\cM}_n^{\mathrm{cp}}
    :=
    \widetilde{\cF}_n^{\mathrm{cp}}
    \times
    [\underline{\sigma}, \overline{\sigma}].
\]
We verify conditions \cref{eq:suff-1}--\cref{eq:suff-3} of
\cref{lem:sufficient} for the rate $\epsilon_n^*$. Note that
\cref{lem:sufficient} is applied with
$\cF_n=\widetilde{\cF}_n^{\mathrm{cp}}$ and
$\cM_n=\widetilde{\cM}_n^{\mathrm{cp}}$.

Now we prove that \eqref{eq:suff-1} holds.
Let
\[
    T_n^{\mathrm{cp}} := T(L_{\mathrm{cp}}, D_n^{\mathrm{cp}}, G_n^{\mathrm{cp}})
\]
and
\[
    K_n^{\mathrm{cp}}
    :=
    K(L_{\mathrm{cp}}, D_n^{\mathrm{cp}}, G_n^{\mathrm{cp}},
    H_n^{\mathrm{cp}}, \overline{B}_n^{\mathrm{cp}}; a_0, b_0, d, m).
\]
By Assumption \eqref{assump:A2},
$\|g\|_{L^2(P_X)}\leq C_X^{1/2}\|g\|_{L^\infty}$.
Moreover, since $\operatorname{clip}$ is 1-Lipschitz, an
$L^\infty$-cover of $\cF_n^{\mathrm{cp}}$ induces an $L^\infty$-cover,
hence an $L^2(P_X)$-cover, of $\widetilde{\cF}_n^{\mathrm{cp}}$.
Therefore, by \cref{lem:entropy},
\[
\begin{split}
    \log \cN\!\bigl(\epsilon_n^*,\, \widetilde{\cF}_n^{\mathrm{cp}},\,
    \|\cdot\|_{L^2(P_X)}\bigr)
    &\leq
    \log \cN\!\bigl(C_X^{-1/2}\epsilon_n^*,\, \cF_n^{\mathrm{cp}},\,
    \|\cdot\|_{L^\infty}\bigr) \\
    &\lesssim
    S_n^{\mathrm{cp}} \log \frac{e T_n^{\mathrm{cp}}}{S_n^{\mathrm{cp}}}
    +
    S_n^{\mathrm{cp}}
    \log\!\left(1 + \frac{\overline{B}_n^{\mathrm{cp}} K_n^{\mathrm{cp}}}
    {C_X^{-1/2} \epsilon_n^*}\right).
\end{split}
\]
Since $L_{\mathrm{cp}}$ is a fixed constant, we have
$T_n^{\mathrm{cp}} \lesssim (N_n^*)^{2+\beta_{\mathrm{cp}}}$ and
$\log K_n^{\mathrm{cp}} \lesssim \log n$, so the right-hand side is bounded by
\[
    S_n^{\mathrm{cp}} \log n
    \asymp N_n^* \log n
    \asymp n(\epsilon_n^*)^2,
\]
verifying \cref{eq:suff-1}.

Next, we show \eqref{eq:suff-2}.
Applying \cref{lem:comp_approx} with $N = N_n^*$, there exists
$f_n^{\mathrm{cp},*} \in KAN(L_{\mathrm{cp}}, D_n^{\mathrm{cp}},
G_n^{\mathrm{cp}}, H_n^{\mathrm{cp}}; m)$
with realizing parameter vector $\btheta_n^{\mathrm{cp},*}$ satisfying
$\|\btheta_n^{\mathrm{cp},*}\|_\infty \leq B_n^{\mathrm{cp},*}$ and
$\|\btheta_n^{\mathrm{cp},*}\|_0 \leq S_n^{\mathrm{cp}}$, and such that
\[
    \|f_0 - f_n^{\mathrm{cp},*}\|_{L^2(P_X)}
    \lesssim (N_n^*)^{-\tilde{s}^*}
    \lesssim \epsilon_n^*(\log n)^{-1/2}
    \leq \frac{\epsilon_n^*}{4}
\]
for all sufficiently large $n$. Since $B_n^{\mathrm{cp},*} \leq
\overline{B}_n^{\mathrm{cp}}$ for all sufficiently large $n$, we have
$f_n^{\mathrm{cp},*} \in \cF_n^{\mathrm{cp}}$.

Let $I_n^{\mathrm{cp},*} := \{t : \theta_{n,t}^{\mathrm{cp},*} \neq 0\}$
and choose $J_n^{\mathrm{cp},*} \subset \{1, \ldots, T_n^{\mathrm{cp}}\}$ with
$I_n^{\mathrm{cp},*} \subset J_n^{\mathrm{cp},*}$ and
$|J_n^{\mathrm{cp},*}| = S_n^{\mathrm{cp}}$.
Define
\[
    K_n^{\mathrm{cp},*}
    :=
    K\!\bigl(L_{\mathrm{cp}}, D_n^{\mathrm{cp}}, G_n^{\mathrm{cp}},
    H_n^{\mathrm{cp}}, B_n^{\mathrm{cp},*}; a_0, b_0, d, m\bigr),
    \qquad
    \delta_n^{\mathrm{cp}}
    :=
    \frac{\epsilon_n^*}{4 C_X^{1/2} K_n^{\mathrm{cp},*}}.
\]
Since $\log K_n^{\mathrm{cp},*}\lesssim \log n$, we have
$\delta_n^{\mathrm{cp}}\leq 1$ and
$\log(\delta_n^{\mathrm{cp}})^{-1}\lesssim \log n$ for all sufficiently large
$n$. By \cref{cor:KANs_lipschitz},
$\|\btheta - \btheta_n^{\mathrm{cp},*}\|_\infty \leq \delta_n^{\mathrm{cp}}$
implies
\[
    \|f_{\btheta} - f_n^{\mathrm{cp},*}\|_{L^\infty}
    \leq
    \frac{\epsilon_n^*}{4C_X^{1/2}},
\]
and hence
\[
    \|f_{\btheta} - f_n^{\mathrm{cp},*}\|_{L^2(P_X)}
    \leq
    \frac{\epsilon_n^*}{4}.
\]

Since $f_0\in\UB$, we have $\operatorname{clip}\circ f_0=f_0$.
Moreover, $\operatorname{clip}$ is 1-Lipschitz. Therefore,
\[
\begin{split}
    \|\operatorname{clip}\circ f_{\btheta} - f_0\|_{L^2(P_X)}
    &=
    \|\operatorname{clip}\circ f_{\btheta}
    - \operatorname{clip}\circ f_0\|_{L^2(P_X)} \\
    &\leq
    \|\operatorname{clip}\circ f_{\btheta}
    - \operatorname{clip}\circ f_n^{\mathrm{cp},*}\|_{L^2(P_X)}
    +
    \|\operatorname{clip}\circ f_n^{\mathrm{cp},*}
    - \operatorname{clip}\circ f_0\|_{L^2(P_X)} \\
    &\leq
    \|f_{\btheta} - f_n^{\mathrm{cp},*}\|_{L^2(P_X)}
    +
    \|f_n^{\mathrm{cp},*} - f_0\|_{L^2(P_X)} \\
    &\leq
    \frac{\epsilon_n^*}{2}.
\end{split}
\]
Let $A_n$ denote the KL neighborhood in \cref{lem:sufficient}, where
\cref{lem:sufficient} is applied with
$\cF_n=\widetilde{\cF}_n^{\mathrm{cp}}$. Using the same prior mass
decomposition as in the proof of \cref{thm:main}
(cf.\ \cref{eq:main-proof10}--\cref{eq:main-proof17}),
\begin{align*}
    \Pi_n(A_n)
    &\geq
    \Pi_n\!\left(
        \btheta : \|\btheta - \btheta_n^{\mathrm{cp},*}\|_\infty
        \leq \delta_n^{\mathrm{cp}}
        \;\middle|\;
        \bgamma = \bone_{J_n^{\mathrm{cp},*}}
    \right)
    \pi_n\!\left(\bgamma = \bone_{J_n^{\mathrm{cp},*}}\right)
    \Pi_n\!\left(\sigma : |\sigma - \sigma_0| \leq \epsilon_n^*/2\right).
\end{align*}
Since $\pi_n(\bgamma = \bone_{J_n^{\mathrm{cp},*}})
=
\binom{T_n^{\mathrm{cp}}}{S_n^{\mathrm{cp}}}^{-1}$, we have
\[
    -\log \pi_n(\bgamma = \bone_{J_n^{\mathrm{cp},*}})
    \leq
    S_n^{\mathrm{cp}}\log\frac{eT_n^{\mathrm{cp}}}{S_n^{\mathrm{cp}}}
    \lesssim
    N_n^*\log n.
\]
Moreover, by \eqref{assump:B1}, using
$\|\btheta_n^{\mathrm{cp},*}\|_\infty\leq B_n^{\mathrm{cp},*}$ and
$\delta_n^{\mathrm{cp}}\leq 1$,
\[
\begin{split}
&\Pi_n\!\left(
        \btheta : \|\btheta - \btheta_n^{\mathrm{cp},*}\|_\infty
        \leq \delta_n^{\mathrm{cp}}
        \;\middle|\;
        \bgamma = \bone_{J_n^{\mathrm{cp},*}}
    \right) \\
&\quad\geq
\left(
    2\delta_n^{\mathrm{cp}}
    \inf_{|u|\leq B_n^{\mathrm{cp},*}+1}
    \widetilde{\pi}_{SL}(u)
\right)^{S_n^{\mathrm{cp}}},
\end{split}
\]
and hence
\[
-\log
\Pi_n\!\left(
        \btheta : \|\btheta - \btheta_n^{\mathrm{cp},*}\|_\infty
        \leq \delta_n^{\mathrm{cp}}
        \;\middle|\;
        \bgamma = \bone_{J_n^{\mathrm{cp},*}}
    \right)
\lesssim
S_n^{\mathrm{cp}}\{\log(\delta_n^{\mathrm{cp}})^{-1}+\log n\}
\lesssim
N_n^*\log n.
\]
Finally, by \eqref{assump:B3}, the same calculation as in
\cref{eq:main-proof16}--\cref{eq:main-proof17} gives
\[
    -\log
    \Pi_n\!\left(\sigma : |\sigma - \sigma_0| \leq \epsilon_n^*/2\right)
    \lesssim
    \log((\epsilon_n^*)^{-1})
    \lesssim
    \log n.
\]
Combining these bounds yields
\[
    -\log \Pi_n(A_n)
    \lesssim
    N_n^* \log n
    \asymp
    n(\epsilon_n^*)^2,
\]
which verifies \cref{eq:suff-2}.

Finally, we show that \eqref{eq:suff-3} holds.
By \eqref{assump:B3}, the variance prior is supported on
$[\underline{\sigma}^2, \overline{\sigma}^2]$. Since \cref{lem:sufficient} is
applied with $\cF_n=\widetilde{\cF}_n^{\mathrm{cp}}$ and
$\cM_n=\widetilde{\cM}_n^{\mathrm{cp}}$, we have
\[
    \Pi_n\bigl((\widetilde{\cM}_n^{\mathrm{cp}})^c\bigr)
    =
    \Pi_n\bigl((\widetilde{\cF}_n^{\mathrm{cp}})^c\bigr)
    \leq
    \Pi_n\bigl((\cF_n^{\mathrm{cp}})^c\bigr).
\]
Since the spike-and-slab prior \cref{eq:sas_prior} fixes exactly
$S_n^{\mathrm{cp}}$ active coordinates, the only way to leave $\cF_n^{\mathrm{cp}}$
is for some active slab coefficient to exceed $\overline{B}_n^{\mathrm{cp}}$.
By the union bound and \eqref{assump:B2},
\[
    \Pi_n\!\bigl((\cF_n^{\mathrm{cp}})^c\bigr)
    \leq
    S_n^{\mathrm{cp}} \cdot c_2
    \exp\!\left[
        -c_3 \left( \frac{\overline{B}_n^{\mathrm{cp}}}{\tau_n} \right)^\alpha
    \right].
\]
Using $\tau_n \lesssim (N_n^*)^{\beta_{\mathrm{cp}}}$ and the definition of
$\overline{B}_n^{\mathrm{cp}}$,
\[
    \left(\frac{\overline{B}_n^{\mathrm{cp}}}{\tau_n}\right)^\alpha
    \gtrsim N_n^* \log n
    \asymp n(\epsilon_n^*)^2.
\]
Choosing $\overline{B}_{\mathrm{cp}}$ sufficiently large gives
\[
    \Pi_n\bigl((\widetilde{\cM}_n^{\mathrm{cp}})^c\bigr)
    =
    o(e^{-Cn(\epsilon_n^*)^2}),
\]
verifying \cref{eq:suff-3}.
Applying \cref{lem:sufficient} completes the proof.
\end{proof}

\clearpage
\section{Numerical study}  \label{app:simulation}

This section reports a numerical study of sparse Bayesian KANs (BKANs) on the target functions in \Cref{fig:targets}.

\subsection{Setup}  \label{app:sim-setup}

\paragraph{Data.}
For each target $f_0$, we generate $X_i \sim U([0,1]^d)$ and $Y_i = f_0(X_i) + \varepsilon_i$ with $\varepsilon_i \sim N(0, 0.1^2)$, where $d = 2$ for $f_1, \dots, f_5$ and $d = 10$ for $f_6$.
Every target is affinely normalized so that its range lies in $[-1, 1]$:
\begin{align*}
    f_1(\bx) &= \tfrac{2}{3}\bigl\{ \mathbf{1}(x_1 \ge 1/2) + \sin(2\pi x_2) - 1/2 \bigr\},
    &
    f_2(\bx) &= \tfrac{8}{3}\bigl\{ |x_1 - 1/2| + (x_2 - 1/2)^2 \bigr\} - 1,
    \\
    f_3(\bx) &= f_1(R\bx), \quad f_4(\bx) = f_2(R\bx),
    &
    f_5(\bx) &= f_3\bigl( (f_3(\bx) + 1)/2, \, (f_4(\bx) + 1)/2 \bigr),
\end{align*}
where $R$ is the rotation by $45^\circ$ about $(1/2, 1/2)$, so that the jump of $f_3$ lies on the line $x_1 + x_2 = 1$, and
$f_6(\bx) = \{ 10 \sin(\pi x_1 x_2) + 20 (x_3 - 1/2)^2 + 10 x_4 + 5 x_5 - 15 \}/15$ is the first function of \citet{friedman1991multivariate}, for which the noise standard deviation is $1/15$ (that is, $1$ on the original scale).
For each target and replication, we draw a training pool of $4{,}000$ observations and use its first $n \in \{250, 500, 1{,}000, 2{,}000, 4{,}000\}$ observations as the training set, together with a validation set of size $100$ and two test sets of size $1{,}000$: one with the noiseless values $f_0(\tilde X_i)$, used for the estimation error, and one with noisy responses $\tilde Y_i = f_0(\tilde X_i) + \tilde\varepsilon_i$, used for the predictive intervals.
Every configuration is replicated five times.

\paragraph{Models.}
The BKAN uses a compact instance of the architecture \eqref{eq:sas_arch} with $d = 2$: depth $L_0 = 5$, width $D_n = 4$ with $N_n = 1$, cubic splines ($m = 3$), and hidden-layer knots on $[-H_n, H_n]$ with $H_n = 2$ and $G_n = 4H_n - 2m = 2$, so that the knot spacing is $1/2$; the first layer uses three knot intervals on $[0, 1]$.
This gives $368$ parameters, counting the spline coefficients and the SiLU weight of every edge.
For $f_6$ we keep the same depth and width, that is, the layer widths $(10, 4, 4, 4, 4, 1)$ with $592$ parameters.
As a continuous relaxation of the spike-and-slab prior \eqref{eq:sas_prior}, each coefficient receives the Gaussian mixture prior $(1 - \pi) N(0, 0.05^2) + \pi N(0, 1)$ with $\pi = 0.1$, and the noise precision $\sigma^{-2}$ receives a $\mathrm{Gamma}(1, 0.02)$ prior.
For comparison, we consider
(i) a Bayesian ReLU network (BNN) with four hidden layers of width $10$ (width $12$ for $f_6$), whose number of parameters matches that of the BKAN, with the same Gaussian mixture prior on every weight and bias and the same noise prior;
(ii) BART \citep{chipman2010bart} with $50$ trees and four chains of $1{,}000$ burn-in and $1{,}000$ posterior draws;
(iii) a Gaussian process (GP) regression with an isotropic RBF kernel of fixed length scale $0.2$ for every target, including $f_6$ with $d = 10$, with the noise precision integrated out under the same Gamma prior; and
(iv) a random forest (RF) \citep{breiman2001random} with $500$ trees.

\paragraph{Inference and evaluation.}
The BKAN and the BNN are fitted by mean-field Gaussian variational inference (VI).
We report the restart with the smallest RMSE of the posterior predictive mean on the validation set; the validation set is not used by BART, the GP, or the RF.
We evaluate the posterior mean $\hat f$ by the test RMSE $\{ \sum_{i=1}^{1000} (\hat f(\tilde X_i) - f_0(\tilde X_i))^2 / 1000 \}^{1/2}$ on the noiseless test set.
For every Bayesian method we also report the pointwise $90\%$ predictive interval for a new response, given by the $5\%$ and $95\%$ quantiles of the posterior predictive draws of $Y$ (for the GP, of its predictive mixture), by its coverage of the noisy test responses $\tilde Y_i$ and its average width; since the noise standard deviation is $0.1$, the width of this interval cannot fall below about $0.33$.

\paragraph{Computing environment.}
All experiments were run on a Mac mini with an Apple M4 chip (10-core CPU), without GPU acceleration.
The BKAN, the BNN, NUTS, and BART are implemented in \texttt{JAX} with \texttt{NumPyro} and \texttt{bartz} \citep{petrillo2024bartz}, in single precision; each fit runs on a single core, and ten fits run in parallel.
The four NUTS chains of one fit also run one after another on a single core.

\paragraph{Choice of the inference algorithm.}
To assess the quality of VI, we compare it with the No-U-Turn sampler (NUTS) \citep{hoffman2014no} for the same BKAN on $f_1$ and $f_2$ at $n = 500$, running four chains of $1{,}500$ warm-up and $1{,}000$ posterior draws, each started at the initial value of the corresponding VI restart.
As shown in \Cref{tab:nuts_vs_vi} and \Cref{fig:nuts_vs_vi}, NUTS attains a smaller RMSE and a coverage closer to the nominal level, while VI yields slightly wider and conservative intervals.
\Cref{fig:nuts_vs_vi_fits} shows that the two posteriors are nevertheless close in function space: both posterior means recover the jump of $f_1$ and the kink of $f_2$, with both slightly smoothing the true function at $x_1 = 1/2$, while the VI intervals are somewhat wider than the NUTS intervals across the domain, in line with the higher coverage of VI in \Cref{tab:nuts_vs_vi}.
However, NUTS takes $55$--$60$ times longer than VI, about seven hours per fit.
We use VI for all reported results.
% Since the BKAN and the BNN share the same prior and the same VI procedure, their comparison is not affected by this choice.

\begin{table}[htbp]
    \centering
    \small
    \caption{NUTS against VI for the same BKAN ($N_n = 1$, $n = 500$), averaged over five replications (standard deviation of the RMSE in parentheses). Coverage and width refer to the pointwise $90\%$ predictive interval for a new response on the noisy test set; time is the median wall-clock time per fit.}
    \label{tab:nuts_vs_vi}
    \begin{tabular}{llcccc}
        \toprule
        Target & Inference & RMSE & Coverage & Width & Minutes \\
        \midrule
        $f_1$ & NUTS & $0.0339$ $(0.0020)$ & $0.902$ & $0.346$ & $414.7$ \\
              & VI   & $0.0491$ $(0.0088)$ & $0.946$ & $0.418$ & $7.5$ \\
        $f_2$ & NUTS & $0.0237$ $(0.0026)$ & $0.896$ & $0.332$ & $456.4$ \\
              & VI   & $0.0327$ $(0.0024)$ & $0.943$ & $0.400$ & $7.5$ \\
        \bottomrule
    \end{tabular}
\end{table}

\begin{figure}[htbp]
    \centering
    \includegraphics[width=0.85\linewidth]{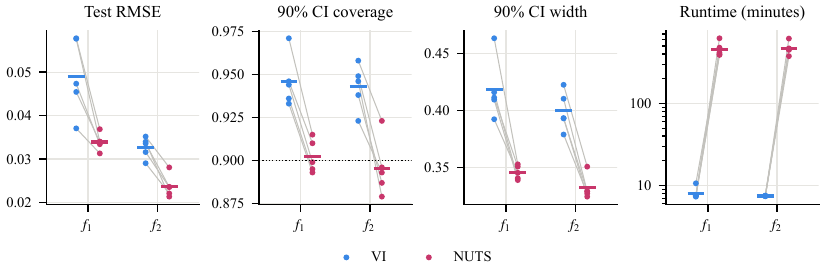}
    \caption{NUTS against VI for the same BKAN ($n = 500$). Dots are replications (paired by grey lines) and bars means (geometric mean for runtime). Coverage and width refer to the $90\%$ predictive interval on the noisy test set, and the dotted line marks its nominal coverage.}
    \label{fig:nuts_vs_vi}
\end{figure}

\begin{figure}[htbp]
    \centering
    \includegraphics[width=0.85\linewidth]{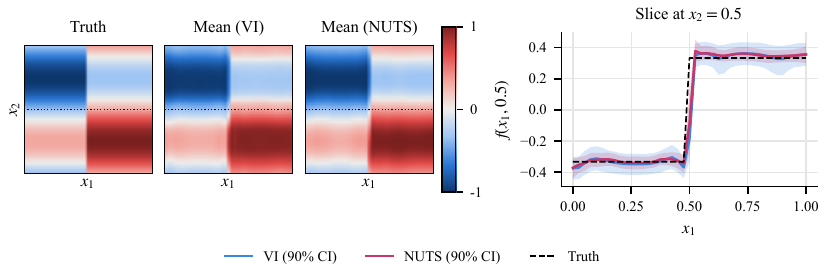}\\[-0.2em]
    {\small (a) $f_1$}\\[0.6em]
    \includegraphics[width=0.85\linewidth]{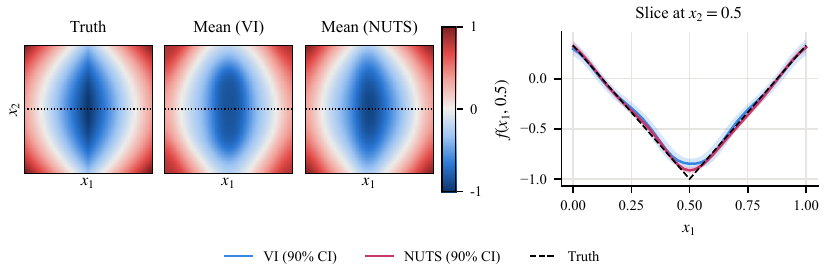}\\[-0.2em]
    {\small (b) $f_2$}
    \caption{VI against NUTS for the same BKAN ($N_n = 1$, $n = 500$, first replication) on (a) $f_1$ and (b) $f_2$: the true function and the posterior means (left), and their slices along the dotted line $x_2 = 1/2$ with pointwise $90\%$ credible intervals (right).}
    \label{fig:nuts_vs_vi_fits}
\end{figure}

\subsection{Results}    \label{app:sim-results}

\paragraph{Estimation error.}
\Cref{fig:rmse_vs_n} shows the test RMSE against the sample size.
The RMSE of the BKAN decreases steadily in $n$ for every target, and the BKAN attains the smallest or near-smallest error in most settings, in particular for $f_2$, $f_4$, $f_5$, and $f_6$ at large $n$.
For the composite target $f_5$, the BKAN has the smallest error for $n \ge 500$, in line with \Cref{thm:composite}, and for $f_6$, the BKAN and the BNN adapt to the five inactive inputs, whereas the GP with a fixed isotropic kernel fails to do so.
The exceptions are $f_1$, where BART is comparable to or slightly better than the BKAN, and $f_3$, where the BNN has a smaller error, especially at small $n$.
Both are explained by the shape of the targets: the jump of $f_1$ is aligned with a coordinate axis, which a single tree split represents exactly, and the jump of $f_3$ lies on the hyperplane $x_1 + x_2 = 1$, which a single ReLU unit represents exactly.
Even for $f_3$, the RMSE of the BKAN decreases faster than that of the other methods as $n$ grows.

\begin{figure}[htbp]
    \centering
    \includegraphics[width=0.8\linewidth]{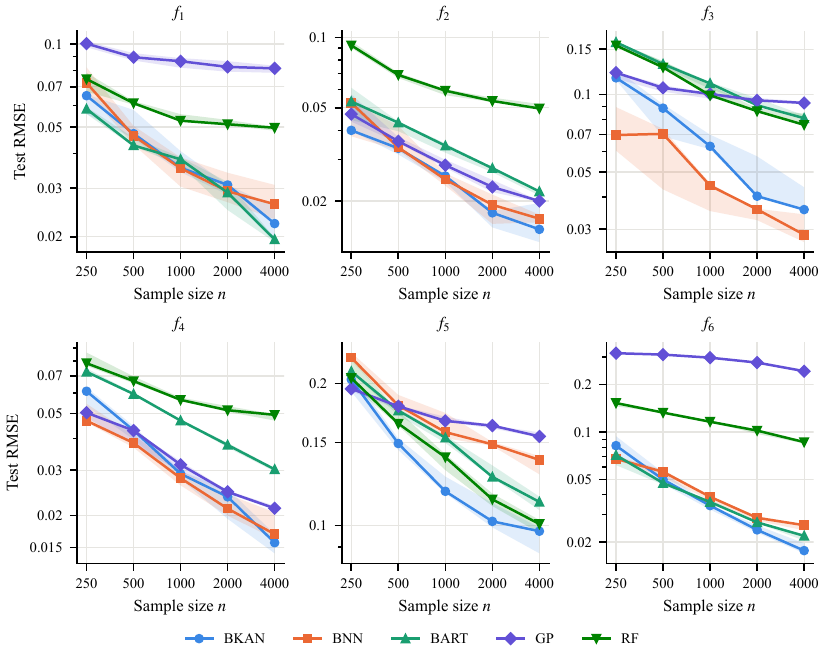}
    \caption{Test RMSE against the sample size $n$ (log--log). Lines are medians over five replications and bands the interquartile range.}
    \label{fig:rmse_vs_n}
\end{figure}

\paragraph{Posterior mean surfaces.}
\Cref{fig:surfaces} compares the posterior mean surfaces with the true functions at $n = 250$ and $n = 4{,}000$.
At $n = 250$, the BKAN recovers $f_1$, $f_2$, $f_4$, and $f_6$ well but smooths out the diagonal jump of $f_3$, which the BNN locates more sharply.
The tree-based methods, BART and RF, reproduce the axis-aligned jump of $f_1$ but show staircase artifacts on the rotated targets $f_3$, $f_4$, and $f_5$, and the GP oversmooths the discontinuities and nearly flattens $f_6$.
At $n = 4{,}000$, where the design is dense enough, the BKAN recovers the diagonal jump of $f_3$ and the layered pattern of $f_5$.

\begin{figure}[htbp]
    \centering
    \begin{minipage}{0.7\linewidth}
        \centering
        \includegraphics[width=\linewidth]{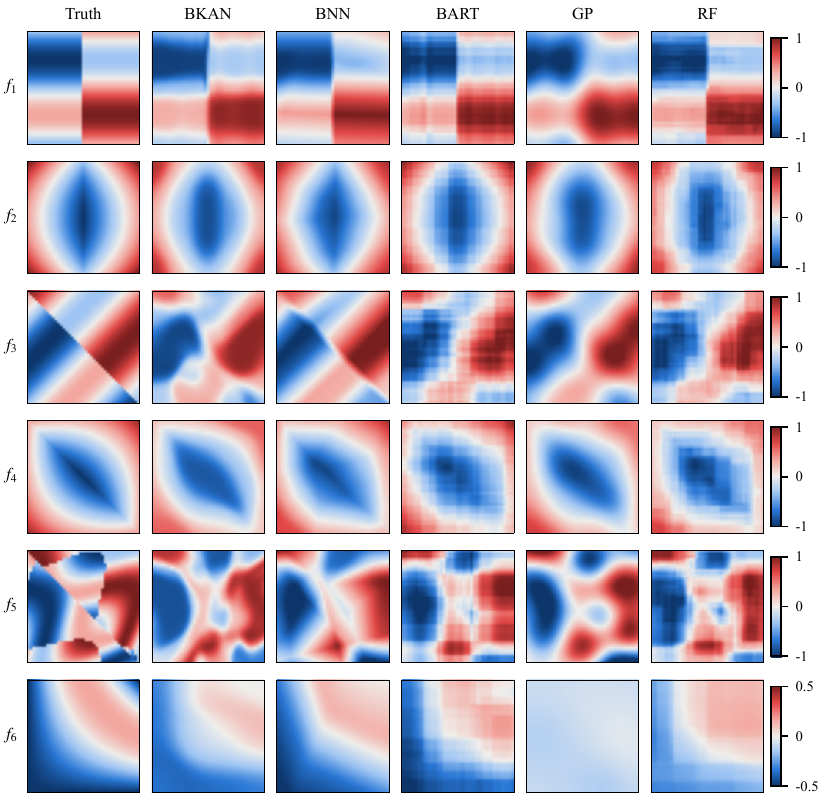}\\[-0.2em]
        {\small (a) $n = 250$}
    \end{minipage}\hfill
    \begin{minipage}{0.7\linewidth}
        \centering
        \includegraphics[width=\linewidth]{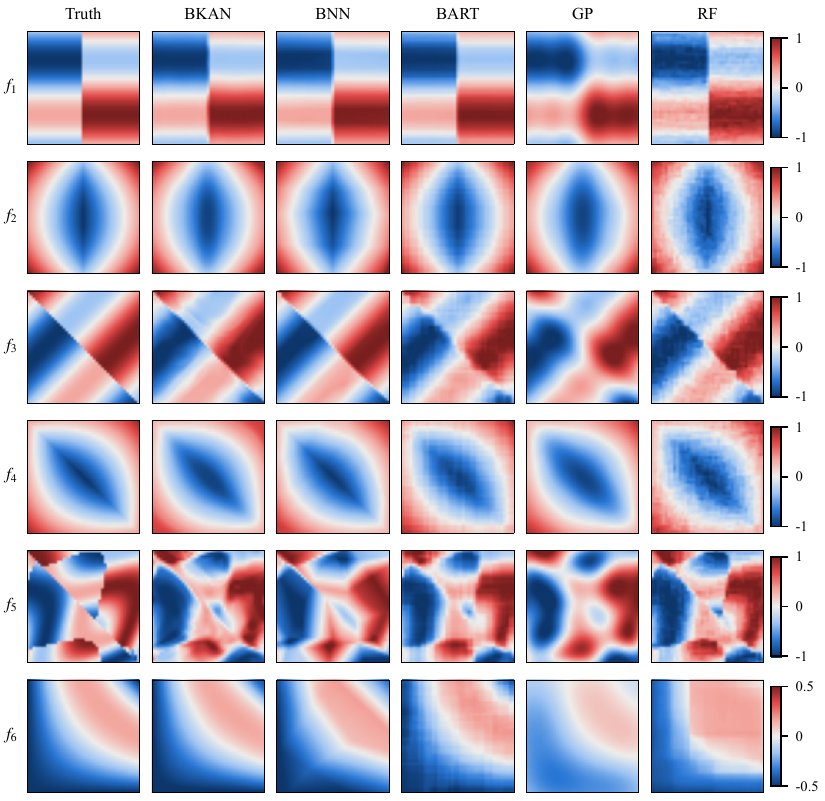}\\[-0.2em]
        {\small (b) $n = 4{,}000$}
    \end{minipage}
    \caption{Posterior mean surfaces (first replication) next to the true functions (first column of each panel) at (a) $n = 250$ and (b) $n = 4{,}000$. For $f_6$, the surfaces are shown on $(x_1, x_2)$ with $x_3 = \dots = x_{10} = 1/2$.}
    \label{fig:surfaces}
\end{figure}

%%%%%%%%%%%%%%%%%%%%%%%%%%%%%%%%%%%%%%%%%%%%%%%%%%%%%%%%%%%%

% \newpage
% \input{checklist.tex}

\end{document}